\documentclass{article} 
\usepackage[table,svgnames]{xcolor}
\usepackage{iclr2023_conference_arxiv,times}

\usepackage{amsmath}
\usepackage{amsthm}
\usepackage{graphicx}
\usepackage{listings}

\definecolor{linkcolor}{RGB}{74, 102, 146}
\usepackage[colorlinks=true,allcolors=linkcolor,pageanchor=true,plainpages=false,pdfpagelabels,bookmarks,bookmarksnumbered]{hyperref}
\usepackage{amssymb}
\usepackage{tabularx}
\usepackage{xcolor}

\newcommand{\norm}[1]{\left\Vert#1\right\Vert}

\newcommand{\parr}[1]{\left (#1\right )}
\newcommand{\brac}[1]{\left [#1\right ]}
\newcommand{\ip}[1]{\left \langle #1 \right \rangle }
\newcommand{\Real}{\mathbb R}

\newcommand{\eps}{\varepsilon}
\newcommand{\too}{\rightarrow}

\newcommand{\divv}{\mathrm{div}} 

\definecolor{mygray}{gray}{0.95}
\newcommand{\CFM}{\scriptscriptstyle \text{CFM}}
\newcommand{\FM}{\scriptscriptstyle \text{FM}}
\newcommand{\SM}{\scriptscriptstyle \text{SM}}
\newcommand{\NM}{\scriptscriptstyle \text{NM}}

\newcommand{\eg}{{e.g.}}
\newcommand{\ie}{{i.e.}}

 \newtheorem{lemma}{Lemma}

\usepackage{amsmath,amsfonts,bm}

\def\eqref#1{equation~\ref{#1}}

\def\1{\bm{1}}

\def\eps{{\epsilon}}

\def\ra{{\textnormal{a}}}

\DeclareMathAlphabet{\mathsfit}{\encodingdefault}{\sfdefault}{m}{sl}
\SetMathAlphabet{\mathsfit}{bold}{\encodingdefault}{\sfdefault}{bx}{n}

\def\gL{{\mathcal{L}}}

\def\gN{{\mathcal{N}}}

\def\gU{{\mathcal{U}}}

\newcommand{\E}{\mathbb{E}}

\usepackage{tabularx,ragged2e,booktabs,caption}
\newcolumntype{C}[1]{>{\Centering}m{#1}}

\newcolumntype{Z}[1]{>{\Left}m{#1}}

\usepackage{wrapfig}
\usepackage{lipsum}
\usepackage{caption}
\usepackage{subcaption}
\usepackage{thmtools,thm-restate}
\usepackage{nicefrac}

\usepackage{graphicx}
\usepackage{amsmath,amssymb}

\usepackage{hyperref}
\usepackage{url}
\usepackage[normalem]{ulem}

\newcommand*\rot[1]{\rotatebox{90}{#1}}

\usepackage{xspace}
\newcommand{\marginal}{marginal\@\xspace}

\renewcommand*{\eg}{{\it e.g.}\@\xspace}
\renewcommand*{\ie}{{\it i.e.}\@\xspace}

\newcommand{\sigmamin}{\sigma_{\text{min}}}

\usepackage{booktabs}
\renewcommand{\ra}[1]{\renewcommand{\arraystretch}{#1}}

\usepackage{array}
\newcolumntype{L}[1]{>{\raggedright\let\newline\\\arraybackslash\hspace{0pt}}m{#1}}
\newcolumntype{R}[1]{>{\raggedleft\let\newline\\\arraybackslash\hspace{0pt}}m{#1}}

\title{Flow Matching for Generative Modeling}

\author{Yaron Lipman$^{1,2}$ \ \ Ricky T. Q. Chen$^{1}$ \ \ Heli Ben-Hamu$^{2}$ \ \ Maximilian Nickel$^{1}$ \ \ Matt Le$^{1}$ \\
    $^1$Meta AI (FAIR)  \ \
    $^2$Weizmann Institute of Science
}

\iclrfinalcopy 

\begin{document}

\maketitle

\begin{abstract}
We introduce a new paradigm for generative modeling built on Continuous Normalizing Flows (CNFs), allowing us to train CNFs at unprecedented scale. Specifically, we present the notion of Flow Matching (FM), a simulation-free approach for training CNFs based on regressing vector fields of fixed conditional probability paths. Flow Matching is compatible with a general family of Gaussian probability paths for transforming between noise and data samples---which subsumes existing diffusion paths as specific instances. Interestingly, we find that employing FM with diffusion paths results in a more robust and stable alternative for training diffusion models. Furthermore, Flow Matching opens the door to training CNFs with other, non-diffusion probability paths. An instance of particular interest is using Optimal Transport (OT) displacement interpolation to define the conditional probability paths. These paths are more efficient than diffusion paths, provide faster training and sampling, and result in better generalization. Training CNFs using Flow Matching on ImageNet leads to consistently better performance than alternative diffusion-based methods in terms of both likelihood and sample quality, and allows fast and reliable sample generation using off-the-shelf numerical ODE solvers.\vspace{-5pt}\looseness=-1
\end{abstract}

\section{Introduction}\vspace{-3pt}
Deep generative models are a class of deep learning algorithms aimed at estimating and sampling from an unknown data distribution.
%
The recent influx of amazing advances in generative modeling, \eg, for image generation \citet{ramesh2022hierarchical,rombach2022high}, is mostly facilitated by the scalable and relatively stable training of diffusion-based models \citet{ho2020denoising,song2020score}. However, the restriction to simple diffusion processes leads to a rather confined space of sampling probability paths, resulting in very long training times and the need to adopt specialized methods~(\eg, \citet{song2020denoising,zhang2022fast}) for efficient sampling.

In this work we consider the general and deterministic framework of Continuous Normalizing Flows (CNFs; \citet{chen2018neural}). 
CNFs are capable of modeling arbitrary probability path 
\begin{wrapfigure}[19]{r}{0.45\textwidth}
\vspace{-17pt}
  \begin{center}
  \includegraphics[width=0.44\textwidth]{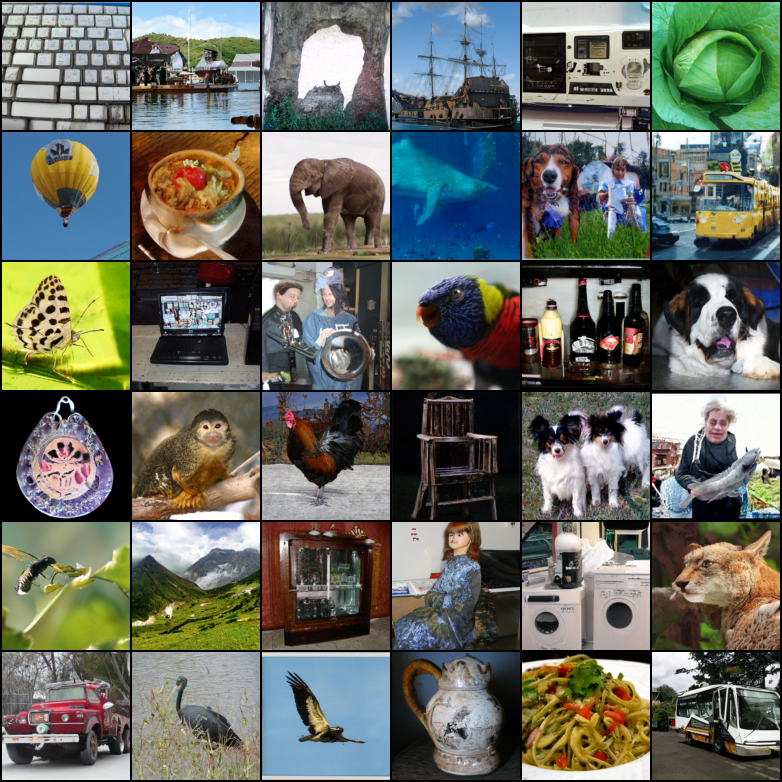}
  \end{center}
  \vspace{-1em}
  \caption{Unconditional ImageNet-128 samples of a CNF trained using Flow Matching with Optimal Transport probability paths.}
\label{fig:teaser_256}
\end{wrapfigure} 
and are in particular known to
encompass the probability paths modeled by diffusion processes \citep{song2021maximum}. However, aside from diffusion that can be trained  efficiently via, \eg, denoising score matching~\citep{vincent2011connection}, no scalable CNF training algorithms are known. Indeed, maximum likelihood training (\eg, \citet{ffjord2018}) require expensive numerical ODE simulations, while existing simulation-free methods either involve intractable integrals \citep{rozen2021moser} or biased gradients \citep{ben2022matching}. 

The goal of this work is to propose Flow Matching (FM), an efficient simulation-free approach to training CNF models, allowing the adoption of general probability paths to supervise CNF training.  Importantly, FM breaks the barriers for scalable CNF training beyond diffusion, and sidesteps the need to reason about diffusion processes to directly work with probability paths.

In particular, we propose the Flow Matching objective (Section \ref{sec:flow_matching}), a simple and intuitive training objective to regress onto a target vector field that generates a desired probability path. 
We first show that we can construct such target vector fields through per-example (\ie, conditional) formulations. 
Then, inspired by denoising score matching, we show that a per-example training objective, termed Conditional Flow Matching (CFM), provides equivalent gradients and does not require explicit knowledge of the intractable target vector field.
Furthermore, we discuss a general family of per-example probability paths (Section \ref{sec:cond_paths}) that can be used for Flow Matching, which subsumes existing diffusion paths as special instances. 
Even on diffusion paths, we find that using FM provides more robust and stable training, and achieves superior performance compared to score matching.
Furthermore, this family of probability paths also includes a particularly interesting case: the vector field that corresponds to an Optimal Transport (OT) displacement interpolant \citep{mccann1997convexity}. We find that conditional OT paths are simpler than diffusion paths, forming straight line trajectories whereas diffusion paths result in curved paths. These properties seem to empirically translate to faster training, faster generation, and better performance. 

We empirically validate Flow Matching and the construction via Optimal Transport paths on ImageNet, a large and highly diverse image dataset. We find that we can easily train models to achieve favorable performance in both likelihood estimation and sample quality amongst competing diffusion-based methods. Furthermore, we find that our models produce better trade-offs between computational cost and sample quality compared to prior methods. Figure \ref{fig:teaser_256} depicts selected unconditional ImageNet 128$\times$128 samples from our model.

\section{Preliminaries: Continuous Normalizing Flows}
\label{s:prelim}
Let $\Real^d$ denote the data space with data points $x=(x^1,\ldots,x^d) \in \Real^d$. Two important objects we use in this paper are: the \emph{probability density path} $p:[0,1]\times \Real^d \too \Real_{>0}$, which is a time dependent\footnote{We use subscript to denote the time parameter, 
\eg, $p_t(x)$.} probability density function, \ie, $\int p_t(x)dx = 1$, and a \emph{time-dependent vector field}, $v:[0,1]\times \Real^d \too \Real^d$. A vector field $v_t$ can be used to construct a time-dependent diffeomorphic map, called a \emph{flow}, $\phi:[0,1]\times \Real^d \too \Real^d$, defined via the ordinary differential equation (ODE):
\begin{align}\label{e:ode}
    \frac{d}{dt}\phi_t(x) &= v_t(\phi_t(x)) \\
    \phi_0(x) &= x
\end{align}
Previously, \cite{chen2018neural} suggested modeling the vector field $v_t$ with a neural network, $v_t(x;\theta)$, where $\theta\in \Real^p$ are its learnable parameters, which in turn leads to a deep parametric model of the flow $\phi_t$, called a \emph{Continuous Normalizing Flow} (CNF). A CNF is used to reshape a simple prior density $p_0$ (\eg, pure noise) to a more complicated one, $p_1$, via the push-forward equation
\begin{equation}\label{e:push_forward}
p_t = [\phi_t]_* p_0
\end{equation}
where the push-forward (or change of variables) operator $*$ is defined by
\begin{equation}\label{e:push_forward_explicit}
    [\phi_t]_* p_0(x) = p_0(\phi_t^{-1}(x))\det \brac{ \frac{\partial \phi_t^{-1}}{\partial x}(x)}.
\end{equation}
A vector field $v_t$ is said to \emph{generate} a probability density path $p_t$ if its flow $\phi_t$ satisfies \eqref{e:push_forward}. One practical way to test if a vector field generates a probability path is using the continuity equation, which is a key component in our proofs, see Appendix \ref{A:continuity_equation}. We recap more information on CNFs, in particular how to compute the probability $p_1(x)$ at an arbitrary point $x\in\Real^d$ in Appendix \ref{A:cnf_prob}. 

\section{Flow Matching}\label{sec:flow_matching}
Let $x_1$ denote a random variable distributed according to some unknown data distribution $q(x_1)$.
We assume we only have access to data samples from $q(x_1)$ but have no access to the density function itself.
Furthermore, we let $p_t$ be a probability path such that $p_0=p$ is a simple distribution, \eg, the standard normal distribution $p(x) = \gN(x | 0, I)$, and let $p_1$ be approximately equal in distribution to $q$. We will later discuss how to construct such a path.
The Flow Matching objective is then designed to match this target probability path, which will allow us to flow from $p_0$ to $p_1$%
.



Given a target probability density path $p_t(x)$ and a corresponding vector field $u_t(x)$, which generates $p_t(x)$, we define the Flow Matching (FM) objective as 
\begin{equation}\label{e:fm}
    \mathcal{L}_{\FM }(\theta) = \mathbb{E}_{t,p_t(x)}\| v_t(x)-u_t(x) \|^2,
\end{equation}
where $\theta$ denotes the learnable parameters of the CNF vector field $v_t$ (as defined in Section \ref{s:prelim}), $t\sim \gU[0,1]$ (uniform distribution), and $x\sim p_t(x)$. Simply put, the FM loss regresses the vector field $u_t$ with a neural network $v_t$. Upon reaching zero loss, the learned CNF model will generate $p_t(x)$.

Flow Matching is a simple and attractive objective, but na\"ively on its own, it is intractable to use in practice since we have no prior knowledge for what an appropriate $p_t$ and $u_t$ are.
There are many choices of probability paths that can satisfy $p_1(x) \approx q(x)$, and more importantly, we generally don't have access to a closed form $u_t$ that generates the desired $p_t$. 
In this section, we show that we can construct both $p_t$ and $u_t$ using probability paths and vector fields that are only defined \emph{per sample}, and an appropriate method of aggregation provides the desired $p_t$ and $u_t$. 
Furthermore, this construction allows us to create a much more tractable objective for Flow Matching.

\subsection{Constructing $p_t,u_t$ from conditional probability paths and vector fields} 
A simple way to construct a target probability path is via a mixture of simpler probability paths: Given a particular data sample $x_1$ we denote by $p_t(x\vert x_1)$ a \emph{conditional probability path} 
such that it satisfies $p_0(x|x_1) = p(x)$ at time $t=0$, and we design $p_1(x|x_1)$ at $t=1$ to be a distribution concentrated around $x=x_1$, \eg, $p_1(x|x_1)=\gN(x|x_1,\sigma^2 I)$, a normal distribution with $x_1$ mean and a sufficiently small standard deviation $\sigma > 0$. 
Marginalizing the conditional probability paths over $q(x_1)$ give rise to \emph{the \marginal probability path}
\begin{equation}\label{e:p_t}
    p_t(x)=\int p_t(x|x_1)q(x_1)dx_1,
\end{equation}
where in particular at time $t=1$, the \marginal probability $p_1$ is a mixture distribution that closely approximates the data distribution $q$,
\begin{equation}\label{e:p_1} 
    p_1(x)=\int p_1(x|x_1)q(x_1)dx_1\approx q(x).
\end{equation}
Interestingly, we can also define a \emph{\marginal vector field}, by ``marginalizing'' over the conditional vector fields in the following sense (we assume $p_t(x)>0$ for all $t$ and $x$):
%
\begin{equation}\label{e:u_t}
    u_t(x) 
    = \int u_t(x\vert x_1) \frac{p_t(x\vert x_1)q(x_1)}{p_t(x)}dx_1,
\end{equation}
where $u_t(\cdot|x_1):\Real^d\too\Real^d$ is a conditional vector field that generates $p_t(\cdot\vert x_1)$. It may not seem apparent, but this way of aggregating the conditional vector fields actually results in the correct vector field for modeling the marginal probability path.

Our first key observation is this:
\vspace{-0.5em}
\begin{center}			
    \colorbox{mygray} {		
      \begin{minipage}{0.977\linewidth} 	
       \centering
        \emph{The marginal vector field (\eqref{e:u_t}) generates the marginal probability path (\eqref{e:p_t}).}
      \end{minipage}}			
\end{center}
This provides a surprising connection between the conditional VFs (those that generate conditional probability paths) and the marginal VF (those that generate the marginal probability path). This connection allows us to break down the unknown and intractable marginal VF into simpler conditional VFs, which are much simpler to define as these only depend on a single data sample. We formalize this in the following theorem.
\begin{restatable}{theorem}{marginalvf}\label{thm:marginal_vf}
Given vector fields $u_t(x|x_1)$ that generate conditional probability paths $p_t(x|x_1)$, for any distribution $q(x_1)$, the marginal vector field $u_t$ in \eqref{e:u_t} generates the marginal probability path $p_t$ in \eqref{e:p_t}, \ie, $u_t$ and $p_t$ satisfy the continuity equation (equation \ref{e:continuity}).
\end{restatable}
The full proofs for our theorems are all provided in Appendix \ref{A:proofs}. Theorem \ref{thm:marginal_vf} can also be derived from the Diffusion Mixture Representation Theorem in \cite{peluchetti2021non} that provides a formula for the marginal drift and diffusion coefficients in diffusion SDEs.

\subsection{Conditional Flow Matching} Unfortunately, due to the intractable integrals in the definitions of the marginal probability path and VF (equations \ref{e:p_t} and \ref{e:u_t}), it is still intractable to compute $u_t$, and consequently, intractable to na\"ively compute an unbiased estimator of the original Flow Matching objective.
Instead, we propose a simpler objective, which surprisingly will result in the same optima as the original objective. 
Specifically, we consider the \emph{Conditional Flow Matching} (CFM) objective,
\begin{equation}\label{e:cfm}
    \mathcal{L}_{\CFM}(\theta) = \mathbb{E}_{t,q(x_1),p_t(x\vert x_1)} \big\| v_t(x) - u_t(x\vert x_1) \big\|^2,
\end{equation}
where $t\sim \gU[0,1]$, $x_1\sim q(x_1)$, and now $x\sim p_t(x|x_1)$. 
Unlike the FM objective, the CFM objective allows us to easily sample unbiased estimates as long as we can efficiently sample from $p_t(x|x_1)$ and compute $u_t(x|x_1)$, both of which can be easily done as they are defined on a per-sample basis.\\
Our second key observation is therefore:
\vspace{-0.5em}
\begin{center}			
    \colorbox{mygray} {		
      \begin{minipage}{0.977\linewidth} 	
       \centering
        \emph{The FM (\eqref{e:fm}) and CFM (\eqref{e:cfm}) objectives have identical gradients w.r.t.~$\theta$.}
      \end{minipage}}			
\end{center}
That is, optimizing the CFM objective is equivalent (in expectation) to optimizing the FM objective. Consequently, this allows us to train a CNF to generate the marginal probability path $p_t$---which in particular, approximates the unknown data distribution $q$ at $t$=$1$---
without ever needing access to either the marginal probability path or the marginal vector field. We simply need to design suitable \emph{conditional} probability paths and vector fields.
We formalize this property in the following theorem.
\begin{restatable}{theorem}{cfm}\label{thm:grad_VR_equals_CVR}
Assuming that $p_t(x)>0$ for all $x\in\Real^d$ and $t\in [0,1]$, then, up to a constant independent of $\theta$, $\gL_{\CFM}$ and $\gL_{\FM}$ are equal. Hence,  
$\nabla_\theta \mathcal{L}_{\FM}(\theta) = \nabla_\theta \mathcal{L}_{\CFM}(\theta)$.
\end{restatable}

\section{Conditional Probability Paths and Vector Fields}
\label{sec:cond_paths}

The Conditional Flow Matching objective works with any choice of conditional probability path and conditional vector fields. In this section, we discuss the construction of $p_t(x \vert x_1)$ and $u_t(x \vert x_1)$ for a general family of Gaussian conditional probability paths. 
Namely, we consider conditional probability paths of the form
\begin{equation}\label{e:pt_gau}
    p_t(x\vert x_1) = \gN(x\, \vert\, \mu_t(x_1), \sigma_t(x_1)^2 I ),
\end{equation}
where $\mu:[0,1]\times \Real^d\too \Real^d$ is the time-dependent mean of the Gaussian distribution, while $\sigma:[0,1]\times\Real\too \Real_{>0}$ describes a time-dependent scalar standard deviation (std).
We set $\mu_0(x_1) = 0$ and $\sigma_0(x_1) = 1$, so that all conditional probability paths converge to the same standard Gaussian noise distribution at $t=0$, $p(x)=\gN(x|0,I)$. We then set $\mu_1(x_1)=x_1$ and $\sigma_1(x_1)=\sigma_\text{min}$, which is set sufficiently small so that $p_1(x \vert x_1)$ is a concentrated Gaussian distribution centered at $x_1$.

There is an infinite number of vector fields that generate any particular probability path (\eg, by adding a divergence free component to the continuity equation, see  \eqref{e:continuity}), but the vast majority of these is due to the presence of components that leave the underlying distribution invariant---for instance, rotational components when the distribution is rotation-invariant---leading to unnecessary extra compute.
We decide to use the simplest vector field corresponding to a canonical transformation for Gaussian distributions. Specifically, consider the flow (conditioned on $x_1$)
\begin{equation}\label{e:phi_t_gau}
    \psi_t(x) = \sigma_t(x_1)x + \mu_t(x_1).
\end{equation}
When $x$ is distributed as a standard Gaussian, $\psi_t(x)$ is the affine transformation that maps to a normally-distributed random variable with mean $\mu_t(x_1)$ and std $\sigma_t(x_1)$.
That is to say, according to equation \ref{e:push_forward_explicit}, $\psi_t$ pushes the noise distribution $p_0(x \vert x_1)=p(x)$ to $p_t(x\vert x_1)$, \ie, 
\begin{equation}
    \brac{\psi_t}_*p(x) = p_t(x\vert x_1).
\end{equation} 
This flow then provides a vector field that generates the conditional probability path:
\begin{equation}\label{e:psi_t_u_t}
    \frac{d}{dt}\psi_t(x) = u_t(\psi_t(x)\vert x_1).
\end{equation}
Reparameterizing $p_t(x | x_1)$ in terms of just $x_0$ and plugging \eqref{e:psi_t_u_t} in the CFM loss we get 
\begin{equation}\label{e:cfm_dt_psi}
     \gL_{\CFM}(\theta) 
     =\E_{t, q(x_1), p(x_0)} \Big\|v_t(\psi_t(x_0)) - \frac{d}{dt}\psi_t(x_0)\Big\|^2.     
\end{equation}
Since $\psi_t$ is a simple (invertible) affine map we can use \eqref{e:psi_t_u_t} to solve for $u_t$ in a closed form. Let $f'$ denote the derivative with respect to time, \ie, $f' = \frac{d}{dt}f$, for a time-dependent function $f$.
\begin{restatable}{theorem}{condvf}\label{thm:cond_vf}
Let $p_t(x\vert x_1)$ be a Gaussian probability path as in \eqref{e:pt_gau}, and $\psi_t$ its corresponding flow map as in \eqref{e:phi_t_gau}. 
Then, the unique vector field that defines $\psi_t$ has the form:
\begin{equation}\label{e:u_t_cond_x1}
    u_t(x\vert x_1) = 
    \frac{\sigma'_t(x_1)}{\sigma_t(x_1)}\parr{x-\mu_t(x_1)} + \mu'_t(x_1).
\end{equation}
Consequently, $u_t(x\vert x_1)$ generates the Gaussian path $p_t(x\vert x_1)$.
\end{restatable}


\subsection{Special instances of Gaussian conditional probability paths}\label{ss:instances}

Our formulation is fully general for arbitrary functions $\mu_t(x_1)$ and $\sigma_t(x_1)$, and we can set them to any differentiable function satisfying the desired boundary conditions. 
We first discuss the special cases that recover probability paths corresponding to previously-used diffusion processes. 
Since we directly work with probability paths, we can simply depart from reasoning about diffusion processes altogether. Therefore, in the second example below, we directly formulate a probability path based on the Wasserstein-2 optimal transport solution as an interesting instance.

\paragraph{Example I: Diffusion conditional VFs.}
Diffusion models start with data points and gradually add noise until it approximates pure noise. These can be formulated as stochastic processes, which have strict requirements in order to obtain closed form representation at arbitrary times $t$, resulting in Gaussian conditional probability paths $p_t(x\vert x_1)$ with specific choices of mean  $\mu_t(x_1)$ and std $\sigma_t(x_1)$ \citep{sohl2015deep, ho2020denoising,song2020score}. 
For example, the reversed (noise$\too$data) Variance Exploding (VE) path has the form
\begin{equation}\label{e:VE}
    p_t(x) = \gN(x | x_1 , \sigma_{1-t}^2 I ),
\end{equation}
where $\sigma_t$ is an increasing function, $\sigma_0=0$, and $\sigma_1 \gg 1$.  Next, \eqref{e:VE} provides the choices of $\mu_t(x_1)= x_1$ and $\sigma_t(x_1)= \sigma_{1-t}$. Plugging these into \eqref{e:u_t_cond_x1} of Theorem \ref{thm:cond_vf} we get 
\begin{equation}\label{e:u_t_VE}
    u_t(x|x_1) = -\frac{\sigma'_{1-t}}{\sigma_{1-t}}(x-x_1).
\end{equation}
The reversed (noise$\too$data) Variance Preserving (VP) diffusion path has the form 
\begin{equation}\label{e:VP_diffusion_path}
    p_t(x\vert x_1) = \gN(x\, \vert \, \alpha_{1-t} x_1 , \parr{1-\alpha_{1-t}^2}I), \text{where } \alpha_t = e^{-\frac{1}{2}T(t)}, T(t)=\int_0^{t} \beta(s)ds, 
\end{equation}
and $\beta$ is the noise scale function. 
Equation \ref{e:VP_diffusion_path} provides the choices of $\mu_t(x_1)=\alpha_{1-t} x_1$ and $\sigma_t(x_1)=\sqrt{1-\alpha_{1-t}^2}$. Plugging these into \eqref{e:u_t_cond_x1} of Theorem \ref{thm:cond_vf} we get 
\begin{equation}\label{e:ut_dif_with_our_method}
    u_t(x\vert x_1) = \frac{\alpha'_{1-t}}{1-\alpha^2_{1-t}} \parr{\alpha_{1-t}x-x_1} = -\frac{T'(1-t)}{2}\brac{
    \frac{e^{-T(1-t)}x - e^{-\frac{1}{2}T(1-t)}x_1}{1-e^{-T(1-t)}}}.
\end{equation}
%
Our construction of the conditional VF $u_t(x|x_1)$ does in fact coincide with the vector field previously used in the deterministic probability flow (\cite{song2020score}, equation 13) when restricted to these conditional diffusion processes; see details in Appendix \ref{A:diff_cond_VP}. 
Nevertheless, combining the diffusion conditional VF with the Flow Matching objective offers an attractive training alternative---which we find to be more stable and robust in our experiments---to existing score matching approaches. 

Another important observation is that, as these probability paths were previously derived as solutions of diffusion processes, they do not actually reach a true noise distribution in finite time. In practice, $p_0(x)$ is simply approximated by a suitable Gaussian distribution for sampling and likelihood evaluation. Instead, our construction provides full control over the probability path, and we can just directly set $\mu_t$ and $\sigma_t$, as we will do next.



\begin{figure}
\centering
\begin{subfigure}[t]{0.0989\linewidth}
\includegraphics[width=\linewidth]{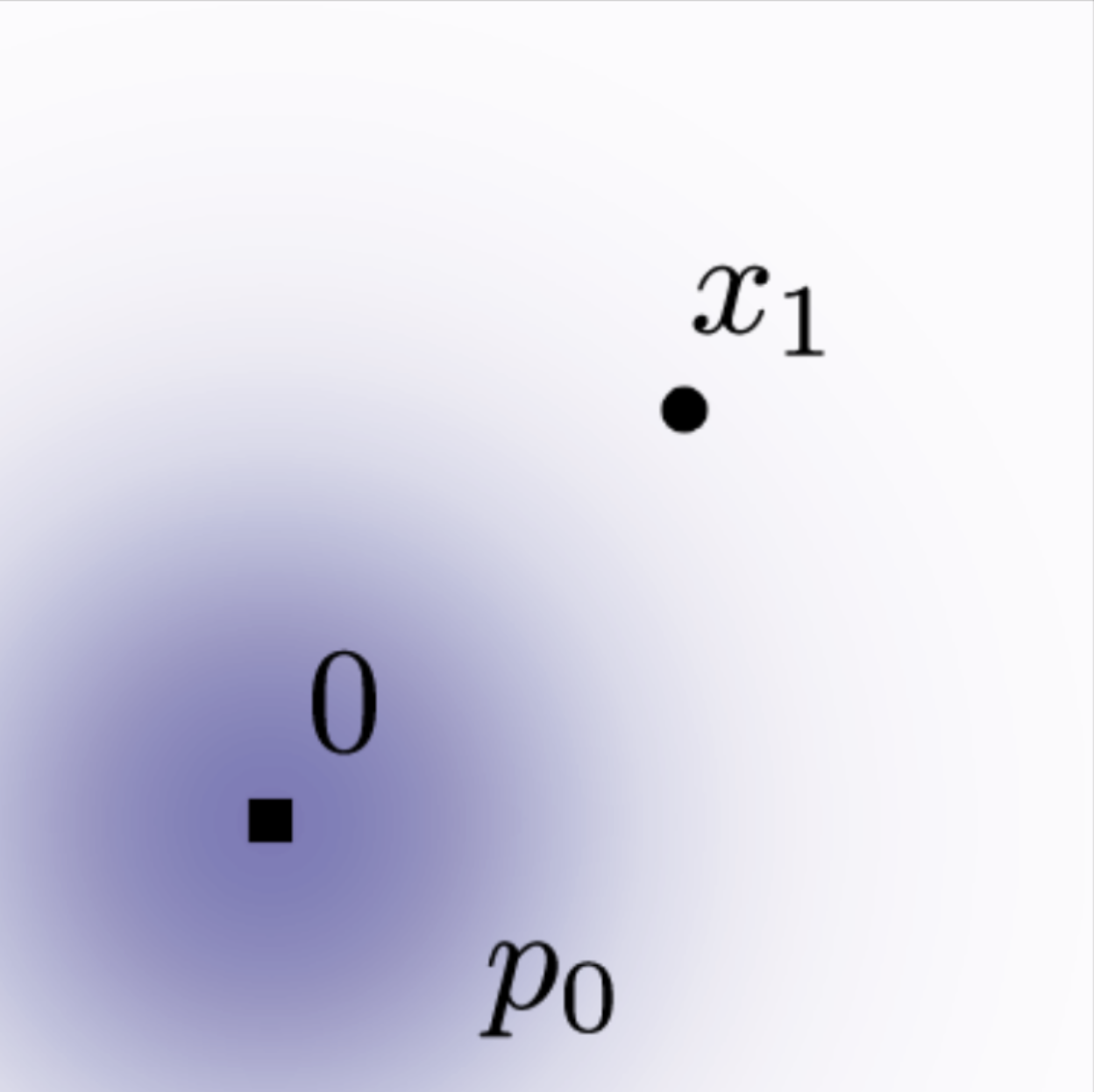}
\end{subfigure}
\hfill
\begin{subfigure}[t]{0.42\linewidth}
\centering
    \begin{subfigure}[t]{0.23\linewidth}
        \centering
        \includegraphics[width=\linewidth]{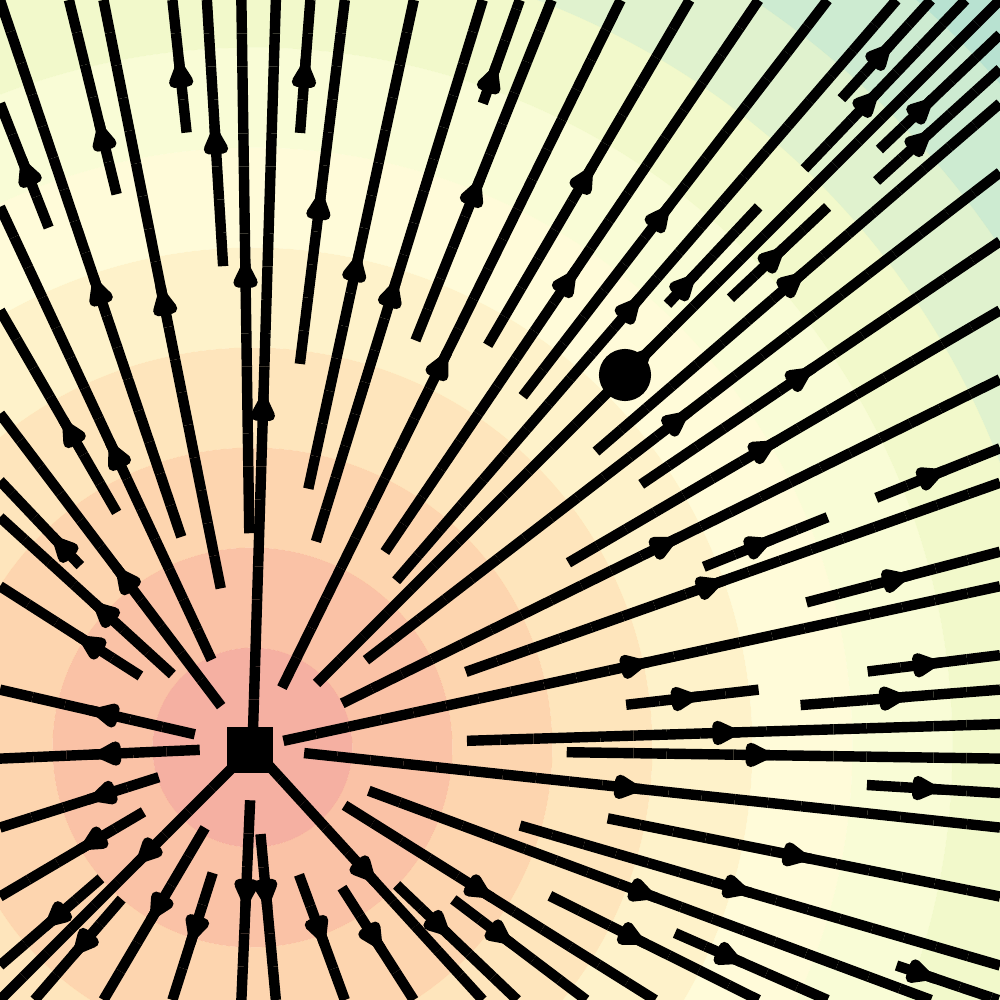}
        \vspace{-1.5em}
        \caption*{\scalebox{0.8}{$t=0.0$}}
    \end{subfigure}
    \begin{subfigure}[t]{0.23\linewidth}
        \centering
        \includegraphics[width=\linewidth]{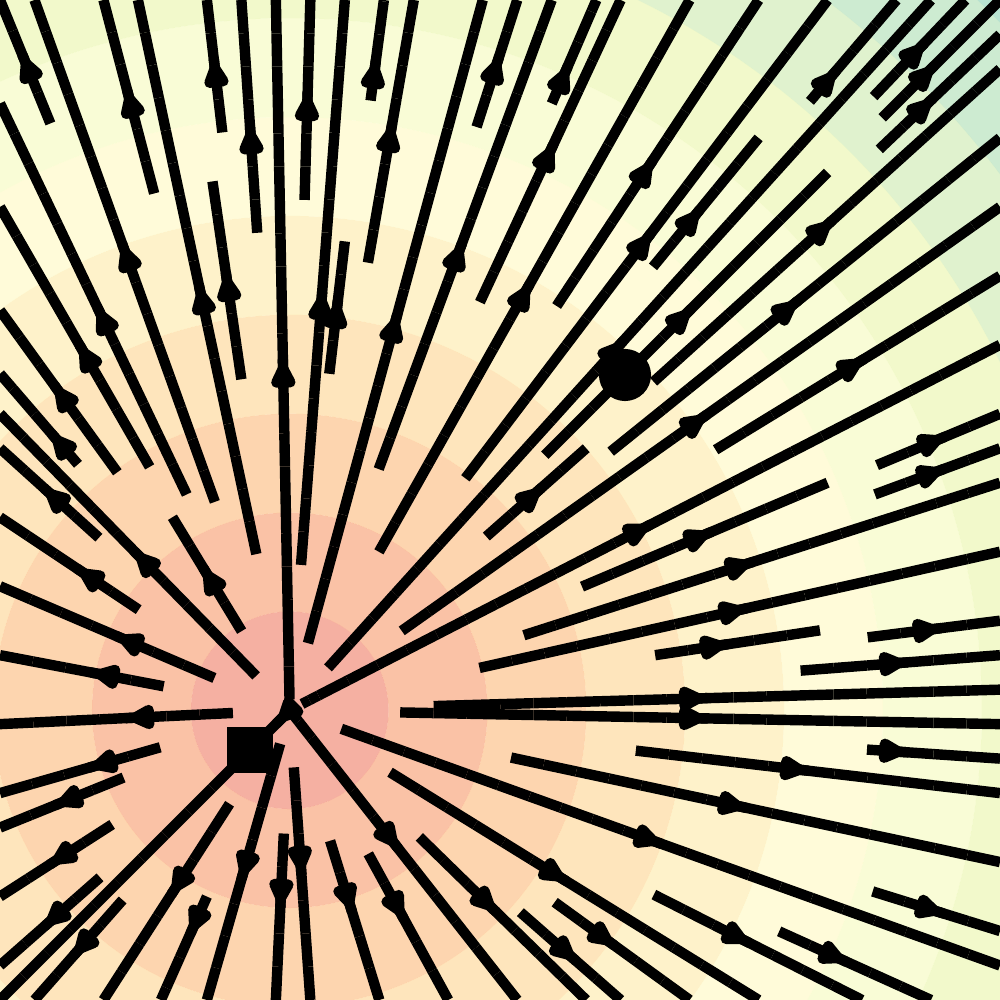}
        \vspace{-1.5em}
        \caption*{\scalebox{0.8}{$t=\nicefrac{1}{3}$}}
    \end{subfigure}
    \begin{subfigure}[t]{0.23\linewidth}
        \centering
        \includegraphics[width=\linewidth]{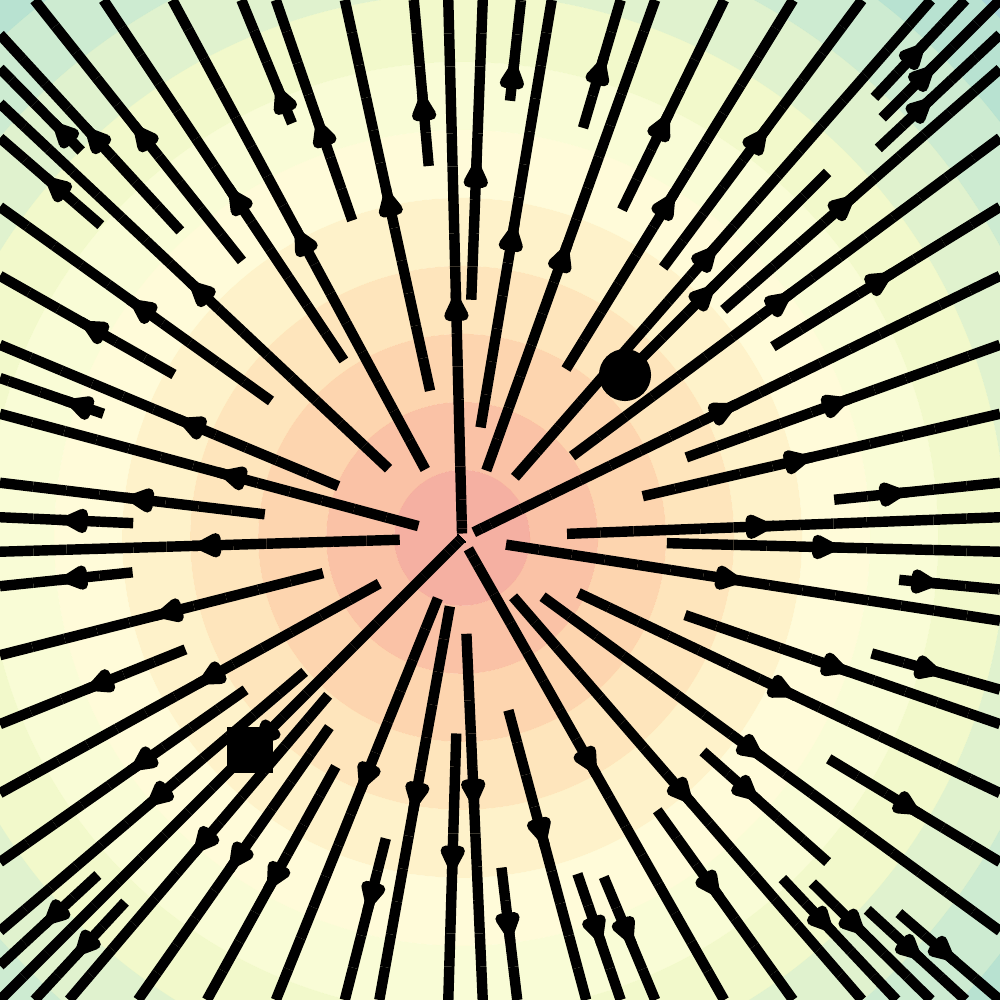}
        \vspace{-1.5em}
        \caption*{\scalebox{0.8}{$t=\nicefrac{2}{3}$}}
    \end{subfigure}
    \begin{subfigure}[t]{0.23\linewidth}
        \centering
        \includegraphics[width=\linewidth]{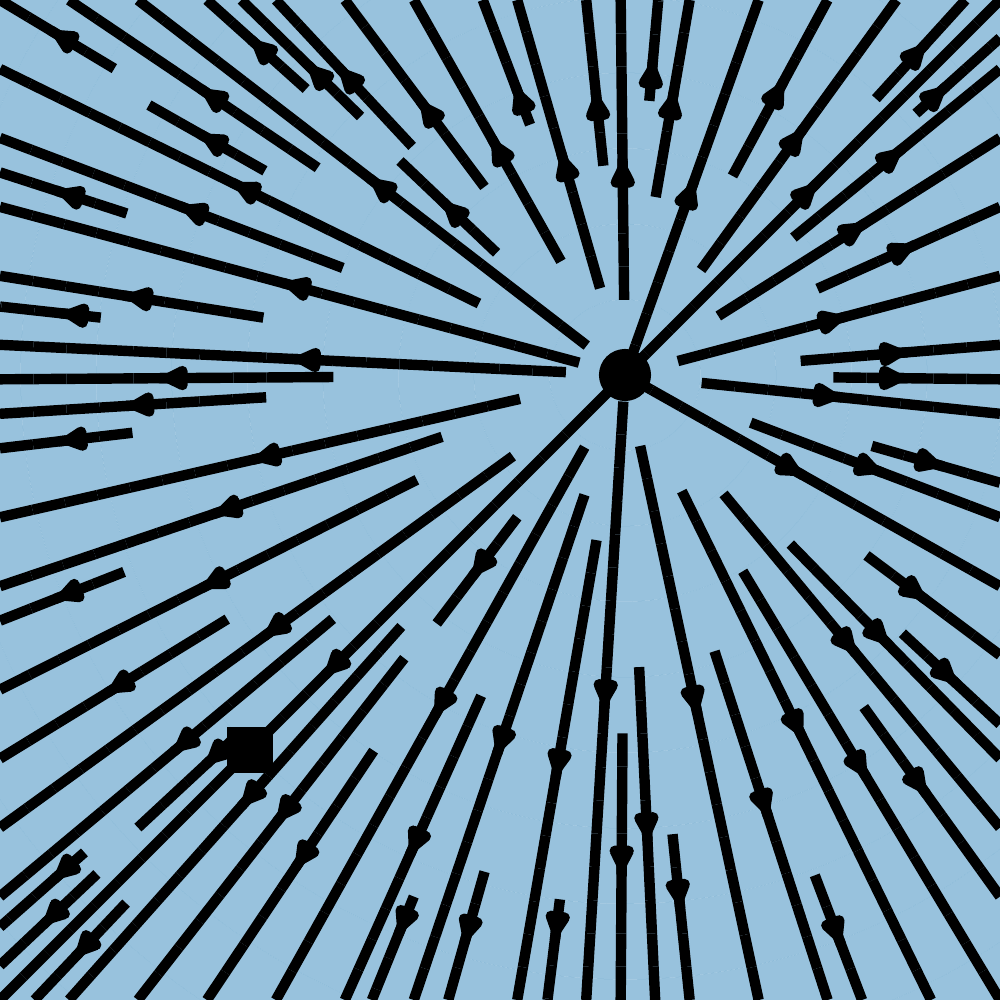}
        \vspace{-1.5em}
        \caption*{\scalebox{0.8}{$t=1.0$}}
    \end{subfigure} \vspace{-5pt}
\caption*{Diffusion path -- conditional score function}
\end{subfigure} 
\hfill
\begin{subfigure}[t]{0.43\linewidth}
\centering
    \begin{subfigure}[t]{0.23\linewidth}
        \centering
        \includegraphics[width=\linewidth]{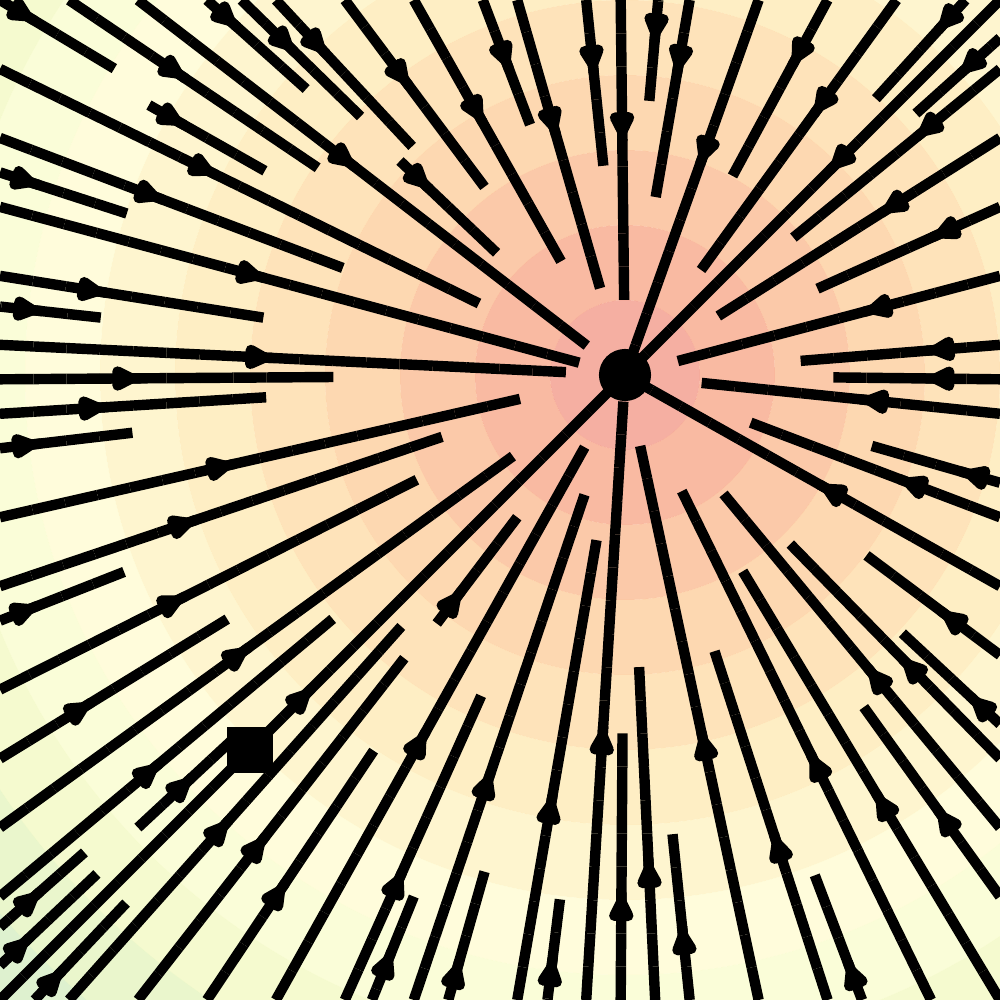}
        \vspace{-1.5em}
        \caption*{\scalebox{0.8}{$t=0.0$}}
    \end{subfigure}
    \begin{subfigure}[t]{0.23\linewidth}
        \centering
        \includegraphics[width=\linewidth]{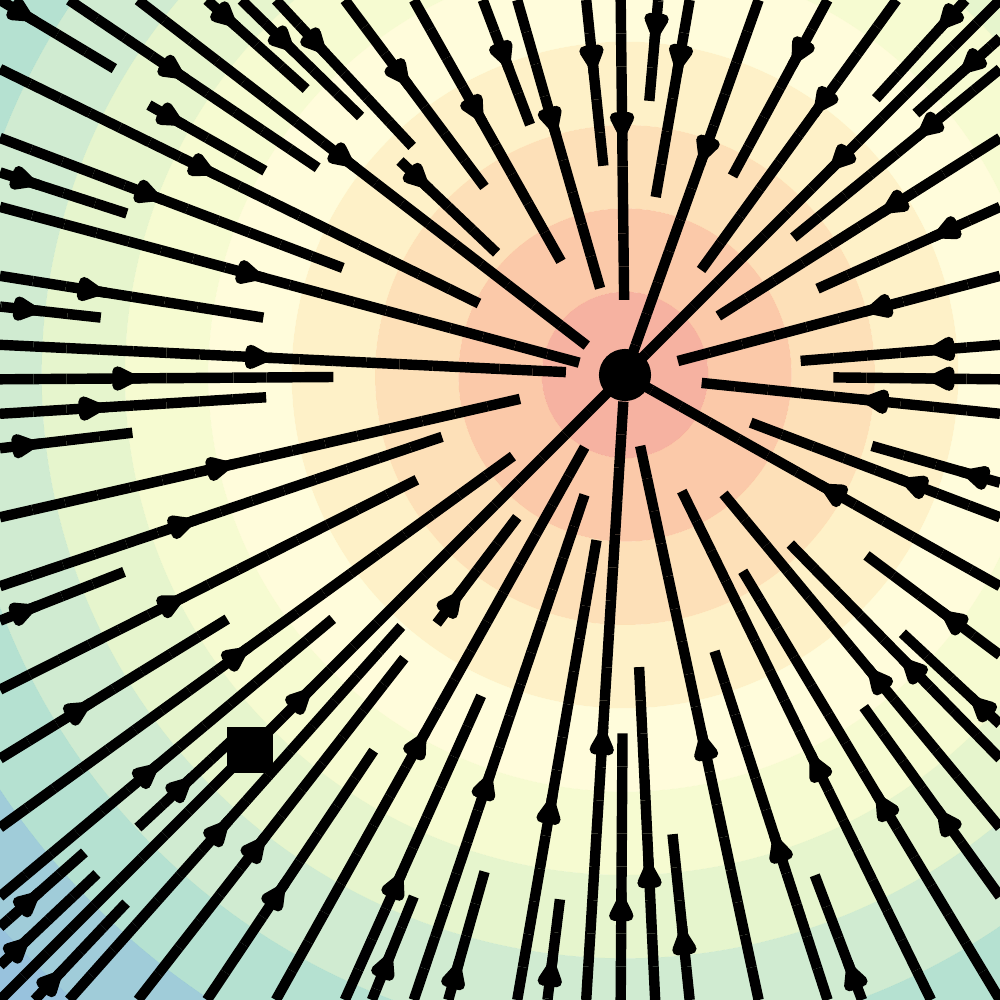}
        \vspace{-1.5em}
        \caption*{\scalebox{0.8}{$t=\nicefrac{1}{3}$}}
    \end{subfigure}
    \begin{subfigure}[t]{0.23\linewidth}
        \centering
        \includegraphics[width=\linewidth]{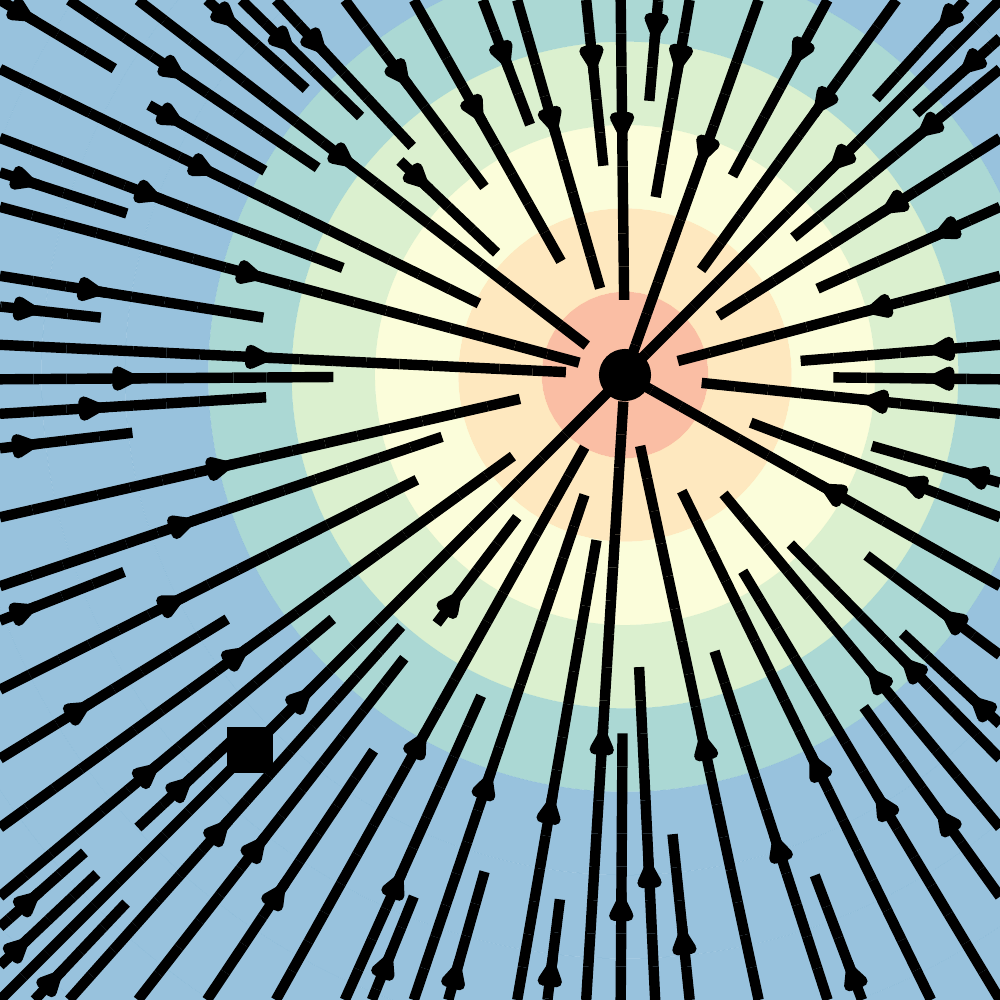}
        \vspace{-1.5em}
        \caption*{\scalebox{0.8}{$t=\nicefrac{2}{3}$}}
    \end{subfigure}
    \begin{subfigure}[t]{0.23\linewidth}
        \centering
        \includegraphics[width=\linewidth]{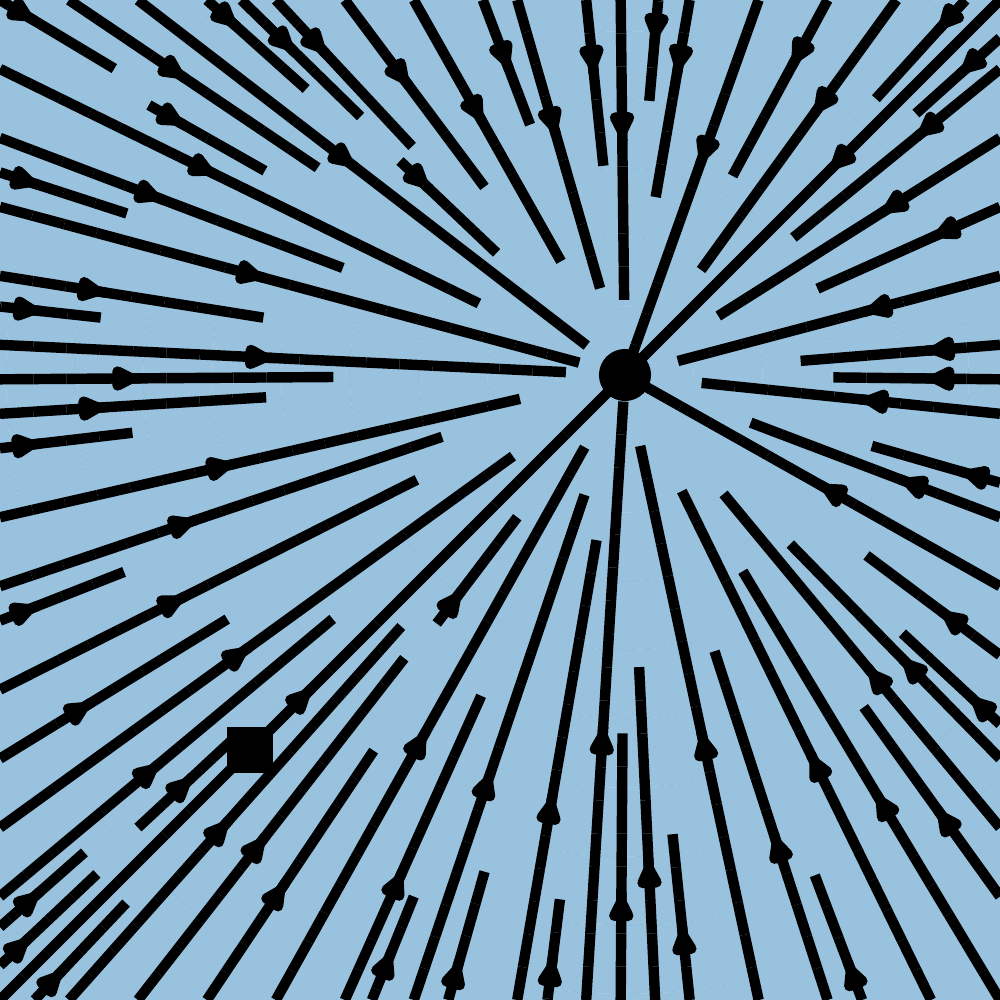}
        \vspace{-1.5em}
        \caption*{\scalebox{0.8}{$t=1.0$}}
    \end{subfigure}\vspace{-5pt}
\caption*{OT path -- conditional vector field}
\end{subfigure}\vspace{-5pt}
\caption{Compared to the diffusion path's conditional score function, the OT path's conditional vector field has constant direction in time and is arguably simpler to fit with a parametric model. Note the blue color denotes larger magnitude while red color denotes smaller magnitude.\vspace{-5pt}}
\label{fig:2d_fv}
\end{figure}

\paragraph{Example II: Optimal Transport conditional VFs.}
An arguably more natural choice for conditional probability paths is to define the mean and the std to simply change linearly in time, \ie, 
\begin{equation}
    \mu_t(x) = tx_1, \text{ and } \sigma_t(x) = 1-(1-\sigmamin)t.
\end{equation}
According to Theorem \ref{thm:cond_vf} this path is generated by the VF
\begin{equation}\label{e:ut_ot}
    u_t(x|x_1)= \frac{x_1-(1-\sigmamin)x}{1-(1-\sigmamin)t},
\end{equation}
which, in contrast to the diffusion conditional VF (\eqref{e:ut_dif_with_our_method}), is defined for all $t\in[0,1]$. The conditional flow that corresponds to $u_t(x|x_1)$ is \begin{equation}\label{e:phi_t_ot}
    \psi_t(x) = (1-(1-\sigmamin)t)x + tx_1,
\end{equation}
and in this case, the CFM loss (see equations \ref{e:cfm}, \ref{e:cfm_dt_psi}) takes the form: 
\begin{equation}\label{e:cfm_cond_ot}
     \gL_{\CFM}(\theta) 
     =\E_{t, q(x_1), p(x_0)} \Big\|v_t(\psi_t(x_0)) - \Big (x_1 - (1-\sigma_{\min})x_0\Big) \Big\|^2.     
\end{equation}
Allowing the mean and std to change linearly not only leads to simple and intuitive paths, but it is actually also optimal in the following sense.
{The conditional flow $\psi_t(x)$} is in fact the Optimal Transport (OT) \emph{displacement map} between the two Gaussians $p_0(x|x_1)$ and $p_1(x|x_1)$. The OT \emph{interpolant}, which is a probability path, is defined to be (see Definition 1.1 in \cite{mccann1997convexity}):
\begin{equation}
    p_t = [(1-t)\mathrm{id} + t\psi]_\star p_0
\end{equation}
where $\psi:\Real^d\too\Real^d$ is the OT map pushing $p_0$ to $p_1$, $\mathrm{id}$ denotes the identity map, \ie, $\mathrm{id}(x)=x$, and $(1-t)\mathrm{id} + t\psi$ is called the OT displacement map. Example 1.7 in \cite{mccann1997convexity} shows,  that in our case of two Gaussians where the first is a standard one, the OT displacement map takes the form of \eqref{e:phi_t_ot}. 

\begin{wrapfigure}[7]{r}{0.30\textwidth}
\vspace{-30pt}
  \begin{center}
  \begin{tabular}{cc}
      \includegraphics[width=0.13\textwidth]{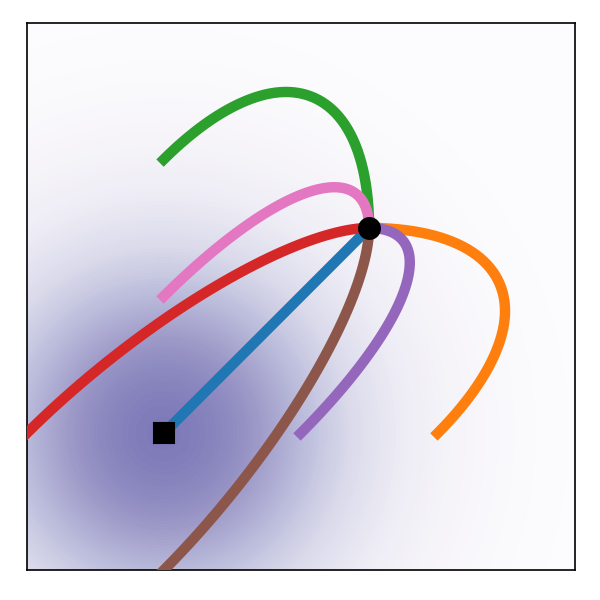} & \includegraphics[width=0.13\textwidth]{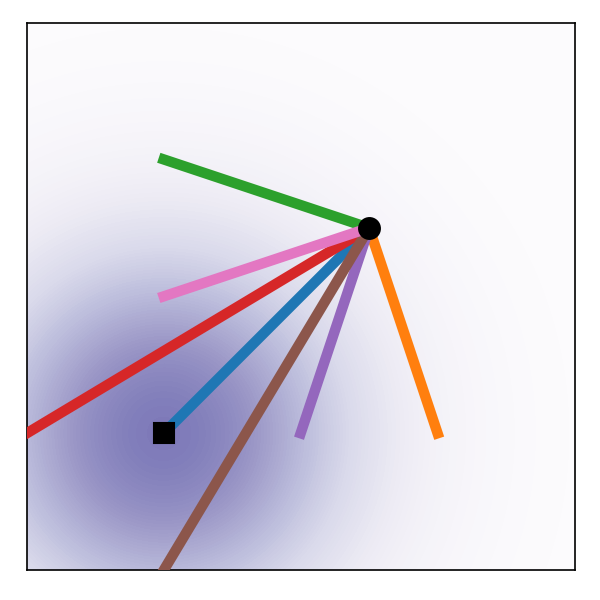}\vspace{-5pt} \\
      {\scriptsize Diffusion} & {\scriptsize OT}  \vspace{-10pt}
  \end{tabular}
  \end{center}
  \caption{Diffusion and OT trajectories. 
}\label{fig:trajectories}
\end{wrapfigure}
Intuitively, particles under the OT displacement map always move in straight line trajectories and with constant speed.
Figure \ref{fig:trajectories} depicts sampling paths for the diffusion and OT conditional VFs. 
Interestingly, we find that sampling trajectory from diffusion paths can ``overshoot'' the final sample, resulting in unnecessary backtracking, whilst the OT paths are guaranteed to stay straight.

Figure \ref{fig:2d_fv} compares the diffusion conditional score function (the regression target in a typical diffusion methods), \ie, $\nabla \log p_t(x|x_1)$ with $p_t$ defined as in \eqref{e:VP_diffusion_path}, with the OT conditional VF (\eqref{e:ut_ot}). 
The start ($p_0$) and end ($p_1$) Gaussians are identical in both examples. 
An interesting observation is that the OT VF has a constant direction in time, which arguably leads to a simpler regression task. This property can also be verified directly from \eqref{e:ut_ot} as the VF can be written in the form $u_t(x|x_1)=g(t)h(x|x_1)$. Figure \ref{fig:2d_diff_fv} in the Appendix shows a visualization of the Diffusion VF. 
%
Lastly, we note that although the conditional flow is optimal, this by no means imply that the marginal VF is an optimal transport solution. Nevertheless, we expect the marginal vector field to remain relatively simple.

\section{Related Work}

Continuous Normalizing Flows were introduced in \citep{chen2018neural} as a continuous-time version of Normalizing Flows (see \eg, \citet{kobyzev2020normalizing,papamakarios2021normalizing} for an overview). Originally, CNFs are trained with the maximum likelihood objective, but this involves expensive ODE simulations for the forward and backward propagation, resulting in high time complexity due to the sequential nature of ODE simulations. Although some works demonstrated the capability of CNF generative models for image synthesis \citep{ffjord2018}, scaling up to very high dimensional images is inherently difficult. 
A number of works attempted to regularize the ODE to be easier to solve, \eg, using augmentation~\citep{dupont2019aug}, adding regularization terms~\citep{yang2019potential,finlay2020how,onken2021ot-flow,tong2020trajectorynet,kelly2020learning}, or stochastically sampling the integration interval \citep{du2022toflow}. These works merely aim to regularize the ODE but do not change the fundamental training algorithm.

In order to speed up CNF training, some works have developed simulation-free CNF training frameworks by explicitly designing the target probability path and the dynamics. For instance, \citet{rozen2021moser} consider a linear interpolation between the prior and the target density but involves integrals that were difficult to estimate in high dimensions, while \citet{ben2022matching} consider general probability paths similar to this work but suffers from biased gradients in the stochastic minibatch regime.
In contrast, the Flow Matching framework allows simulation-free training with unbiased gradients and readily scales to very high dimensions.



Another approach to simulation-free training relies on the construction of a diffusion process to indirectly define the target probability path~\citep{sohl2015deep,ho2020denoising,song2019score}.
\citet{song2020score} shows that diffusion models are trained using denoising score matching~\citep{vincent2011connection}, a conditional objective that provides unbiased gradients with respect to the score matching objective. 
Conditional Flow Matching draws inspiration from this result, but generalizes to matching vector fields directly.
Due to the ease of scalability, diffusion models have received increased attention, producing a variety of improvements such as loss-rescaling~\citep{song2021maximum}, adding classifier guidance along with architectural improvements~\citep{dhariwal2021diffusion}, and learning the noise schedule~\citep{nichol2021improved,kingma2021vdm}. However, \citep{nichol2021improved} and \citep{kingma2021vdm} only consider a restricted setting of Gaussian conditional paths defined by simple diffusion processes with a single parameter---in particular, it does not include our conditional OT path.
%
{In an another line of works, \citep{DeBortoli2021schscore, wang2021schbridges,peluchetti2021non} proposed finite time diffusion constructions via diffusion bridges theory resolving the approximation error incurred by infinite time denoising constructions.}
While existing works make use of a connection between diffusion processes and continuous normalizing flows with the same probability path~\citep{maoutsa2020interacting,song2020score,song2021maximum}, our work allows us to generalize beyond the class of probability paths modeled by simple diffusion.
With our work, it is possible to completely sidestep the diffusion process construction and reason directly with probability paths, while still retaining efficient training and log-likelihood evaluations.  Lastly, concurrently to our work \citep{liu2022flow,albergo2022building} arrived at similar conditional objectives for simulation-free training of CNFs, while \citet{neklyudov2023action} derived an implicit objective when $u_t$ is assumed to be a gradient field.\looseness=-1


\begin{figure}
\centering
\begin{tabular}{@{}c@{\hspace{4pt}}|@{\hspace{4pt}}c@{\hspace{2pt}}c@{\hspace{2pt}}c@{\hspace{2pt}}c@{\hspace{2pt}}c@{}}
    \includegraphics[width=0.68\textwidth]{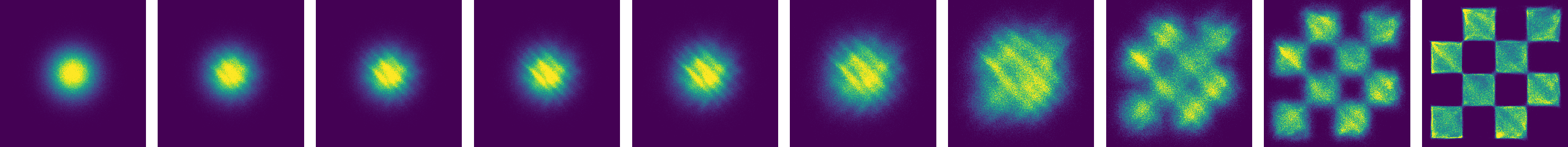} & \rot{\tiny{SM \textsuperscript{w}/ Dif}} &
      \includegraphics[width=0.0632\textwidth]{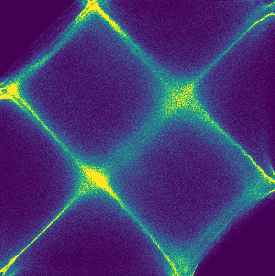} & 
      \includegraphics[width=0.0632\textwidth]{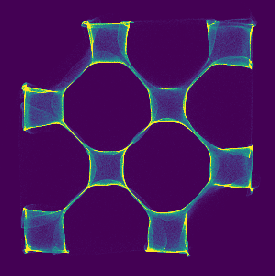} &
      \includegraphics[width=0.0632\textwidth]{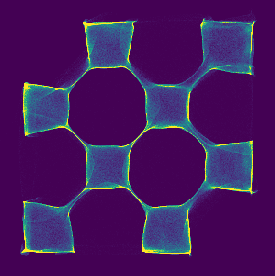} &
      \includegraphics[width=0.0632\textwidth]{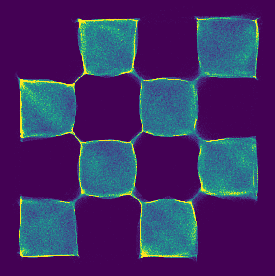}  \vspace{-3pt} \\
    \scriptsize{Score matching \textsuperscript{w}/ Diffusion} & & & & &  \\
    \includegraphics[width=0.68\textwidth]{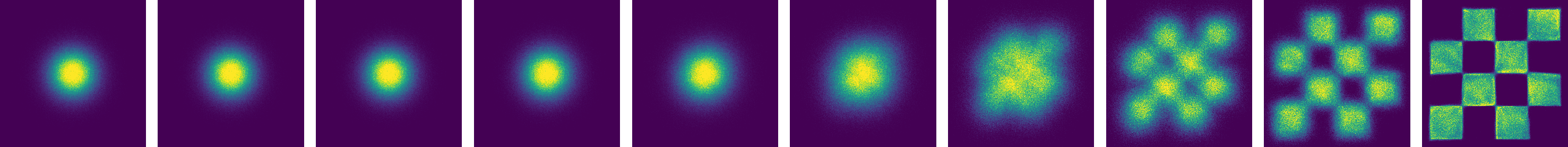} &
    \rot{\tiny{FM \textsuperscript{w}/ Dif}} &
      \includegraphics[width=0.0632\textwidth]{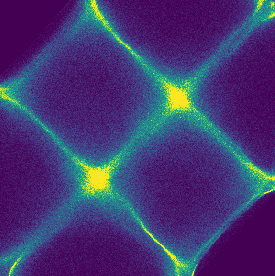} & 
      \includegraphics[width=0.0632\textwidth]{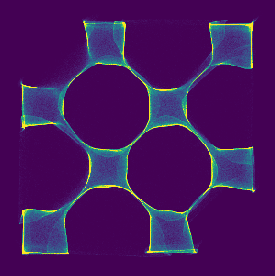} &
      \includegraphics[width=0.0632\textwidth]{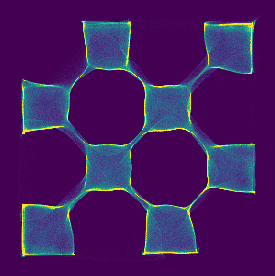} &
      \includegraphics[width=0.0632\textwidth]{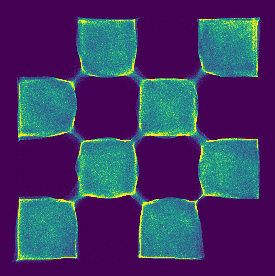} \vspace{-3pt}
      \\
    \scriptsize{Flow Matching \textsuperscript{w}/ Diffusion } & & & & & \\
    \includegraphics[width=0.68\textwidth]{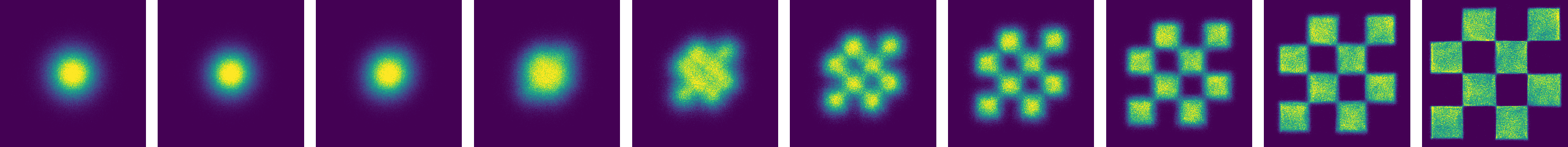} &
    \rot{\tiny{FM \textsuperscript{w}/ OT}} &
      \includegraphics[width=0.0632\textwidth]{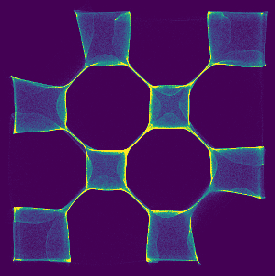} & 
      \includegraphics[width=0.0632\textwidth]{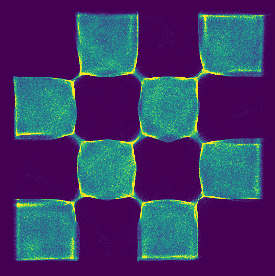} &
      \includegraphics[width=0.0632\textwidth]{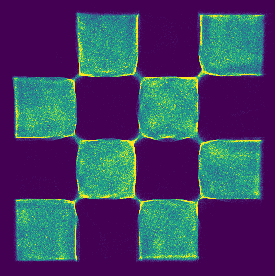} &
      \includegraphics[width=0.0632\textwidth]{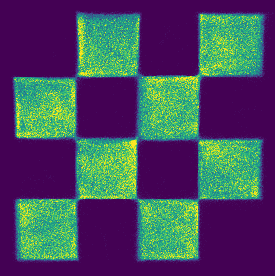} \vspace{-3pt}  \\
      \scriptsize{Flow Matching \textsuperscript{w}/ OT } & & {\scriptsize NFE=4} & {\scriptsize NFE=8} & {\scriptsize NFE=10} & {\scriptsize NFE=20}
\end{tabular}
    \caption{(\emph{left}) Trajectories of CNFs trained with different objectives on 2D checkerboard data. 
    The OT path introduces the checkerboard pattern much earlier, while FM results in more stable training.
    (\emph{right}) FM with OT results in more efficient sampling, solved using the midpoint scheme.\vspace{-10pt}}
    \label{fig:2d_checkerboard}
\end{figure}

\section{Experiments}

We explore the empirical benefits of using Flow Matching on the image datasets of  CIFAR-10~\citep{krizhevsky2009learning} and ImageNet at resolutions 32, 64, and 128~\citep{chrabaszcz2017downsampled,deng2009-imagenet}. We also ablate the choice of diffusion path in Flow Matching, particularly between the standard variance preserving diffusion path and the optimal transport path. 
We discuss how sample generation is improved by directly parameterizing the generating vector field and using the Flow Matching objective. Lastly we show Flow Matching can also be used in the conditional generation setting. Unless otherwise specified, we evaluate likelihood and samples from the model using \texttt{dopri5} \citep{dormand1980family} at absolute and relative tolerances of 1e-5. Generated samples can be found in the Appendix, and all implementation details are in Appendix \ref{A:implementation-details}.\vspace{-5pt}


\begin{table}\centering
\ra{1.05}
\setlength{\tabcolsep}{2.0pt}
\resizebox{\linewidth}{!}{%
\begin{tabular}{@{} l @{} R{3.0em} R{3.0em} R {3.0em} r @{} R{3.0em} R{3.0em} R{3.0em} r @{} R{3.0em} R{3.0em} R{3.0em} r @{}}\toprule
  & \multicolumn{3}{c}{\bf CIFAR-10} &
  & \multicolumn{3}{c}{\bf ImageNet 32$\times$32} &
  & \multicolumn{3}{c}{\bf ImageNet 64$\times$64} \\
\cmidrule(lr){2-4} \cmidrule(lr){6-8} \cmidrule(lr){10-12}
Model & {NLL$\downarrow$} & {FID$\downarrow$} & {NFE$\downarrow$}
& & {NLL$\downarrow$} & {FID$\downarrow$} & {NFE$\downarrow$}
& & {NLL$\downarrow$} & {FID$\downarrow$} & {NFE$\downarrow$} \\
\cmidrule(r){1-1}\cmidrule(lr){2-4} \cmidrule(lr){6-8} \cmidrule(lr){10-12}
\textit{\small Ablations}\\
\;\; DDPM   & 
3.12 & 7.48 & 274 & & 
3.54 & 6.99 & 262
 & & 
3.32 & 17.36 & 264 \\
\;\; Score Matching & 
3.16 & 19.94 & 242 & & 
3.56 & 5.68 & 178 & & 
3.40 & 19.74 & 441 \\
\;\; ScoreFlow  & 
3.09 & 20.78 & 428 & & 
3.55 & 14.14  & 195 & & 
3.36 & 24.95 & 601 \\

\cmidrule(r){1-1}\cmidrule(lr){2-4} \cmidrule(lr){6-8} \cmidrule(lr){10-12}%
\textit{\small Ours}\\
\;\; FM \textsuperscript{w}/ Diffusion  &  
3.10 & 8.06 & 183 & & 
3.54 & 6.37 & 193 & & 
3.33 & 16.88 & 187 \\
\;\; FM \textsuperscript{w}/ OT  & 
\textbf{2.99} & \textbf{6.35} & \textbf{142} & & 
\textbf{3.53} & \textbf{5.02} & \textbf{122} & & 
\textbf{3.31} & \textbf{14.45} & \textbf{138} \\
\bottomrule
\end{tabular}\quad 
\begin{tabular}{@{} l @{} R{2.8em} R{2.8em} r @{}}\toprule
  & \multicolumn{2}{c}{\bf ImageNet 128$\times$128} \\
\cmidrule(lr){2-3} 
Model & {NLL$\downarrow$} & {FID$\downarrow$} \\
\cmidrule(r){1-1} \cmidrule(lr){2-3} 
MGAN~{\tiny\citep{hoang2018mgan}} & 
-- & 58.9 \\
PacGAN2~{\tiny\citep{lin2018pacgan}} & 
-- & 57.5 \\
Logo-GAN-AE~{\tiny\citep{sage2018logo}} & 
-- & 50.9 \\
Self-cond. GAN~{\tiny\citep{luvcic2019high}} & 
-- & 41.7 \\
Uncond. BigGAN~{\tiny\citep{luvcic2019high}} & 
-- & 25.3 \\
PGMGAN~{\tiny\citep{armandpour2021partition}}  & 
-- & 21.7 \\
\cmidrule(r){1-1} \cmidrule(lr){2-3}%
FM \textsuperscript{w}/ OT  & 
\textbf{2.90} & \textbf{20.9} \\
\bottomrule
\end{tabular}\vspace{-5pt}
}
\caption{Likelihood (BPD),  quality of generated samples (FID), and evaluation time (NFE) for the same model trained with different methods. \vspace{-5pt}}
\label{tab:img_results}
\end{table}

\pagebreak
\subsection{Density Modeling and Sample Quality on ImageNet}\vspace{-5pt}

{We start by comparing the same model architecture, \ie, the U-Net architecture from \citet{dhariwal2021diffusion} with minimal changes, trained on CIFAR-10, and ImageNet 32/64 with different popular diffusion-based losses: DDPM from \citep{ho2020denoising}, Score Matching (SM) \citep{song2020score}, and Score Flow (SF) \citep{song2021maximum}; see Appendix \ref{a:dif_baselines} for exact details.} Table \ref{tab:img_results} (left)  summarizes our results alongside these baselines reporting negative log-likelihood (NLL) in units of bits per dimension (BPD), sample quality as measured by the Frechet Inception Distance (FID; \citet{heusel2017gans}), and averaged number of function evaluations (NFE) required for the adaptive solver to reach its a prespecified numerical tolerance, averaged over 50k samples. All models are trained using the same architecture, hyperparameter values and number of training iterations, where baselines are allowed more iterations for better convergence. Note that these are \emph{unconditional} models. 
%
{On both CIFAR-10 and ImageNet, FM-OT consistently obtains best results across all our quantitative measures compared to competing methods. We are noticing a higher that usual FID performance in CIFAR-10 compared to previous works \citep{ho2020denoising,song2020score,song2021maximum} that can possibly be explained by the fact that our used architecture was not optimized for CIFAR-10.}

{Secondly, Table \ref{tab:img_results} (right) compares a model trained using Flow Matching with the OT path on ImageNet at resolution 128$\times$128. Our FID is state-of-the-art with the exception of IC-GAN~\citep{casanova2021instance} which uses conditioning with a self-supervised ResNet50 model, and therefore is left out of this table. } Figures \ref{fig:imagenet32_samples}, \ref{fig:imagenet64_samples}, \ref{fig:imagenet128_samples} in the Appendix show non-curated samples from these models.

\begin{wrapfigure}[10]{r}{0.35\textwidth}
\vspace{-19pt}
  \begin{center}
  \includegraphics[width=0.35\textwidth]{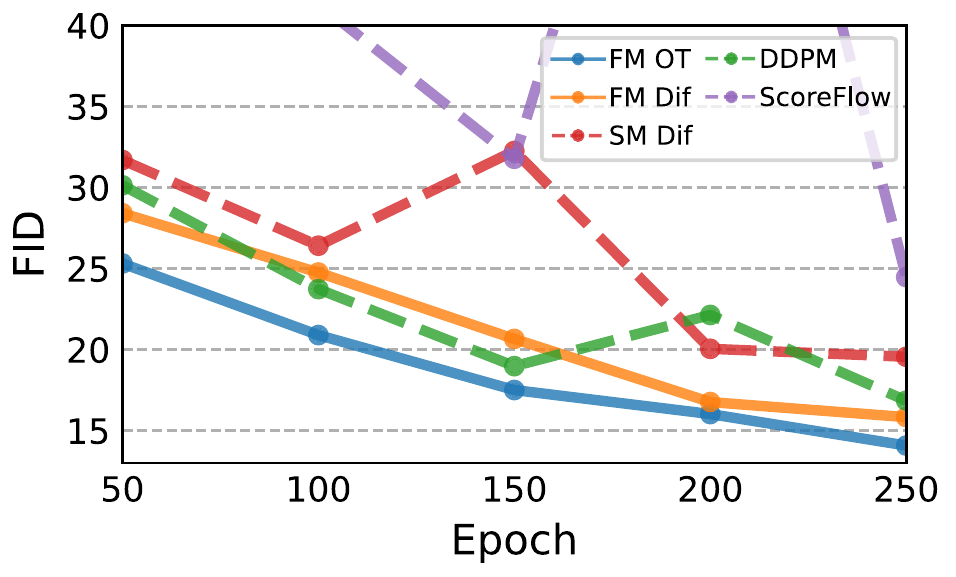} 
  \end{center}\vspace{-14pt}
  \caption{{Image quality during training, ImageNet 64$\times$64.}} \label{fig:FID_vs_epochs}
\end{wrapfigure}
\textbf{Faster training.} 
While existing works train diffusion models with a very high number of iterations (\eg, 1.3m and 10m iterations are reported by Score Flow and VDM, respectively), we  find that Flow Matching generally converges much faster. 
{Figure \ref{fig:FID_vs_epochs} shows FID curves during training of Flow Matching and all baselines for ImageNet 64$\times$64; FM-OT is able to lower the FID faster and to a greater extent than the alternatives. } 
%
%
%
For ImageNet-128 \cite{dhariwal2021diffusion} train for 4.36m iterations with batch size 256, while FM (with 25\% larger model) used 500k iterations with batch size 1.5k, \ie, 33\% less image throughput; see Table \ref{tab:hyper-params} for exact details. 
Furthermore, the cost of sampling from a model can drastically change during training for score matching, whereas the sampling cost stays constant when training with Flow Matching (Figure \ref{fig:2d_NFE_vs_epochs} in Appendix). \vspace{-5pt}

\subsection{Sampling Efficiency}\vspace{-5pt}

\begin{figure}
    \centering
    \begin{subfigure}[b]{0.32\linewidth}
        \includegraphics[width=\linewidth]{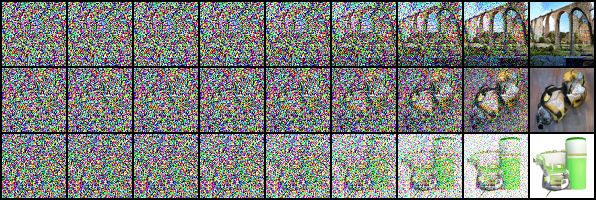}
        \caption*{\scriptsize Score Matching w/ Diffusion}
    \end{subfigure}
    \begin{subfigure}[b]{0.32\linewidth}
        \includegraphics[width=\linewidth]{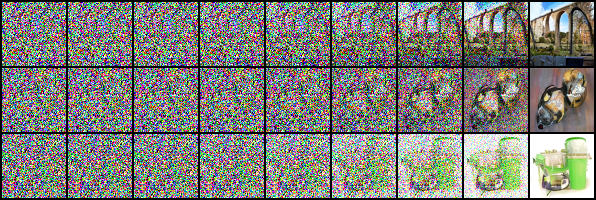}
        \caption*{\scriptsize Flow Matching w/ Diffusion}
    \end{subfigure} 
    \begin{subfigure}[b]{0.32\linewidth}
        \includegraphics[width=\linewidth]{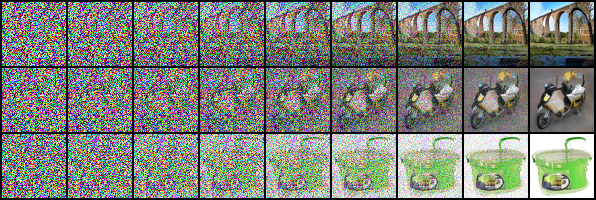}
        \caption*{\scriptsize Flow Matching w/ OT}
    \end{subfigure}\vspace{-5pt}
    \caption{Sample paths from the same initial noise with models trained on ImageNet 64$\times$64. The OT path reduces noise roughly linearly, while diffusion paths visibly remove noise only towards the end of the path. Note also the differences between the generated images.}\vspace{-10pt}
    \label{fig:sample_paths}
\end{figure}

For sampling, we first draw a random noise sample $x_0 \sim \gN(0,I)$ then compute $\phi_1(x_0)$ by solving \eqref{e:ode} with the trained VF, $v_t$, on the interval $t\in [0, 1]$ using an ODE solver. While diffusion models can also be sampled through an SDE formulation, this can be highly inefficient and many methods that propose fast samplers (\eg, ~\citet{song2020denoising,zhang2022fast}) directly make use of the ODE perspective (see Appendix \ref{A:diff_cond_VP}). 
In part, this is due to ODE solvers being much more efficient---yielding lower error at similar computational costs~\citep{kloeden2012numerical}---and the multitude of available ODE solver schemes. 
When compared to our ablation models, we find that models trained using Flow Matching with the OT path always result in the most efficient sampler, regardless of ODE solver, as demonstrated next. 

\textbf{Sample paths.} We first qualitatively visualize the difference in sampling paths between diffusion and OT. Figure \ref{fig:sample_paths} shows samples from ImageNet-64 models using identical random seeds, where we find that the OT path model starts generating images sooner than the diffusion path models, where noise dominates the image until the very last time point.
We additionally depict the probability density paths in 2D generation of a checkerboard pattern, Figure \ref{fig:2d_checkerboard} (left), noticing a similar trend.

%

\textbf{Low-cost samples.} We next switch to fixed-step solvers and compare low ($\leq$100) NFE samples computed with the ImageNet-32 models from Table \ref{tab:img_results}. 
In Figure \ref{fig:fid_vs_nfe} (left), we compare the per-pixel MSE of low NFE solutions compared with 1000 NFE solutions (we use 256 random noise seeds), and notice that the FM with OT model produces the best numerical error, in terms of computational cost, requiring roughly only 60\% of the NFEs to reach the same error threshold as diffusion models. Secondly, Figure \ref{fig:fid_vs_nfe} (right) shows how FID changes as a result of the computational cost, where we find FM with OT is able to achieve decent FID even at very low NFE values, producing better trade-off between sample quality and cost compared to ablated models. Figure \ref{fig:2d_checkerboard} (right) shows low-cost sampling effects for the 2D checkerboard experiment. \vspace{-3pt}

\begin{figure}
    \centering
    \begin{tabular}{@{}c@{\hspace{5pt}}|@{\hspace{5pt}}c@{\hspace{2pt}}c@{\hspace{2pt}}c@{}}
    \includegraphics[width=0.23\textwidth]{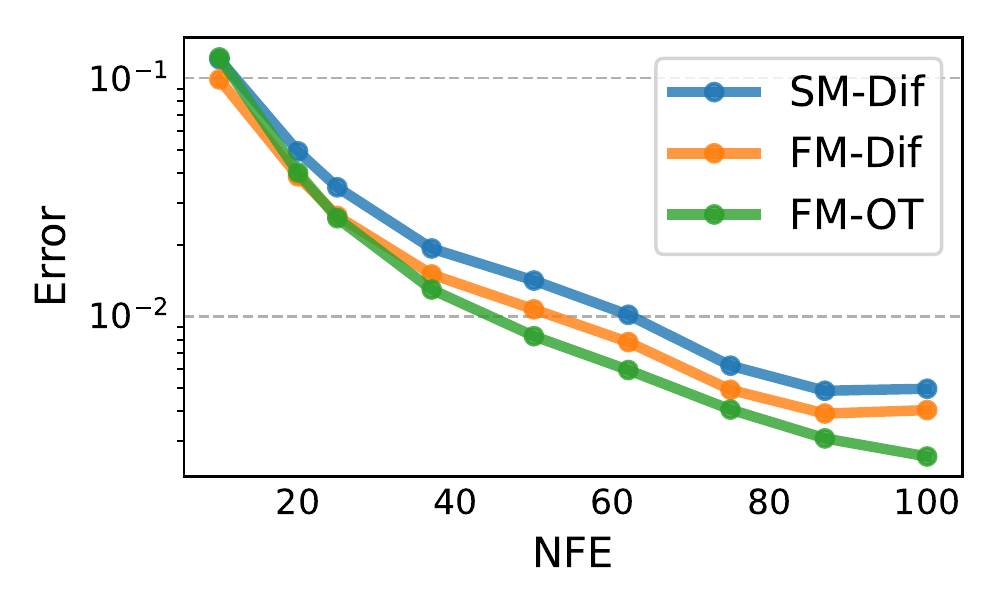} &
         \includegraphics[width=0.23\textwidth]{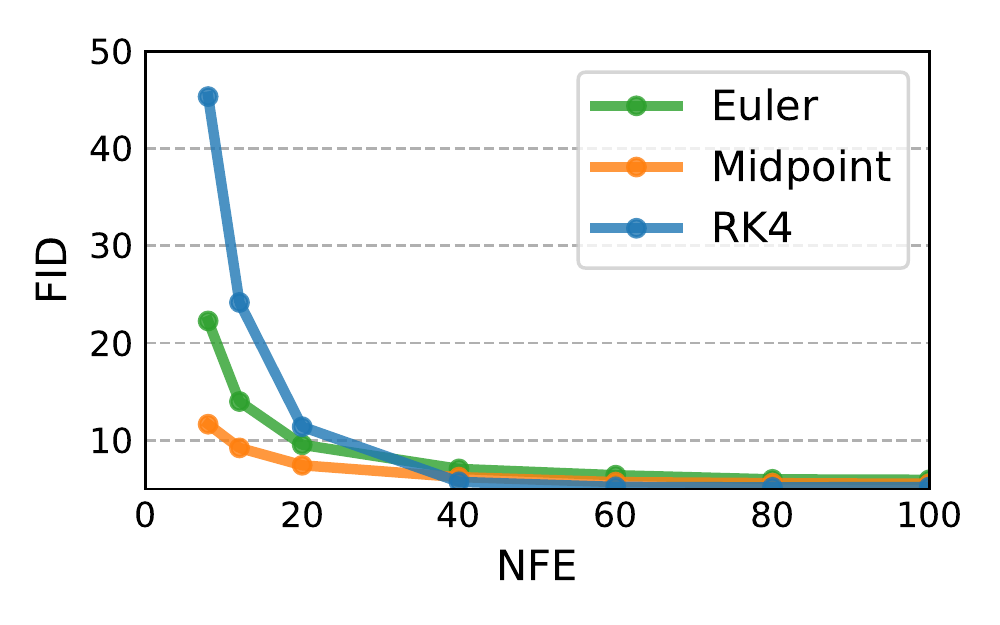} & \includegraphics[width=0.23\textwidth]{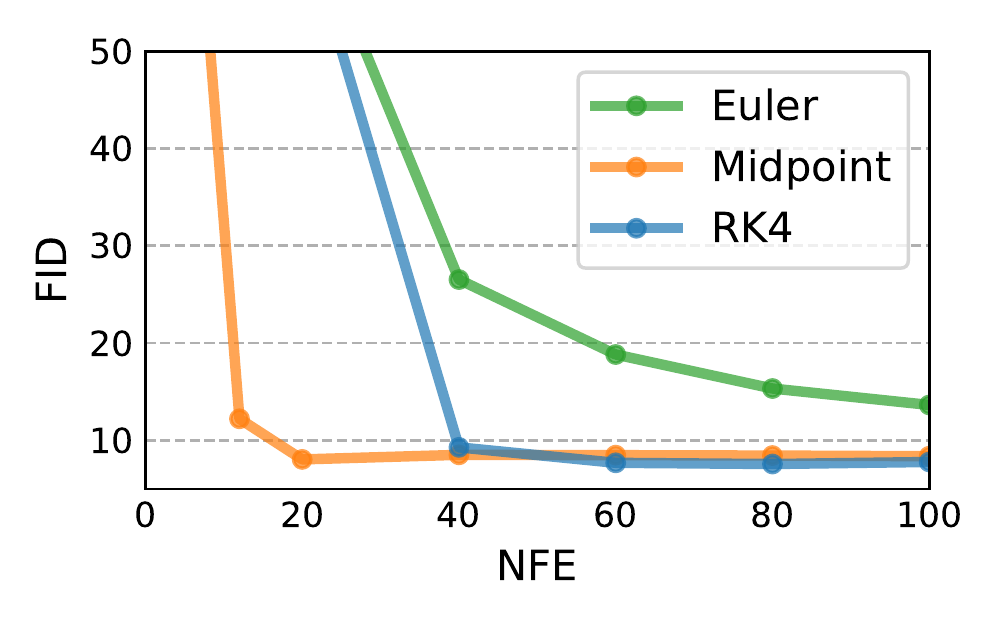} &
         \includegraphics[width=0.23\textwidth]{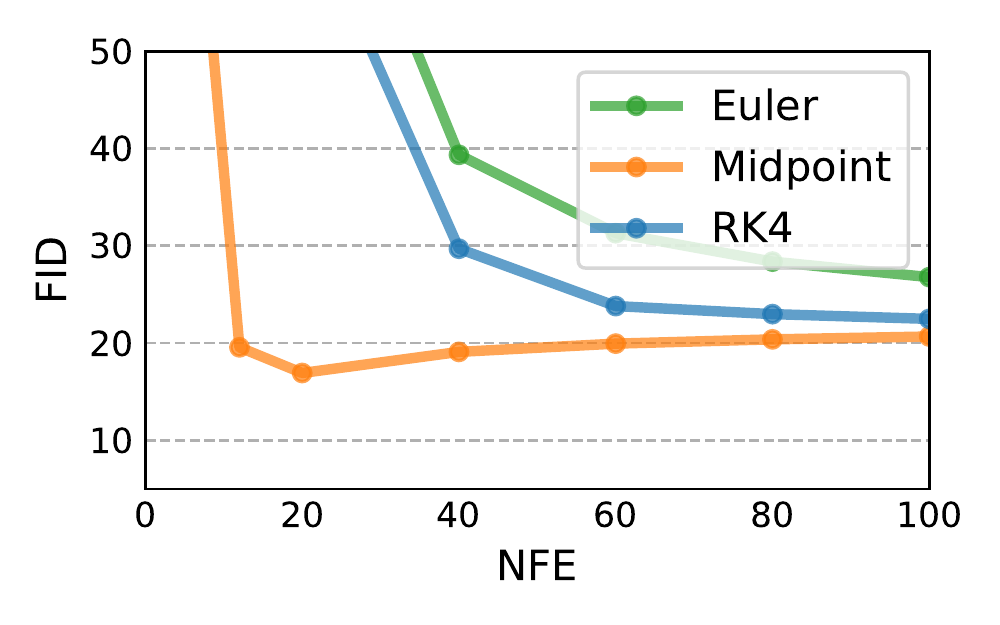}\vspace{-5pt}
         \\
         \quad \ \  {\scriptsize Error of ODE solution }
         & \ \ \
         {\scriptsize Flow matching \textsuperscript{w}/ OT}
         &  \ \ 
         {\scriptsize Flow matching \textsuperscript{w}/ Diffusion}
         & \ \ 
         {\scriptsize Score matching \textsuperscript{w}/ Diffusion}\vspace{-5pt}
    \end{tabular}
    \caption{Flow Matching, especially when using OT paths, allows us to use fewer evaluations for sampling while retaining similar numerical error (left) and sample quality (right). Results are shown for models trained on ImageNet 32$\times$32, and numerical errors are for the midpoint scheme.\vspace{-8pt}}
    \label{fig:fid_vs_nfe}
\end{figure}

\subsection{Conditional sampling from low-resolution images}

\begin{wraptable}[9]{r}{5.5cm}\vspace{-10pt}
\ra{1.05}
\setlength{\tabcolsep}{2.0pt}
\resizebox{5.5cm}{!}{
\begin{tabular}{@{} l @{} R{3.0em} R{3.3em} R{3.3em} R{3.3em} }\toprule 
Model &  {FID$\downarrow$} & {IS$\uparrow$} & {PSNR$\uparrow$} & {SSIM$\uparrow$} \\
\cmidrule(r){1-1} \cmidrule(lr){2-5}
Reference & 
1.9 & 240.8 & {--} & -- \\
\cmidrule(r){1-1} \cmidrule(lr){2-5}%
Regression & 
15.2 & 121.1 & \textbf{27.9} & \textbf{0.801} \\
SR3~{\tiny\citep{saharia2022image}} & 
5.2 & 180.1 & 26.4 & 0.762 \\
\cmidrule(r){1-1} \cmidrule(lr){2-5}%
FM \textsuperscript{w}/ OT  & 
\textbf{3.4} & \textbf{200.8} & 24.7 & 0.747  \\
\bottomrule
\end{tabular} }
\caption{Image super-resolution on the ImageNet validation set.}\label{tab:super_resolution}
\end{wraptable} 
{Lastly, we experimented with Flow Matching for conditional image generation. In particular, upsampling images from 64$\times$64 to 256$\times$256.
We follow the evaluation procedure in \citep{saharia2022image} and compute the FID of the upsampled validation images; baselines include reference (FID of original validation set), and regression. Results are in Table \ref{tab:super_resolution}. Upsampled image samples are shown in Figures \ref{fig:upsampled_1}, \ref{fig:upsampled_2} in the Appendix. FM-OT achieves similar PSNR and SSIM values to \citep{saharia2022image} while considerably improving on FID and IS, which as argued by \citep{saharia2022image} is a better indication of generation quality.

}


\section{Conclusion}\vspace{-3pt}
We introduced Flow Matching, a new simulation-free framework for training Continuous Normalizing Flow models, relying on conditional constructions to effortlessly scale to very high dimensions.
Furthermore, the FM framework provides an alternative view on diffusion models, and suggests forsaking the stochastic/diffusion construction in favor of more directly specifying the probability path, allowing us to, \eg, construct paths that allow faster sampling and/or improve generation.
We experimentally showed the ease of training and sampling when using the Flow Matching framework, and in the future, we expect FM to open the door to allowing a multitude of probability paths (\eg, non-isotropic Gaussians or more general kernels altogether). 

\newpage

\section{Social responsibility}
Along side its many positive applications, image generation can also be used for harmful proposes. Using content-controlled training sets and image validation/classification can help reduce these uses. Furthermore, the energy demand for training large deep learning models is increasing at a rapid pace~\citep{openaiandcompute,thompson2020computational}, focusing on methods that are able to train using less gradient updates / image throughput can lead to significant time and energy savings.

\bibliography{fm_arxiv_v2}

\begin{thebibliography}{52}
\providecommand{\natexlab}[1]{#1}
\providecommand{\url}[1]{\texttt{#1}}
\expandafter\ifx\csname urlstyle\endcsname\relax
  \providecommand{\doi}[1]{doi: #1}\else
  \providecommand{\doi}{doi: \begingroup \urlstyle{rm}\Url}\fi

\bibitem[Albergo \& Vanden-Eijnden(2022)Albergo and
  Vanden-Eijnden]{albergo2022building}
Michael~S Albergo and Eric Vanden-Eijnden.
\newblock Building normalizing flows with stochastic interpolants.
\newblock \emph{arXiv preprint arXiv:2209.15571}, 2022.

\bibitem[Amodei et~al.(2018)Amodei, Hernandez, SastryJack, Clark, Brockman, and
  Sutskever]{openaiandcompute}
Dario Amodei, Danny Hernandez, Girish SastryJack, Jack Clark, Greg Brockman,
  and Ilya Sutskever.
\newblock Ai and compute.
\newblock \url{https://openai.com/blog/ai-and-compute/}, 2018.

\bibitem[Armandpour et~al.(2021)Armandpour, Sadeghian, Li, and
  Zhou]{armandpour2021partition}
Mohammadreza Armandpour, Ali Sadeghian, Chunyuan Li, and Mingyuan Zhou.
\newblock Partition-guided gans.
\newblock In \emph{Proceedings of the IEEE/CVF Conference on Computer Vision
  and Pattern Recognition}, pp.\  5099--5109, 2021.

\bibitem[Ben-Hamu et~al.(2022)Ben-Hamu, Cohen, Bose, Amos, Grover, Nickel,
  Chen, and Lipman]{ben2022matching}
Heli Ben-Hamu, Samuel Cohen, Joey Bose, Brandon Amos, Aditya Grover, Maximilian
  Nickel, Ricky T.~Q. Chen, and Yaron Lipman.
\newblock Matching normalizing flows and probability paths on manifolds.
\newblock \emph{arXiv preprint arXiv:2207.04711}, 2022.

\bibitem[Casanova et~al.(2021)Casanova, Careil, Verbeek, Drozdzal, and
  Romero~Soriano]{casanova2021instance}
Arantxa Casanova, Marlene Careil, Jakob Verbeek, Michal Drozdzal, and Adriana
  Romero~Soriano.
\newblock Instance-conditioned gan.
\newblock \emph{Advances in Neural Information Processing Systems},
  34:\penalty0 27517--27529, 2021.

\bibitem[Chen(2018)]{torchdiffeq}
Ricky T.~Q. Chen.
\newblock torchdiffeq, 2018.
\newblock URL \url{https://github.com/rtqichen/torchdiffeq}.

\bibitem[Chen et~al.(2018)Chen, Rubanova, Bettencourt, and
  Duvenaud]{chen2018neural}
Ricky T.~Q. Chen, Yulia Rubanova, Jesse Bettencourt, and David~K Duvenaud.
\newblock Neural ordinary differential equations.
\newblock \emph{Advances in neural information processing systems}, 31, 2018.

\bibitem[Chrabaszcz et~al.(2017)Chrabaszcz, Loshchilov, and
  Hutter]{chrabaszcz2017downsampled}
Patryk Chrabaszcz, Ilya Loshchilov, and Frank Hutter.
\newblock A downsampled variant of imagenet as an alternative to the cifar
  datasets.
\newblock \emph{arXiv preprint arXiv:1707.08819}, 2017.

\bibitem[De~Bortoli et~al.(2021)De~Bortoli, Thornton, Heng, and
  Doucet]{DeBortoli2021schscore}
Valentin De~Bortoli, James Thornton, Jeremy Heng, and Arnaud Doucet.
\newblock Diffusion schr\"odinger bridge with applications to score-based
  generative modeling.
\newblock \penalty0 (arXiv:2106.01357), Dec 2021.
\newblock \doi{10.48550/arXiv.2106.01357}.
\newblock URL \url{http://arxiv.org/abs/2106.01357}.
\newblock arXiv:2106.01357 [cs, math, stat].

\bibitem[Deng et~al.(2009)Deng, Dong, Socher, Li, Li, and
  Fei-Fei]{deng2009-imagenet}
Jia Deng, Wei Dong, Richard Socher, Li-Jia Li, Kai Li, and Li~Fei-Fei.
\newblock Imagenet: A large-scale hierarchical image database.
\newblock In \emph{2009 IEEE Conference on Computer Vision and Pattern
  Recognition}, pp.\  248--255, 2009.
\newblock \doi{10.1109/CVPR.2009.5206848}.

\bibitem[Dhariwal \& Nichol(2021)Dhariwal and Nichol]{dhariwal2021diffusion}
Prafulla Dhariwal and Alexander~Quinn Nichol.
\newblock Diffusion models beat {GAN}s on image synthesis.
\newblock In A.~Beygelzimer, Y.~Dauphin, P.~Liang, and J.~Wortman Vaughan
  (eds.), \emph{Advances in Neural Information Processing Systems}, 2021.
\newblock URL \url{https://openreview.net/forum?id=AAWuCvzaVt}.

\bibitem[Dormand \& Prince(1980)Dormand and Prince]{dormand1980family}
John~R Dormand and Peter~J Prince.
\newblock A family of embedded runge-kutta formulae.
\newblock \emph{Journal of computational and applied mathematics}, 6\penalty0
  (1):\penalty0 19--26, 1980.

\bibitem[Du et~al.(2022)Du, Luo, Chen, Xu, and Zeng]{du2022toflow}
Shian Du, Yihong Luo, Wei Chen, Jian Xu, and Delu Zeng.
\newblock To-flow: Efficient continuous normalizing flows with temporal
  optimization adjoint with moving speed, 2022.
\newblock URL \url{https://arxiv.org/abs/2203.10335}.

\bibitem[Dupont et~al.(2019)Dupont, Doucet, and Teh]{dupont2019aug}
Emilien Dupont, Arnaud Doucet, and Yee~Whye Teh.
\newblock Augmented neural odes.
\newblock In H.~Wallach, H.~Larochelle, A.~Beygelzimer, F.~d\'{} Alch\'{e}-Buc,
  E.~Fox, and R.~Garnett (eds.), \emph{Advances in Neural Information
  Processing Systems}, volume~32. Curran Associates, Inc., 2019.
\newblock URL
  \url{https://proceedings.neurips.cc/paper/2019/file/21be9a4bd4f81549a9d1d241981cec3c-Paper.pdf}.

\bibitem[Finlay et~al.(2020)Finlay, Jacobsen, Nurbekyan, and
  Oberman]{finlay2020how}
Chris Finlay, Jörn-Henrik Jacobsen, Levon Nurbekyan, and Adam~M. Oberman.
\newblock How to train your neural ode: the world of jacobian and kinetic
  regularization.
\newblock In \emph{ICML}, pp.\  3154--3164, 2020.
\newblock URL \url{http://proceedings.mlr.press/v119/finlay20a.html}.

\bibitem[Grathwohl et~al.(2018)Grathwohl, Chen, Bettencourt, Sutskever, and
  Duvenaud]{ffjord2018}
Will Grathwohl, Ricky T.~Q. Chen, Jesse Bettencourt, Ilya Sutskever, and David
  Duvenaud.
\newblock Ffjord: Free-form continuous dynamics for scalable reversible
  generative models, 2018.
\newblock URL \url{https://arxiv.org/abs/1810.01367}.

\bibitem[Heusel et~al.(2017)Heusel, Ramsauer, Unterthiner, Nessler, and
  Hochreiter]{heusel2017gans}
Martin Heusel, Hubert Ramsauer, Thomas Unterthiner, Bernhard Nessler, and Sepp
  Hochreiter.
\newblock Gans trained by a two time-scale update rule converge to a local nash
  equilibrium.
\newblock \emph{Advances in neural information processing systems}, 30, 2017.

\bibitem[Ho et~al.(2020)Ho, Jain, and Abbeel]{ho2020denoising}
Jonathan Ho, Ajay Jain, and Pieter Abbeel.
\newblock Denoising diffusion probabilistic models.
\newblock \emph{Advances in Neural Information Processing Systems},
  33:\penalty0 6840--6851, 2020.

\bibitem[Hoang et~al.(2018)Hoang, Nguyen, Le, and Phung]{hoang2018mgan}
Quan Hoang, Tu~Dinh Nguyen, Trung Le, and Dinh Phung.
\newblock Mgan: Training generative adversarial nets with multiple generators.
\newblock In \emph{International conference on learning representations}, 2018.

\bibitem[Kelly et~al.(2020)Kelly, Bettencourt, Johnson, and
  Duvenaud]{kelly2020learning}
Jacob Kelly, Jesse Bettencourt, Matthew~J Johnson, and David~K Duvenaud.
\newblock Learning differential equations that are easy to solve.
\newblock \emph{Advances in Neural Information Processing Systems},
  33:\penalty0 4370--4380, 2020.

\bibitem[Kingma et~al.(2021)Kingma, Salimans, Poole, and Ho]{kingma2021vdm}
Diederik~P Kingma, Tim Salimans, Ben Poole, and Jonathan Ho.
\newblock Variational diffusion models.
\newblock In A.~Beygelzimer, Y.~Dauphin, P.~Liang, and J.~Wortman Vaughan
  (eds.), \emph{Advances in Neural Information Processing Systems}, 2021.
\newblock URL \url{https://openreview.net/forum?id=2LdBqxc1Yv}.

\bibitem[Kloeden et~al.(2012)Kloeden, Platen, and Schurz]{kloeden2012numerical}
Peter~Eris Kloeden, Eckhard Platen, and Henri Schurz.
\newblock \emph{Numerical solution of SDE through computer experiments}.
\newblock Springer Science \& Business Media, 2012.

\bibitem[Kobyzev et~al.(2020)Kobyzev, Prince, and
  Brubaker]{kobyzev2020normalizing}
Ivan Kobyzev, Simon~JD Prince, and Marcus~A Brubaker.
\newblock Normalizing flows: An introduction and review of current methods.
\newblock \emph{IEEE transactions on pattern analysis and machine
  intelligence}, 43\penalty0 (11):\penalty0 3964--3979, 2020.

\bibitem[Krizhevsky et~al.(2009)Krizhevsky, Hinton,
  et~al.]{krizhevsky2009learning}
Alex Krizhevsky, Geoffrey Hinton, et~al.
\newblock Learning multiple layers of features from tiny images.
\newblock 2009.

\bibitem[Lin et~al.(2018)Lin, Khetan, Fanti, and Oh]{lin2018pacgan}
Zinan Lin, Ashish Khetan, Giulia Fanti, and Sewoong Oh.
\newblock Pacgan: The power of two samples in generative adversarial networks.
\newblock \emph{Advances in neural information processing systems}, 31, 2018.

\bibitem[Liu et~al.(2022)Liu, Gong, and Liu]{liu2022flow}
Xingchao Liu, Chengyue Gong, and Qiang Liu.
\newblock Flow straight and fast: Learning to generate and transfer data with
  rectified flow.
\newblock \emph{arXiv preprint arXiv:2209.03003}, 2022.

\bibitem[Lu{\v{c}}i{\'c} et~al.(2019)Lu{\v{c}}i{\'c}, Tschannen, Ritter, Zhai,
  Bachem, and Gelly]{luvcic2019high}
Mario Lu{\v{c}}i{\'c}, Michael Tschannen, Marvin Ritter, Xiaohua Zhai, Olivier
  Bachem, and Sylvain Gelly.
\newblock High-fidelity image generation with fewer labels.
\newblock In \emph{International conference on machine learning}, pp.\
  4183--4192. PMLR, 2019.

\bibitem[Maoutsa et~al.(2020{\natexlab{a}})Maoutsa, Reich, and
  Opper]{maoutsa2020}
Dimitra Maoutsa, Sebastian Reich, and Manfred Opper.
\newblock Interacting particle solutions of fokker{\textendash}planck equations
  through gradient{\textendash}log{\textendash}density estimation.
\newblock \emph{Entropy}, 22\penalty0 (8):\penalty0 802, jul
  2020{\natexlab{a}}.
\newblock \doi{10.3390/e22080802}.
\newblock URL \url{https://doi.org/10.3390%2Fe22080802}.

\bibitem[Maoutsa et~al.(2020{\natexlab{b}})Maoutsa, Reich, and
  Opper]{maoutsa2020interacting}
Dimitra Maoutsa, Sebastian Reich, and Manfred Opper.
\newblock Interacting particle solutions of fokker--planck equations through
  gradient--log--density estimation.
\newblock \emph{Entropy}, 22\penalty0 (8):\penalty0 802, 2020{\natexlab{b}}.

\bibitem[McCann(1997)]{mccann1997convexity}
Robert~J McCann.
\newblock A convexity principle for interacting gases.
\newblock \emph{Advances in mathematics}, 128\penalty0 (1):\penalty0 153--179,
  1997.

\bibitem[Neklyudov et~al.(2023)Neklyudov, Severo, and
  Makhzani]{neklyudov2023action}
Kirill Neklyudov, Daniel Severo, and Alireza Makhzani.
\newblock Action matching: A variational method for learning stochastic
  dynamics from samples, 2023.
\newblock URL \url{https://openreview.net/forum?id=T6HPzkhaKeS}.

\bibitem[Nichol \& Dhariwal(2021)Nichol and Dhariwal]{nichol2021improved}
Alexander~Quinn Nichol and Prafulla Dhariwal.
\newblock Improved denoising diffusion probabilistic models.
\newblock In \emph{International Conference on Machine Learning}, pp.\
  8162--8171. PMLR, 2021.

\bibitem[Onken et~al.(2021)Onken, Wu~Fung, Li, and Ruthotto]{onken2021ot-flow}
Derek Onken, Samy Wu~Fung, Xingjian Li, and Lars Ruthotto.
\newblock Ot-flow: Fast and accurate continuous normalizing flows via optimal
  transport.
\newblock \emph{Proceedings of the AAAI Conference on Artificial Intelligence},
  35\penalty0 (10):\penalty0 9223--9232, May 2021.
\newblock URL \url{https://ojs.aaai.org/index.php/AAAI/article/view/17113}.

\bibitem[Papamakarios et~al.(2021)Papamakarios, Nalisnick, Rezende, Mohamed,
  and Lakshminarayanan]{papamakarios2021normalizing}
George Papamakarios, Eric~T Nalisnick, Danilo~Jimenez Rezende, Shakir Mohamed,
  and Balaji Lakshminarayanan.
\newblock Normalizing flows for probabilistic modeling and inference.
\newblock \emph{J. Mach. Learn. Res.}, 22\penalty0 (57):\penalty0 1--64, 2021.

\bibitem[Peluchetti(2021)]{peluchetti2021non}
Stefano Peluchetti.
\newblock Non-denoising forward-time diffusions.
\newblock 2021.

\bibitem[Ramesh et~al.(2022)Ramesh, Dhariwal, Nichol, Chu, and
  Chen]{ramesh2022hierarchical}
Aditya Ramesh, Prafulla Dhariwal, Alex Nichol, Casey Chu, and Mark Chen.
\newblock Hierarchical text-conditional image generation with clip latents.
\newblock \emph{arXiv preprint arXiv:2204.06125}, 2022.

\bibitem[Rombach et~al.(2022)Rombach, Blattmann, Lorenz, Esser, and
  Ommer]{rombach2022high}
Robin Rombach, Andreas Blattmann, Dominik Lorenz, Patrick Esser, and Bj{\"o}rn
  Ommer.
\newblock High-resolution image synthesis with latent diffusion models.
\newblock In \emph{Proceedings of the IEEE/CVF Conference on Computer Vision
  and Pattern Recognition}, pp.\  10684--10695, 2022.

\bibitem[Rozen et~al.(2021)Rozen, Grover, Nickel, and Lipman]{rozen2021moser}
Noam Rozen, Aditya Grover, Maximilian Nickel, and Yaron Lipman.
\newblock Moser flow: Divergence-based generative modeling on manifolds.
\newblock In A.~Beygelzimer, Y.~Dauphin, P.~Liang, and J.~Wortman Vaughan
  (eds.), \emph{Advances in Neural Information Processing Systems}, 2021.
\newblock URL \url{https://openreview.net/forum?id=qGvMv3undNJ}.

\bibitem[Sage et~al.(2018)Sage, Agustsson, Timofte, and Van~Gool]{sage2018logo}
Alexander Sage, Eirikur Agustsson, Radu Timofte, and Luc Van~Gool.
\newblock Logo synthesis and manipulation with clustered generative adversarial
  networks.
\newblock In \emph{Proceedings of the IEEE Conference on Computer Vision and
  Pattern Recognition}, pp.\  5879--5888, 2018.

\bibitem[Saharia et~al.(2022)Saharia, Ho, Chan, Salimans, Fleet, and
  Norouzi]{saharia2022image}
Chitwan Saharia, Jonathan Ho, William Chan, Tim Salimans, David~J Fleet, and
  Mohammad Norouzi.
\newblock Image super-resolution via iterative refinement.
\newblock \emph{IEEE Transactions on Pattern Analysis and Machine
  Intelligence}, 2022.

\bibitem[Sohl-Dickstein et~al.(2015)Sohl-Dickstein, Weiss, Maheswaranathan, and
  Ganguli]{sohl2015deep}
Jascha Sohl-Dickstein, Eric Weiss, Niru Maheswaranathan, and Surya Ganguli.
\newblock Deep unsupervised learning using nonequilibrium thermodynamics.
\newblock In \emph{International Conference on Machine Learning}, pp.\
  2256--2265. PMLR, 2015.

\bibitem[Song et~al.(2020{\natexlab{a}})Song, Meng, and
  Ermon]{song2020denoising}
Jiaming Song, Chenlin Meng, and Stefano Ermon.
\newblock Denoising diffusion implicit models.
\newblock \emph{arXiv preprint arXiv:2010.02502}, 2020{\natexlab{a}}.

\bibitem[Song \& Ermon(2019)Song and Ermon]{song2019score}
Yang Song and Stefano Ermon.
\newblock Generative modeling by estimating gradients of the data distribution.
\newblock In H.~Wallach, H.~Larochelle, A.~Beygelzimer, F.~d\'{} Alch\'{e}-Buc,
  E.~Fox, and R.~Garnett (eds.), \emph{Advances in Neural Information
  Processing Systems}, volume~32. Curran Associates, Inc., 2019.
\newblock URL
  \url{https://proceedings.neurips.cc/paper/2019/file/3001ef257407d5a371a96dcd947c7d93-Paper.pdf}.

\bibitem[Song et~al.(2020{\natexlab{b}})Song, Sohl-Dickstein, Kingma, Kumar,
  Ermon, and Poole]{song2020score}
Yang Song, Jascha Sohl-Dickstein, Diederik~P Kingma, Abhishek Kumar, Stefano
  Ermon, and Ben Poole.
\newblock Score-based generative modeling through stochastic differential
  equations.
\newblock \emph{arXiv preprint arXiv:2011.13456}, 2020{\natexlab{b}}.

\bibitem[Song et~al.(2021)Song, Durkan, Murray, and Ermon]{song2021maximum}
Yang Song, Conor Durkan, Iain Murray, and Stefano Ermon.
\newblock Maximum likelihood training of score-based diffusion models.
\newblock In \emph{Thirty-Fifth Conference on Neural Information Processing
  Systems}, 2021.

\bibitem[Thompson et~al.(2020)Thompson, Greenewald, Lee, and
  Manso]{thompson2020computational}
Neil~C Thompson, Kristjan Greenewald, Keeheon Lee, and Gabriel~F Manso.
\newblock The computational limits of deep learning.
\newblock \emph{arXiv preprint arXiv:2007.05558}, 2020.

\bibitem[Tong et~al.(2020)Tong, Huang, Wolf, Van~Dijk, and
  Krishnaswamy]{tong2020trajectorynet}
Alexander Tong, Jessie Huang, Guy Wolf, David Van~Dijk, and Smita Krishnaswamy.
\newblock Trajectorynet: A dynamic optimal transport network for modeling
  cellular dynamics.
\newblock In \emph{International conference on machine learning}, pp.\
  9526--9536. PMLR, 2020.

\bibitem[Villani(2009)]{villani2009optimal}
C{\'e}dric Villani.
\newblock \emph{Optimal transport: old and new}, volume 338.
\newblock Springer, 2009.

\bibitem[Vincent(2011)]{vincent2011connection}
Pascal Vincent.
\newblock A connection between score matching and denoising autoencoders.
\newblock \emph{Neural computation}, 23\penalty0 (7):\penalty0 1661--1674,
  2011.

\bibitem[Wang et~al.(2021)Wang, Jiao, Xu, Wang, and Yang]{wang2021schbridges}
Gefei Wang, Yuling Jiao, Qian Xu, Yang Wang, and Can Yang.
\newblock Deep generative learning via schr\"{o}dinger bridge.
\newblock \penalty0 (arXiv:2106.10410), Jul 2021.
\newblock \doi{10.48550/arXiv.2106.10410}.
\newblock URL \url{http://arxiv.org/abs/2106.10410}.
\newblock arXiv:2106.10410 [cs].

\bibitem[Yang \& Karniadakis(2019)Yang and Karniadakis]{yang2019potential}
Liu Yang and George~E. Karniadakis.
\newblock Potential flow generator with {\textdollar}l{\_}2{\textdollar}
  optimal transport regularity for generative models.
\newblock \emph{CoRR}, abs/1908.11462, 2019.
\newblock URL \url{http://arxiv.org/abs/1908.11462}.

\bibitem[Zhang \& Chen(2022)Zhang and Chen]{zhang2022fast}
Qinsheng Zhang and Yongxin Chen.
\newblock Fast sampling of diffusion models with exponential integrator.
\newblock \emph{arXiv preprint arXiv:2204.13902}, 2022.

\end{thebibliography}
\bibliographystyle{iclr2023_conference_arxiv}

\newpage 
\appendix

\section{Theorem Proofs}
\label{A:proofs}

\marginalvf*
\begin{proof}
To verify this, we check that $p_t$ and $u_t$ satisfy the continuity equation (\eqref{e:continuity}):

\begin{align*}
    \frac{d}{dt}p_t(x) &= \int \Big( \frac{d}{dt} p_t(x\vert x_1)  \Big) q(x_1) dx_1 = -\int \mathrm{div}\Big(   u_t(x\vert x_1)p_t(x\vert x_1)  \Big) q(x_1) dx_1 \\
    &= -\mathrm{div}\Big(\int    u_t(x\vert x_1)p_t(x\vert x_1)   q(x_1) dx_1\Big) = -\mathrm{div} \Big( u_t(x) p_t(x) \Big ),
\end{align*}
where in the second equality we used the fact that $u_t(\cdot\vert x_1)$ generates $p_t(\cdot\vert x_1)$, in the last equality we used \eqref{e:u_t}. Furthermore, the first and third equalities are justified by assuming the integrands satisfy the regularity conditions of the Leibniz Rule (for exchanging integration and differentiation).
\end{proof}

\cfm*
\begin{proof}
To ensure existence of all integrals and to allow the changing of integration order (by Fubini's Theorem) in the following we assume that $q(x)$ and $p_t(x|x_1)$ are decreasing to zero at a sufficient speed as $\norm{x}\too \infty$, and that $u_t,v_t,\nabla_\theta v_t$ are bounded.

First, using the standard bilinearity of the $2$-norm we have that 
\begin{align*}
    \|v_t(x)-u_t(x)\|^2 &= \norm{v_t(x)}^2 - 2\ip{v_t(x),u_t(x)} + \norm{u_t(x)}^2 \\
    \|v_t(x)-u_t(x\vert x_1)\|^2  &= \norm{v_t(x)}^2 - 2\ip{v_t(x),u_t(x\vert x_1)} + \norm{u_t(x\vert x_1)}^2
\end{align*}
Next, remember that $u_t$ is independent of $\theta$ and note that 
\begin{align*}
    \E_{p_t(x)} \|v_t(x)\|^2 &= \int \|v_t(x)\|^2 p_t(x) dx = \int \|v_t(x)\|^2 p_t(x\vert x_1)q(x_1) dx_1 dx \\ &= \E_{q(x_1),p_t(x\vert x_1)}\|v_t(x)\|^2,
\end{align*}
where in the second equality we use \eqref{e:p_t}, and in the third equality we change the order of integration.  
Next,
\begin{align*}
    \E_{p_t(x)}\ip{v_t(x),u_t(x)} &= \int\ip{v_t(x),\frac{\int u_t(x\vert x_1) p_t(x\vert x_1)q(x_1)dx_1}{p_t(x)}}p_t(x)dx \\
    &=
    \int\ip{v_t(x),\int u_t(x\vert x_1) p_t(x\vert x_1)q(x_1)dx_1}dx \\
    &= 
    \int\ip{v_t(x), u_t(x\vert x_1) }p_t(x\vert x_1)q(x_1)dx_1dx \\
    &= \E_{q(x_1),p_t(x\vert x_1)} \ip{v_t(x), u_t(x\vert x_1) },
\end{align*}
where in the last equality we change again the order of integration. 
\end{proof}

\condvf*
\begin{proof}
For notational simplicity let $w_t(x)=u_t(x\vert x_1)$. Now consider \eqref{e:ode}:
\begin{equation*}
    \frac{d}{dt}\psi_t(x) = w_t(\psi_t(x)).
\end{equation*}
Since $\psi_t$ is invertible (as $\sigma_t(x_1)>0$) we let $x=\psi^{-1}(y)$ and get
\begin{equation}\label{ea:phi_phi-inv}
    \psi'_t(\psi^{-1}(y)) = w_t(y),
\end{equation}
where we used the apostrophe notation for the derivative to emphasis that $\psi'_t$ is evaluated at $\psi^{-1}(y)$. Now, inverting $\psi_t(x)$ provides
\begin{equation*}
    \psi_t^{-1}(y)=\frac{y-\mu_t(x_1)}{\sigma_t(x_1)}.
\end{equation*}
Differentiating $\psi_t$ with respect to $t$ gives
\begin{equation*}
    \psi_t'(x)=\sigma_t'(x_1)x + \mu_t'(x_1).
\end{equation*}
Plugging these last two equations in \eqref{ea:phi_phi-inv} we get
\begin{equation*}
    w_t(y) = \frac{\sigma'_t(x_1)}{\sigma_t(x_1)}\parr{y-\mu_t(x_1)} + \mu_t'(x_1)
\end{equation*}
as required. 
\end{proof}

\section{The continuity equation}
\label{A:continuity_equation}
One method of testing if a vector field $v_t$ generates a probability path $p_t$ is the continuity equation \citep{villani2009optimal}. It is a Partial Differential Equation (PDE) providing a necessary and sufficient condition to ensuring that a vector field $v_t$ generates $p_t$,
\begin{equation}\label{e:continuity}
    \frac{d}{dt}p_t(x) + \divv(p_t(x) v_t(x)) = 0,
\end{equation}
where the divergence operator, $\divv$, is defined with respect to the spatial variable $x=(x^1,\ldots,x^d)$, \ie, $\divv=\sum_{i=1}^d \frac{\partial }{\partial x^i}$. 


\section{Computing probabilities of the CNF model}
\label{A:cnf_prob}

We are given an arbitrary data point $x_1\in \Real^d$ and need to compute the model probability at that point, \ie, $p_1(x_1)$. Below we recap how this can be done covering the basic relevant ODEs, the scaling of the divergence computation, taking into account data transformations (\eg, centering of data), and Bits-Per-Dimension computation. 

\paragraph{ODE for computing $p_1(x_1)$}
The continuity equation with \eqref{e:ode} lead to the instantaneous change of variable  \citep{chen2018neural,ben2022matching}:
\begin{align*}
    \frac{d}{dt}\log p_t(\phi_t(x)) + \divv (v_t(\phi_t(x))=0.
\end{align*}
Integrating $t\in [0,1]$ gives:
\begin{align}\label{ea:log_diff}
    \log p_1(\phi_1(x)) - \log p_0(\phi_0(x)) = -\int_0^1 \divv(v_t(\phi_t(x))) dt
\end{align}
Therefore, the log probability can be computed together with the flow trajectory by solving the ODE:
\begin{align}\label{ea:system}
    \frac{d}{dt}\begin{bmatrix}
    \phi_t(x) \\
    f(t) 
    \end{bmatrix}=
    \begin{bmatrix}
    v_t(\phi_t(x)) \\ 
    -\divv (v_t(\phi_t(x)))
    \end{bmatrix}
\end{align}
Given initial conditions 
\begin{align}\label{ea:initial_conds}
    \begin{bmatrix}
    \phi_0(x) \\ 
    f(0)
    \end{bmatrix} = \begin{bmatrix}
    x_0 \\ c
    \end{bmatrix}.
\end{align}
the solution $\brac{\phi_t(x),f(t)}^T$ is uniquely defined (up to some mild conditions on the VF $v_t$). Denote $x_1 = \phi_1(x)$, and according to \eqref{ea:log_diff}, 
\begin{equation}\label{ea:f1}
    f(1) = c + \log p_1(x_1) - \log p_0(x_0).
\end{equation}

Now, we are given an arbitrary $x_1$ and want to compute $p_1(x_1)$. For this end, we will need to solve \eqref{ea:system} in reverse. That is,
\begin{align}\label{ea:system_final}
    \frac{d}{ds}\begin{bmatrix}
    \phi_{1-s}(x) \\
    f(1-s) 
    \end{bmatrix}=
    \begin{bmatrix}
    -v_{1-s}(\phi_{1-s}(x)) \\ 
    \divv (v_{1-s}(\phi_{1-s}(x)))
    \end{bmatrix}
\end{align}
and we solve this equation for $s\in [0,1]$ with the initial conditions at $s=0$:
\begin{align}\label{ea:initial_conds_final}
    \begin{bmatrix}
    \phi_1(x) \\ 
    f(1)
    \end{bmatrix} = \begin{bmatrix}
    x_1 \\ 0
    \end{bmatrix}.
\end{align}
From uniqueness of ODEs, the solution will be identical to the solution of \eqref{ea:system} with initial conditions in \eqref{ea:initial_conds} where $c=\log p_0(x_0)-\log p_1(x_1)$. This can be seen from \eqref{ea:f1} and setting $f(1)=0$. Therefore we get that 
\begin{align*}
    f(0) &= \log p_0(x_0)-\log p_1(x_1)
\end{align*}
and consequently 
\begin{align}\label{ea:p_1_final}
    \log p_1(x_1) = \log p_0(x_0) - f(0).
\end{align}
To summarize, to compute $p_1(x_1)$ we first solve the ODE in \eqref{ea:system_final} with initial conditions in \eqref{ea:initial_conds_final}, and the compute \eqref{ea:p_1_final}. 

\paragraph{Unbiased estimator to $p_1(x_1)$}
Solving \eqref{ea:system_final} requires computation of $\divv$ of VFs in $\Real^d$ which is costly. \cite{ffjord2018} suggest to replace the divergence by the (unbiased) Hutchinson trace estimator, 
\begin{align}\label{ea:system_final_hutchinson}
    \frac{d}{ds}\begin{bmatrix}
    \phi_{1-s}(x) \\
    \tilde{f}(1-s)
    \end{bmatrix}=
    \begin{bmatrix}
    -v_{1-s}(\phi_{1-s}(x)) \\ 
    z^T Dv_{1-s}(\phi_{1-s}(x)) z
    \end{bmatrix},
\end{align}
where $z\in \Real^d$ is a sample from a random variable such that $\E zz^T=I$. Solving the ODE in \eqref{ea:system_final_hutchinson} exactly (in practice, with a  small controlled error) with initial conditions in \eqref{ea:initial_conds_final} leads to
\begin{align*}
    \E_z \brac{\log p_0(x_0) - \tilde{f}(0)} &= \log p_0(x_0) - \E_z \brac{\tilde{f}(0)-\tilde{f}(1)} \\
    &= \log p_0(x_0) - \E_z\brac{\int_0^1 z^T Dv_{1-s}(\phi_{1-s}(x)) z \, ds}  \\
    &= \log p_0(x_0) - \int_0^1 \E_z \brac{z^T Dv_{1-s}(\phi_{1-s}(x)) z } ds  \\
    &= \log p_0(x_0) - \int_0^1 \divv(v_{1-s}(\phi_{1-s}(x))) ds 
    \\
    &= \log p_0(x_0) - \parr{f(0)-f(1)} \\
    &= \log p_0(x_0) - \parr{\log p_0(x_0) - \log p_1(x_1)} \\
    &= \log p_1(x_1),
\end{align*} 
where in the third equality we switched order of integration assuming the sufficient condition of Fubini's theorem hold, and in the previous to last equality we used \eqref{ea:f1}. 
Therefore the random variable 
\begin{align}\label{ea:ubiased_p1}
    \log p_0(x_0) - \tilde{f}(0)
\end{align}
is an unbiased estimator for $\log p_1(x_1)$. To summarize, for a scalable unbiased estimation of $p_1(x_1)$ we first solve the ODE in \eqref{ea:system_final_hutchinson} with initial conditions in \eqref{ea:initial_conds_final}, and then output  \eqref{ea:ubiased_p1}. 

\paragraph{Transformed data}
Often, before training our generative model we transform the data, \eg, we scale and/or translate the data. Such a transformation is denoted by $\varphi^{-1}:\Real^d\too\Real^d$ and our generative model becomes a composition 
\begin{align*}
    \psi(x) = \varphi \circ \phi(x)
\end{align*}
where $\phi:\Real^d \too \Real^d$ is the model we train. Given a prior probability $p_0$ we have that the push forward of this probability under $\psi$  (\eqref{e:push_forward} and \eqref{e:push_forward_explicit}) takes the form
\begin{align*}
    p_1(x) = \psi_* p_0 (x) &= p_0(\phi^{-1}(\varphi^{-1}(x)))\det\brac{D\phi^{-1}(\varphi^{-1}(x))}\det\brac{D\varphi^{-1}(x)}\\ 
    &= \parr{\phi_* p_0 (\varphi^{-1}(x))} \det\brac{D\varphi^{-1}(x)}
\end{align*}
and therefore
\begin{align*}
    \log p_1(x) = \log \phi_* p_0 (\varphi^{-1}(x)) +  \log \det\brac{D\varphi^{-1}(x)}.
\end{align*}
For images $d=H\times W\times 3$ and we consider a transform $\phi$ that maps each pixel value from $[-1,1]$ to $[0,256]$. Therefore, 
\begin{equation*}
    \varphi(y) = 2^7(y+1)
\end{equation*}
and
\begin{equation*}
    \varphi^{-1}(x) = 2^{-7}x  - 1
\end{equation*}
For this case we have
\begin{equation}\label{ea:lop_p1}
\log p_1(x) = \log \phi_* p_0 (\varphi^{-1}(x)) - 7d\log 2.
\end{equation}

\paragraph{Bits-Per-Dimension (BPD) computation}
BPD is defined by
\begin{equation}
    \mathrm{BPD} = \E_{x_1}\brac{ -\frac{\log_2 p_1(x_1)}{d}} = \E_{x_1}\brac{ -\frac{\log p_1(x_1)}{d\log 2}}
\end{equation}
Following \eqref{ea:lop_p1} we get
\begin{align*}
    \mathrm{BPD} = -\frac{\log \phi_* p_0 (\varphi^{-1}(x))}{d\log 2 } + 7
\end{align*}
and $\log \phi_* p_0 (\varphi^{-1}(x))$ is approximated using the unbiased estimator in \eqref{ea:ubiased_p1} over the transformed data $\varphi^{-1}(x_1)$. Averaging the unbiased estimator on a large test test $x_1$ provides a good approximation to the test set BPD.

\section{Diffusion conditional vector fields}
\label{A:diff_cond_VP}

We derive the vector field governing the Probability Flow ODE (equation 13 in \cite{song2020score}) for the VE and VP diffusion paths (\eqref{e:VP_diffusion_path}) and note that it coincides with the conditional vector fields we derive using Theorem \ref{thm:cond_vf}, namely the vector fields defined in equations \ref{e:VE} and \ref{e:ut_dif_with_our_method}.

We start with a short primer on how to find a conditional vector field for the probability path described by the Fokker-Planck equation, then instantiate it for the VE and VP probability paths. 

Since in the diffusion literature the diffusion process runs from data at time $t=0$ to noise at time $t=1$, we will need the following lemma to translate the diffusion VFs to our convention of $t=0$ corresponds to noise and $t=1$ corresponds to data:
\begin{lemma}\label{lem:reverse}
Consider a flow defined by a vector field $u_t(x)$ generating probability density path $p_t(x)$. Then, the vector field $\tilde{u}_t(x) = -u_{1-t}(x)$ generates the path $\tilde{p}_t(x) = p_{1-t}(x)$ when initiated from $\tilde{p}_0(x) = p_1(x)$.
\end{lemma}
\begin{proof}
We use the continuity equation (\eqref{e:continuity}):
\begin{align*}
    \frac{d}{dt}\tilde{p}_t(x)=\frac{d}{dt}p_{1-t}(x) &= -p'_{1-t}(x) \\
    &= \divv(p_{1-t}(x)u_{1-t}(x)) \\
    &= -\divv(\tilde{p}_{t}(x)(-u_{1-t}(x))) 
\end{align*}
and therefore $\tilde{u}_t(x)=-u_{1-t}(x)$ generates $\tilde{p}_t(x)$. 
\end{proof}

\paragraph{Conditional VFs for Fokker-Planck probability paths} Consider a Stochastic Differential Equation (SDE) of the standard form
\begin{equation}\label{e:sde}
    dy = f_t dt + g_t dw
\end{equation}
with time parameter $t$, drift $f_t$, diffusion coefficient $g_t$, and $dw$ is the Wiener process. The solution $y_t$ to the SDE is a stochastic process, \ie, a continuous time-dependent random variable, the probability density of which, $p_t(y_t)$, is characterized by the Fokker-Planck equation:
\begin{equation}
    \frac{dp_t}{dt} = -\mathrm{div}(f_t p_t) + \frac{g_t^2}{2}\Delta p_t
\end{equation}
where $\Delta$ represents the Laplace operator (in $y$), namely $\mathrm{div}\nabla$, where $\nabla$ is the gradient operator (also in $y$). Rewriting this equation in the form of the continuity equation can be done as follows \citep{maoutsa2020}:
\begin{align*}
    \frac{dp_t}{dt} 
    = -\mathrm{div}\Big( f_t p_t - \frac{g^2}{2}\frac{\nabla p_t}{p_t}p_t \Big) 
    = -\mathrm{div}\Big(\big( f_t - \frac{g_t^2}{2}\nabla \log p_t\big)p_t \Big) = -\mathrm{div}\Big( w_t p_t \Big)
\end{align*}
where the vector field
\begin{equation}\label{e:w}
    w_t = f_t - \frac{g_t^2}{2}\nabla \log p_t
\end{equation}
satisfies the continuity equation with the probability path $p_t$, and therefore generates $p_t$. 

\paragraph{Variance Exploding (VE) path}
The SDE for the VE path is 
\begin{equation*}
 dy = \sqrt{\frac{d}{dt}\sigma_t^2}dw,
\end{equation*}
where $\sigma_0=0$ and increasing to infinity as $t\too 1$.
The SDE is moving from data, $y_0$, at $t=0$ to noise, $y_1$, at $t=1$ with the probability path
\begin{equation*}
    p_t(y|y_0) = \gN(y | y_0, \sigma^2_t I).
\end{equation*}
The conditional VF according to \eqref{e:w} is:
\begin{equation*}
    w_t(y|y_0) = \frac{ \sigma_t'}{\sigma_t}(y-y_0)
\end{equation*}
Using Lemma \ref{lem:reverse} we get that the probability path 
\begin{equation*}
    \tilde{p}_t(y|y_0) = \gN(y|y_0, \sigma_{1-t}^2 I)
\end{equation*}
is generated by 
\begin{align*}
    \tilde{w}_t(y|y_0) =-\frac{ \sigma_{1-t}'}{\sigma_{1-t}}(y-y_0),
\end{align*}
which coincides with \eqref{e:u_t_VE}.

\paragraph{Variance Preserving (VP) path}
The SDE for the VP path is 
\begin{equation*}
    dy = -\frac{T'(t)}{2}y + \sqrt{T'(t)}dw,
\end{equation*}
where $T(t)=\int_0^t \beta(s)ds$, $t\in [0,1]$. The SDE coefficients are therefore
\begin{align*}
    f_s(y) = -\frac{T'(s)}{2}y, \quad g_s = \sqrt{T'(s)} 
\end{align*}
and
\begin{align*}
    &p_t(y \vert  y_0) = \gN(y \vert e^{-\frac{1}{2}T(t)}y_0, (1-e^{-T(t)})I).
\end{align*}
Plugging these choices in \eqref{e:w} we get the conditional VF
\begin{equation}
    w_t(y\vert y_0) =\frac{T'(t)}{2}\parr{\frac{y-e^{-\frac{1}{2}T(t)}y_0}{1-e^{-T(t)}} -y }
\end{equation}
Using Lemma \ref{lem:reverse} to reverse the time we get the conditional VF for the reverse probability path:
\begin{align*}
    \tilde{w}_t(y|y_0) &= -\frac{T'(1-t)}{2}\parr{\frac{y-e^{-\frac{1}{2}T(1-t)}y_0}{1-e^{-T(1-t)}} -y } \\
    &= -\frac{T'(1-t)}{2}\brac{\frac{e^{-T(1-t)}y-e^{-\frac{1}{2}T(1-t)}y_0}{1-e^{-T(1-t)}}},
\end{align*}
which coincides with \eqref{e:ut_dif_with_our_method}.


\begin{figure}
\centering
\begin{subfigure}[t]{0.0989\linewidth}
\includegraphics[width=\linewidth]{figures/plots/2d_vf_reference.png}
\end{subfigure}
\begin{subfigure}[t]{0.43\linewidth}
\centering
    \begin{subfigure}[t]{0.23\linewidth}
        \centering
        \includegraphics[width=\linewidth]{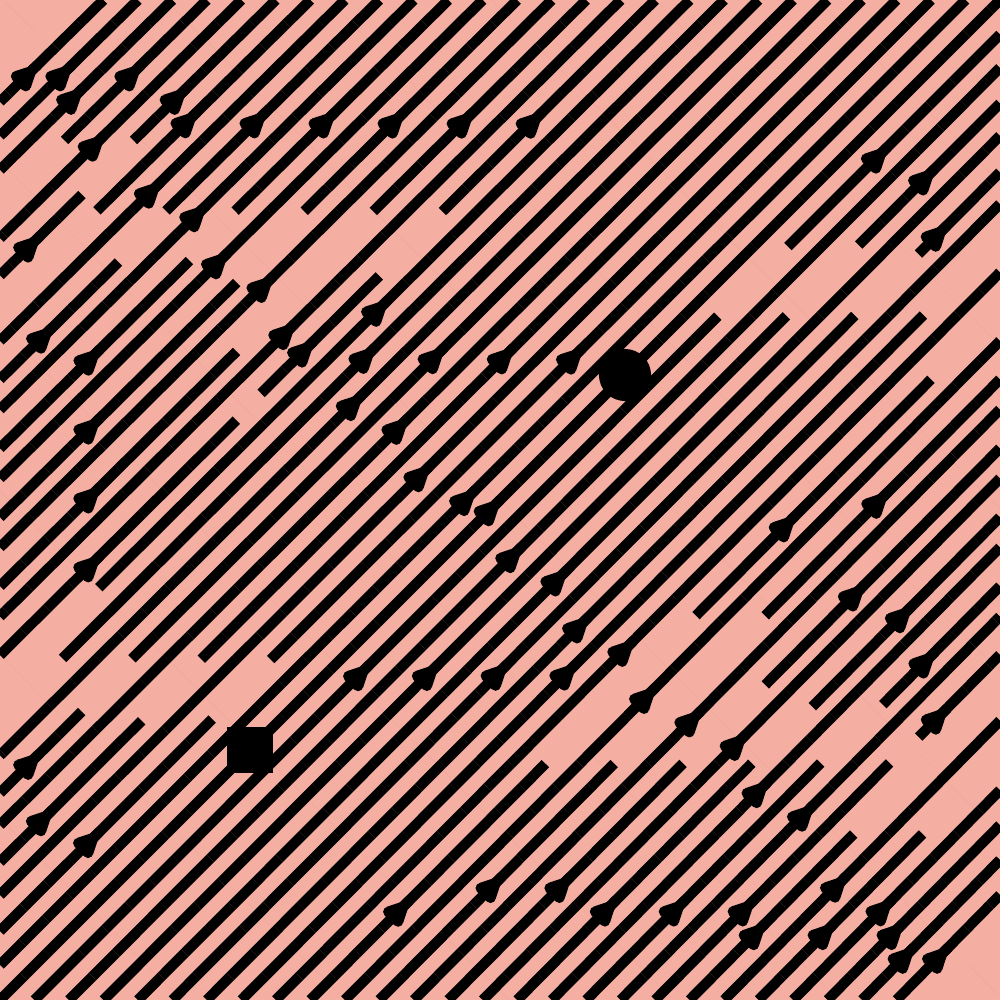}
        \vspace{-1.5em}
        \caption*{\scalebox{0.8}{$t=0.0$}}
    \end{subfigure}
    \begin{subfigure}[t]{0.23\linewidth}
        \centering
        \includegraphics[width=\linewidth]{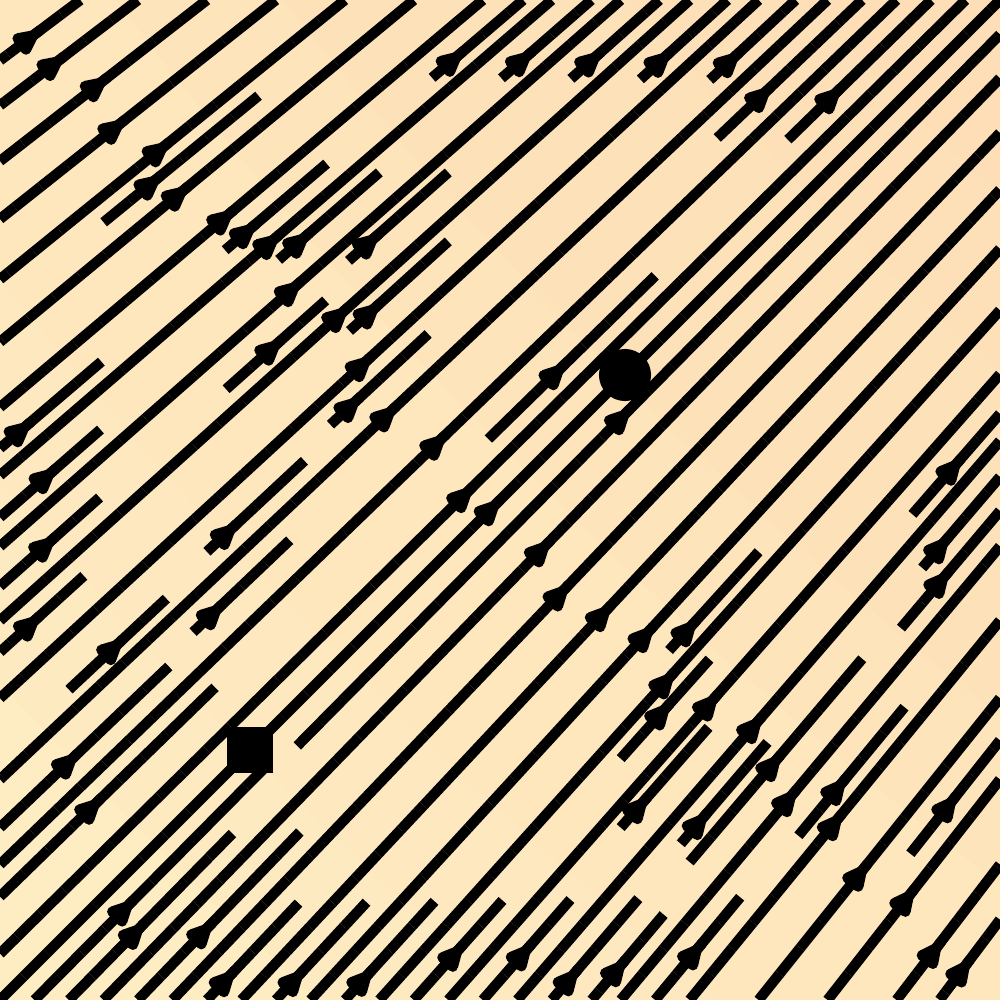}
        \vspace{-1.5em}
        \caption*{\scalebox{0.8}{$t=\nicefrac{1}{3}$}}
    \end{subfigure}
    \begin{subfigure}[t]{0.23\linewidth}
        \centering
        \includegraphics[width=\linewidth]{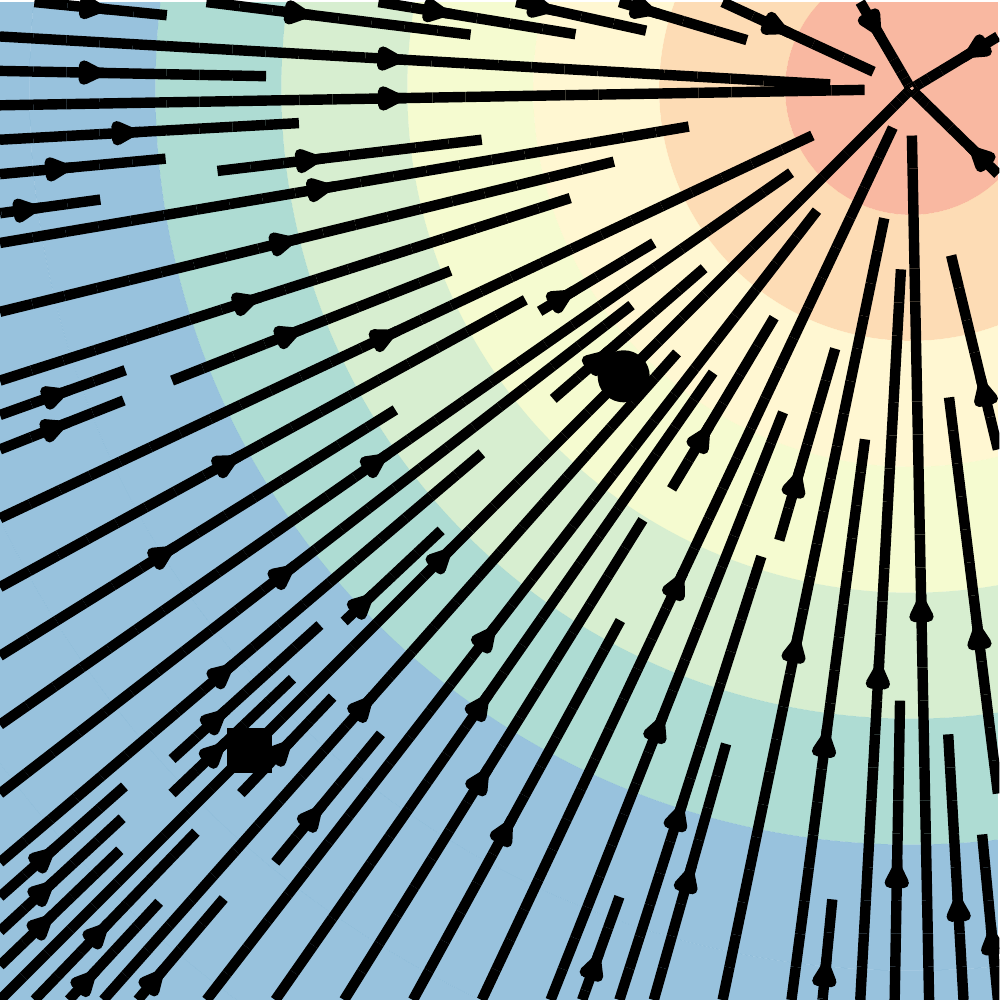}
        \vspace{-1.5em}
        \caption*{\scalebox{0.8}{$t=\nicefrac{2}{3}$}}
    \end{subfigure}
    \begin{subfigure}[t]{0.23\linewidth}
        \centering
        \includegraphics[width=\linewidth]{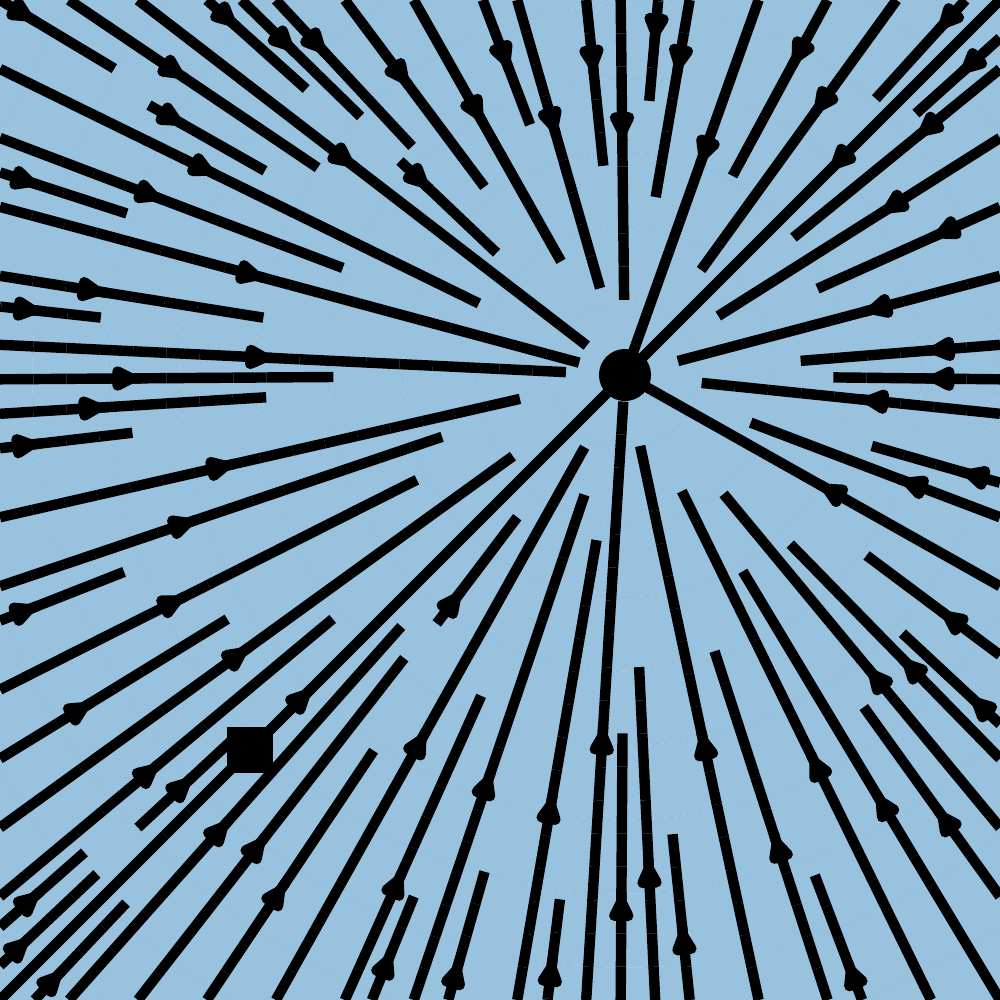}
        \vspace{-1.5em}
        \caption*{\scalebox{0.8}{$t=1.0$}}
    \end{subfigure}
\caption*{Diffusion path -- conditional vector field}
\end{subfigure}
\caption{VP Diffusion path's conditional vector field. Compare to Figure \ref{fig:2d_fv}. }
\label{fig:2d_diff_fv}
\end{figure}

\begin{figure}
\centering
\begin{tabular}{c}
    \includegraphics[width=0.68\textwidth]{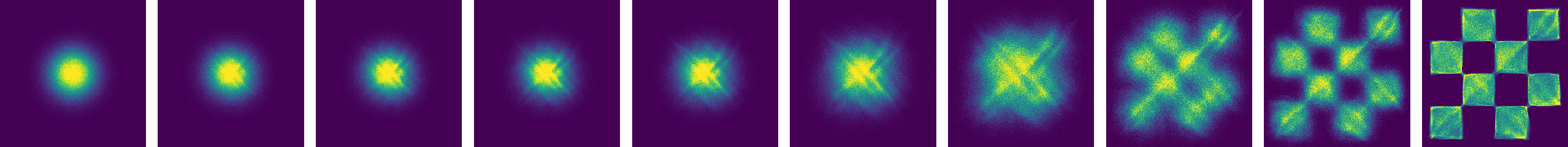}   \\
     \scriptsize{ScoreFlow}\\ 
     \includegraphics[width=0.68\textwidth]{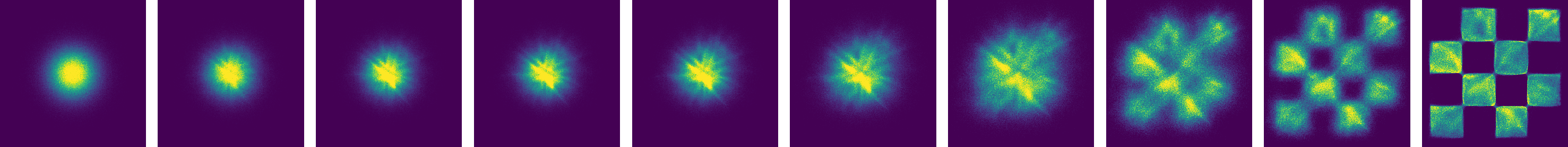}   \\
     \scriptsize{DDPM}
\end{tabular}
    \caption{Trajectories of CNFs trained with ScoreFlow~\citep{song2021maximum} and DDPM~\citep{ho2020denoising} losses on 2D checkerboard data, using the same learning rate and other hyperparameters as Figure \ref{fig:2d_checkerboard}. }
    \label{fig:2d_checkerboard_ScoreFlow}
\end{figure}

\section{Implementation details}
\label{A:implementation-details}

For the 2D example we used an MLP with 5-layers of 512 neurons each, while for images we used the UNet architecture from \citet{dhariwal2021diffusion}. 
For images, we center crop images and resize to the appropriate dimension, whereas for the 32$\times$32 and 64$\times$64 resolutions we use the same pre-processing as \citep{chrabaszcz2017downsampled}. 
The three methods (FM-OT, FM-Diffusion, and SM-Diffusion) are always trained on the same architecture, same hyper-parameters, and for the same number of epochs.

\subsection{Diffusion baselines}\label{a:dif_baselines}
{
\paragraph{Losses.}
We consider three options as diffusion baselines that correspond to the most popular diffusion loss parametrizations \citep{song2019score,song2021maximum,ho2020denoising,kingma2021vdm}. We will assume general Gaussian path form of \eqref{e:pt_gau}, \ie, 
\begin{equation*}
    p_t(x|x_1)=\gN(x|\mu_t(x_1),\sigma_t^2(x_1) I).
\end{equation*}

Score Matching loss is 
\begin{align}
    \gL_{\SM}(\theta) &= \E_{t,q(x_1),p_t(x|x_1)} \lambda(t) \norm{s_t(x) - \nabla \log p_t(x|x_1)}^2 \\
    &= \E_{t,q(x_1),p_t(x|x_1)} \lambda(t) \norm{s_t(x) - \frac{x-\mu_t(x_1)}{\sigma_t^2(x_1)}}^2.
\end{align}
Taking $\lambda(t) = \sigma_t^{2}(x_1)$ corresponds to the original Score Matching (SM) loss from \cite{song2019score}, while considering $\lambda(t)=\beta(1-t)$ ($\beta$ is defined below) corresponds to the Score Flow (SF) loss motivated by an NLL upper bound \citep{song2021maximum}; $s_t$ is the learnable score function. DDPM (Noise Matching) loss from \cite{ho2020denoising} (equation 14) is
\begin{align}
    \gL_{\NM}(\theta)  &= \E_{t,q(x_1),p_t(x|x_1)}  \norm{\epsilon_t(x) - \frac{x-\mu_t(x_1)}{\sigma_t(x_1)} }^2 \\
    &= 
    \E_{t,q(x_1),p_0(x_0)} \Big \|  {\epsilon_t(\sigma_t(x_1) x_0 + \mu_t(x_1)) - x_0 }\Big \|^2
\end{align}
where $p_0(x)=\gN(x|0,I)$ is the standard Gaussian, and $\epsilon_t$ is the learnable noise function.

\paragraph{Diffusion path.} For the diffusion path we use the standard VP diffusion (\eqref{e:ut_dif_with_our_method}), namely,
\begin{equation*}
    \mu_t(x_1) = \alpha_{1-t}x_1, \quad \sigma_t(x_1) = \sqrt{1-\alpha_{1-t}^2},\quad  \text{where } \alpha_t = e^{-\frac{1}{2}T(t)},\quad  T(t)=\int_0^{t} \beta(s)ds,
\end{equation*}
with, as suggested in \cite{song2020score}, $\beta(s) = \beta_{\min} + s(\beta_{\max}-\beta_{\min})$ and consequently 
$$T(s) = \int_0^s \beta(r)dr = s\beta_{\min} + \frac{1}{2}s^2(\beta_{\max} - \beta_{\min}),$$
where $\beta_{\min}=0.1$, $\beta_{\max}=20$ and time is sampled in $[0,1-\eps]$, $\eps=10^{-5}$ for training and likelihood and $\eps=10^{-5}$ for sampling.

\paragraph{Sampling.} Score matching samples are produced by solving the ODE (\eqref{e:ode}) with the vector field
\begin{equation}\label{ea:score_flow}
    u_t(x) = -\frac{T'(1-t)}{2}\brac{s_t(x) - x}.
\end{equation} 
DDPM samples are computed with \eqref{ea:score_flow} after setting $s_t(x) = \epsilon_t(x)/\sigma_t$, where $\sigma_t = \sqrt{1-\alpha_{1-t}^2}$. 
}

\subsection{Training \& evaluation details}

\begin{table}
\centering
\resizebox{\linewidth}{!}{%
\begin{tabular}{l c c c c}
\toprule
 & CIFAR10 & ImageNet-32 & ImageNet-64 & ImageNet-128 \\
\midrule
Channels & 256 & 256 & 192 & 256  \\
Depth & 2 & 3 & 3 & 3  \\
Channels multiple & 1,2,2,2 & 1,2,2,2 & 1,2,3,4 & 1,1,2,3,4 \\
Heads & 4 & 4 & 4 & 4 \\
Heads Channels & 64 & 64 & 64 & 64 \\
Attention resolution & 16 & 16,8 & 32,16,8 & 32,16,8 \\
Dropout & 0.0 & 0.0 & 0.0 & 0.0 \\
Effective Batch size & 256 & 1024 & 2048 & 1536 \\
GPUs & 2 & 4 & 16 & 32  \\
Epochs & 1000 & 200 & 250 & 571 \\
Iterations & 391k & 250k  & 157k & {500k}  \\
Learning Rate & 5e-4 & 1e-4 & 1e-4 & 1e-4 \\
Learning Rate Scheduler & Polynomial Decay & Polynomial Decay & Constant & Polynomial Decay \\
Warmup Steps & 45k & 20k & - & 20k \\
\bottomrule
\end{tabular}
}
\caption{Hyper-parameters used for training each model}
\label{tab:hyper-params}
\end{table}

We report the hyper-parameters used in Table \ref{tab:hyper-params}.  We use full 32 bit-precision for training CIFAR10 and ImageNet-32 and 16-bit mixed precision for training ImageNet-64/128/256.  All models are trained using the Adam optimizer with the following parameters: $\beta_1 = 0.9$, $\beta_2=0.999$, weight decay = 0.0, and $\epsilon = 1e{-8}$. All methods we trained (\ie, FM-OT, FM-Diffusion, SM-Diffusion) using  identical architectures, with the same parameters for the the same number of Epochs (see Table \ref{tab:hyper-params} for details).  We use either a constant learning rate schedule or a polynomial decay schedule (see Table \ref{tab:hyper-params}).  The polynomial decay learning rate schedule includes a warm-up phase for a specified number of training steps.  In the warm-up phase, the learning rate is linearly increased from $1e{-8}$ to the peak learning rate (specified in Table \ref{tab:hyper-params}).  Once the peak learning rate is achieved, it linearly decays the learning rate down to $1e{-8}$ until the final training step.

When reporting negative log-likelihood, we dequantize using the standard uniform dequantization. We report an importance-weighted estimate using
\begin{equation}
    \log \frac{1}{K} \sum_{k=1}^K p_t (x + u_k), \text{ where } u_k \sim \gU(0, 1),
\end{equation}
with $x$ is in \{0, \dots, 255\} and solved at $t=1$ with an adaptive step size solver \texttt{dopri5} with \texttt{atol=rtol=1e-5} using the \texttt{torchdiffeq}~\citep{torchdiffeq} library. Estimated values for different values of $K$ are in Table \ref{tab:nll_k_results}.

When computing FID/Inception scores for CIFAR10, ImageNet-32/64 we use the TensorFlow \lstinline{GAN} library \footnote{\url{https://github.com/tensorflow/gan}}.  To remain comparable to \cite{dhariwal2021diffusion} for ImageNet-128 we use the evaluation script they include in their publicly available code repository \footnote{\url{https://github.com/openai/guided-diffusion}}.

\section{Additional tables and figures}

\begin{table}\centering
\ra{1.2}
\setlength{\tabcolsep}{2.5pt}
\resizebox{\linewidth}{!}{%
\begin{tabular}{@{} l R{3.5em} R{3.5em} R{3.5em} r R{3.5em} R{3.5em} R{3.5em} r R{3.5em} R{3.5em} R{3.5em} r @{}}\toprule
  & \multicolumn{3}{c}{\bf CIFAR-10} &
  & \multicolumn{3}{c}{\bf ImageNet 32$\times$32} &
  & \multicolumn{3}{c}{\bf ImageNet 64$\times$64} \\
\cmidrule(lr){2-4} \cmidrule(lr){5-8} \cmidrule(lr){9-12}
Model & {$K$=1} & {$K$=20} & {$K$=50}
& & {$K$=1} & {$K$=5} & {$K$=15}
& & {$K$=1} & {$K$=5} & {$K$=10} \\
\cmidrule(r){1-1} \cmidrule(lr){2-4} \cmidrule(lr){5-8} \cmidrule(lr){9-12}
\textit{\small Ablation}\\
\;\; DDPM &
3.24 & 3.14 & 3.12 & & 
3.62 & 3.57 & 3.54 & & 
3.36 & 3.33 & 3.32 \\
\;\; Score Matching  & 
3.28 & 3.18 & 3.16 & & 
3.65 & 3.59 & 3.57 & & 
3.43 & 3.41 & 3.40 \\
\;\; ScoreFlow & 
3.21 & 3.11  & 3.09 & & 
3.63 & 3.57 & 3.55
 & & 
3.39 & 3.37 & 3.36 \\
\cmidrule(r){1-1} \cmidrule(lr){2-4} \cmidrule(lr){5-8} \cmidrule(lr){9-12}
\textit{\small Ours}\\
\;\; FM \textsuperscript{w}/ Diffusion &  
3.23 & 3.13 & 3.10 & & 
3.64 & 3.58 & 3.56 & & 
3.37 & 3.34 & 3.33 \\
\;\; FM \textsuperscript{w}/ OT & 
3.11
 & 3.01 & 2.99 & & 
3.62 & 3.56 & 3.53 & & 
3.35 & 3.33 & 3.31 \\

\bottomrule
\end{tabular}
}
\caption{Negative log-likelihood (in bits per dimension) on the test set with different values of $K$ using uniform dequantization. }
\label{tab:nll_k_results}
\end{table}

\begin{figure}
  \begin{center} 
      \includegraphics[width=0.5\textwidth]{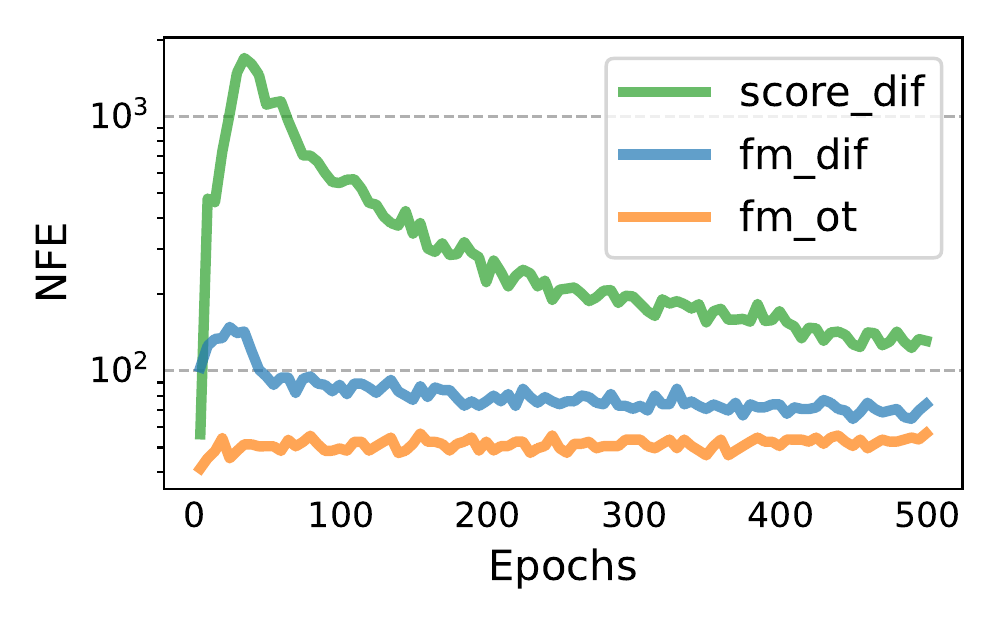} 
  \end{center}
  \caption{Function evaluations for sampling during training, for models trained on CIFAR-10 using \texttt{dopri5} solver with tolerance $1e^{-5}$.}
  \label{fig:2d_NFE_vs_epochs}
\end{figure}
%

\begin{figure}
    \centering
    \includegraphics[width=\textwidth]{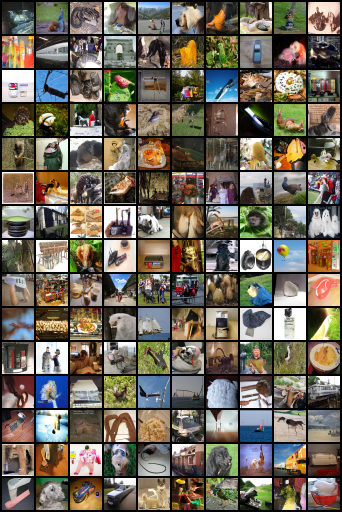}
    \caption{Non-curated unconditional ImageNet-32 generated images of a CNF trained with FM-OT. }
    \label{fig:imagenet32_samples}
\end{figure}

\begin{figure}
    \centering
    \includegraphics[width=\textwidth]{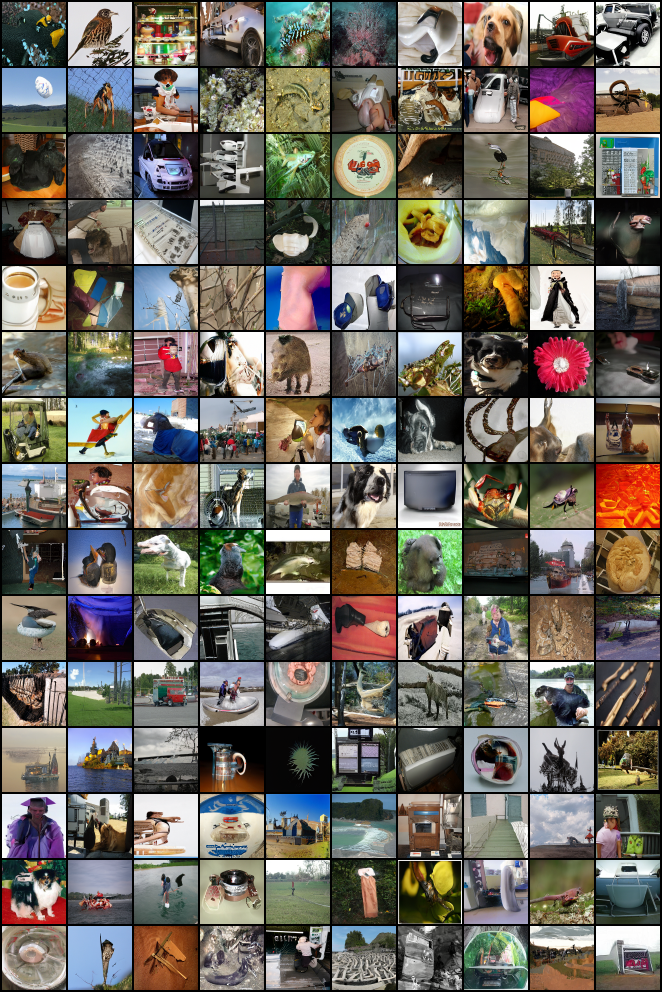}
    \caption{Non-curated unconditional ImageNet-64 generated images of a CNF trained with FM-OT. }
    \label{fig:imagenet64_samples}
\end{figure}

\begin{figure}
    \centering
    \includegraphics[width=\textwidth]{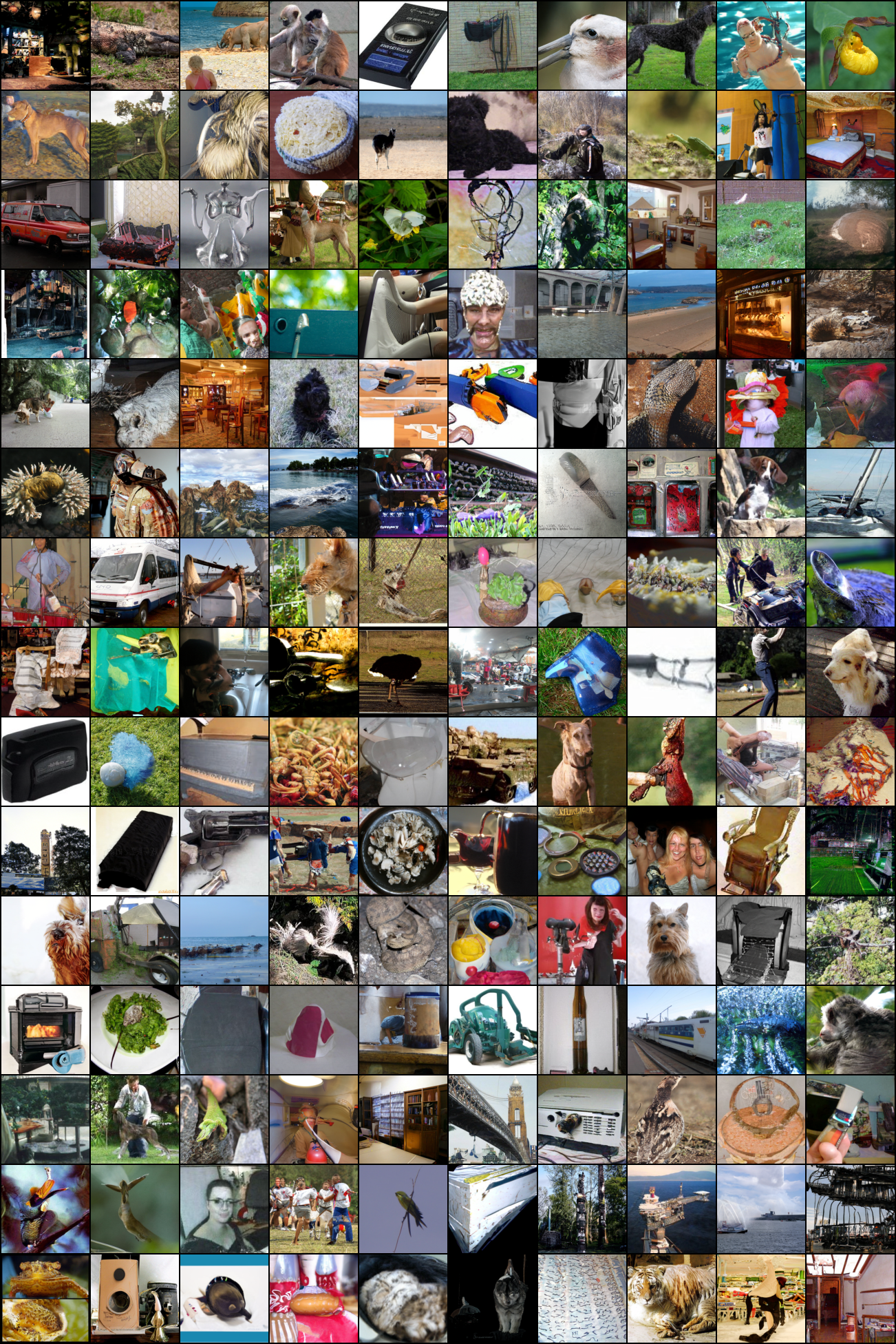}
    \caption{Non-curated unconditional ImageNet-128 generated images of a CNF trained with FM-OT. }
    \label{fig:imagenet128_samples}
\end{figure}

\begin{figure}
    \centering
\begin{tabular}{@{}ccc@{}}
     \includegraphics[width=0.5\textwidth]{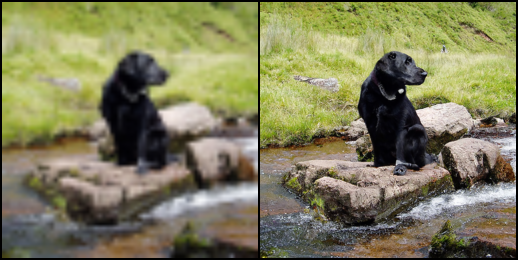} & \includegraphics[width=0.5\textwidth]{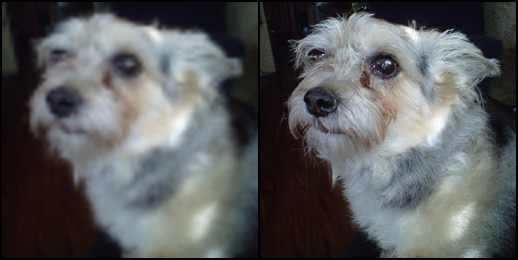} \\ \includegraphics[width=0.5\textwidth]{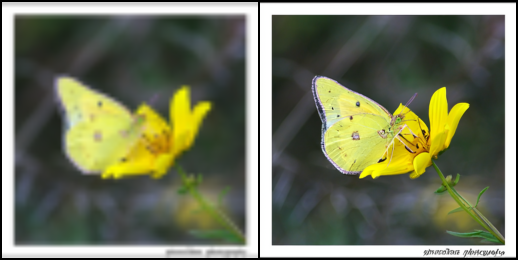} & \includegraphics[width=0.5\textwidth]{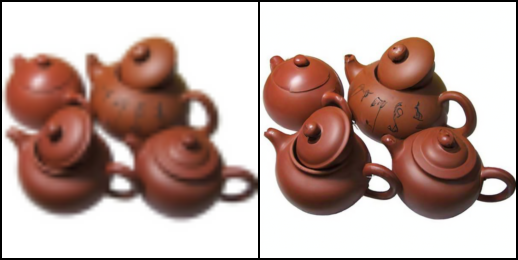} \\
     \includegraphics[width=0.5\textwidth]{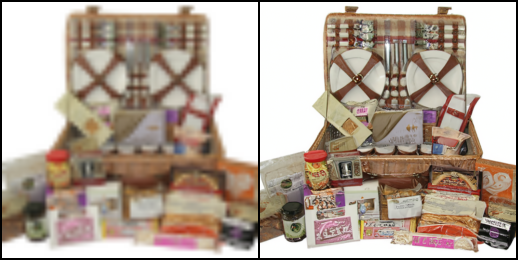} & \includegraphics[width=0.5\textwidth]{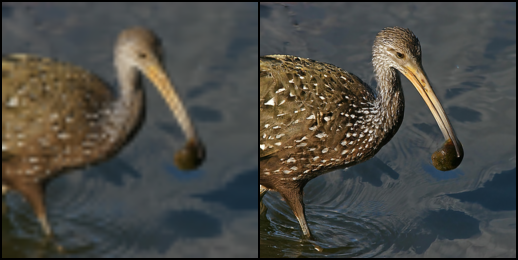} \\ \includegraphics[width=0.5\textwidth]{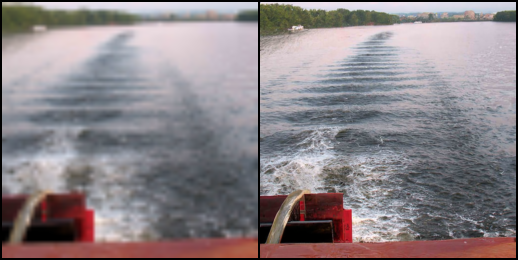} & \includegraphics[width=0.5\textwidth]{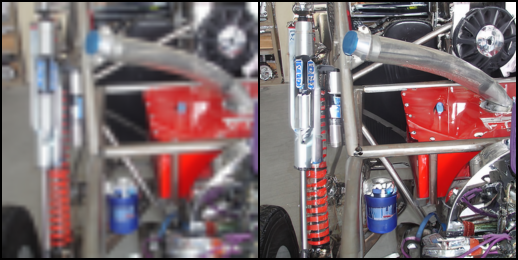} \\
     \includegraphics[width=0.5\textwidth]{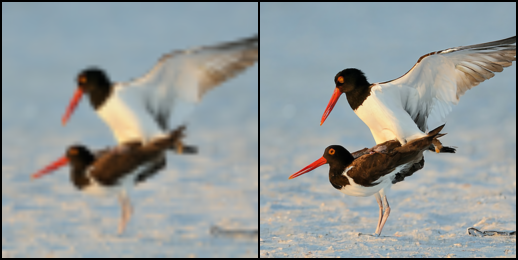} & \includegraphics[width=0.5\textwidth]{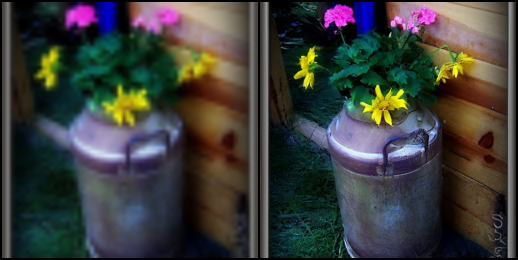} \\ \includegraphics[width=0.5\textwidth]{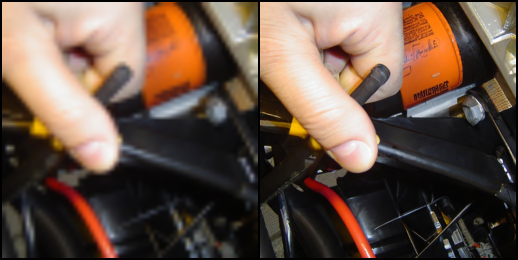} & \includegraphics[width=0.5\textwidth]{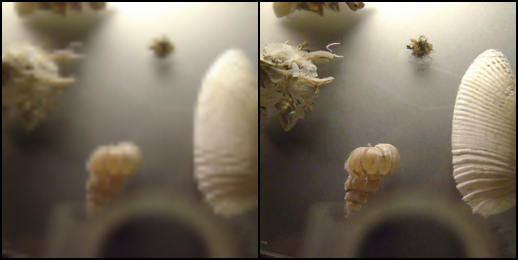} \\
\end{tabular}
    \caption{Conditional generation 64$\times$64$\rightarrow$256$\times$256. Flow Matching OT upsampled images from validation set.}
    \label{fig:upsampled_1}
\end{figure}

\begin{figure}
    \centering
\begin{tabular}{@{}ccc@{}}
     \includegraphics[width=0.5\textwidth]{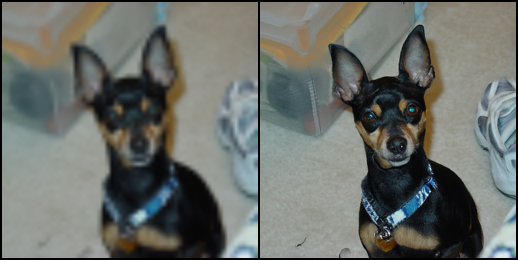} & \includegraphics[width=0.5\textwidth]{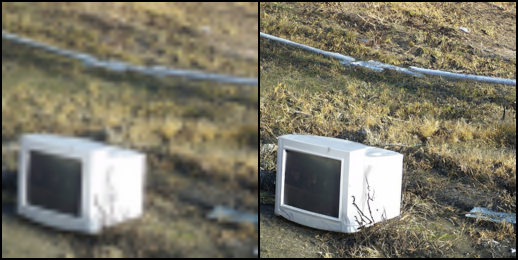} \\
     \includegraphics[width=0.5\textwidth]{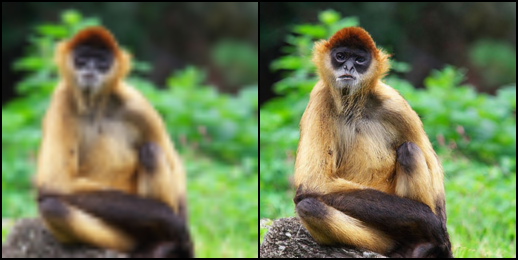} & \includegraphics[width=0.5\textwidth]{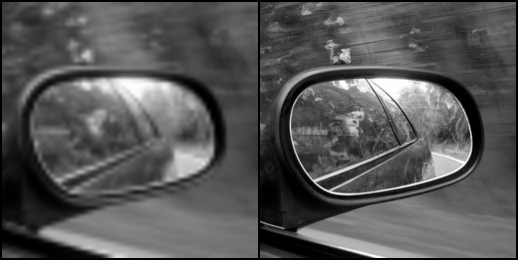} \\
     \includegraphics[width=0.5\textwidth]{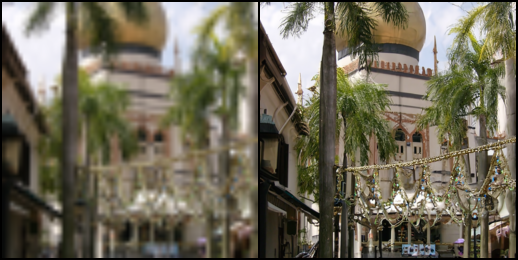} & \includegraphics[width=0.5\textwidth]{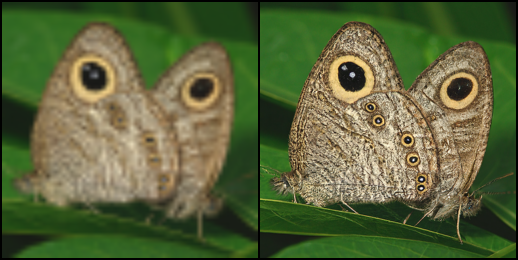} \\
     \includegraphics[width=0.5\textwidth]{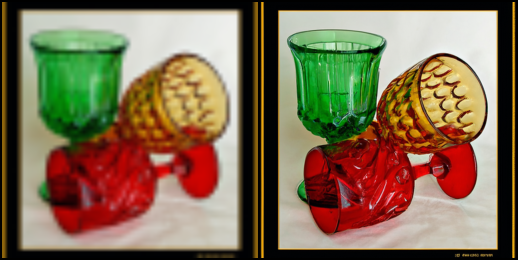} & \includegraphics[width=0.5\textwidth]{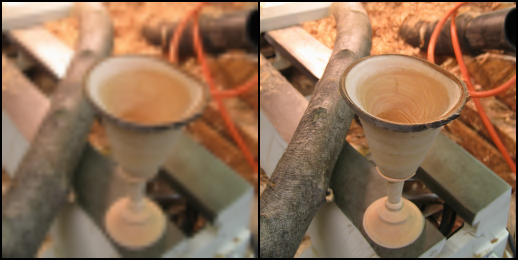} \\
     \includegraphics[width=0.5\textwidth]{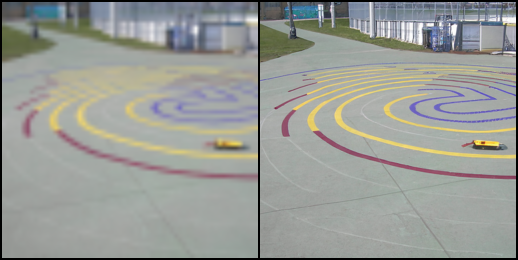} & \includegraphics[width=0.5\textwidth]{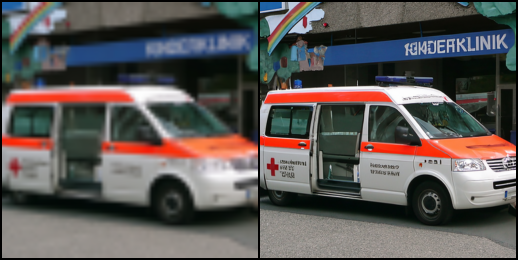} \\
     \includegraphics[width=0.5\textwidth]{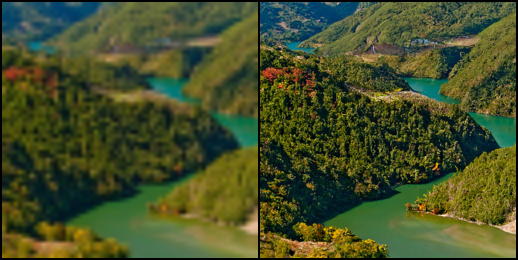} & \includegraphics[width=0.5\textwidth]{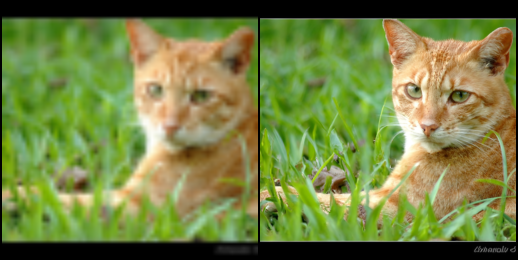} \\
\end{tabular}
    \caption{Conditional generation 64$\times$64$\rightarrow$256$\times$256. Flow Matching OT upsampled images from validation set.}
    \label{fig:upsampled_2}
\end{figure}

\begin{figure}
    \centering
    
    \begin{subfigure}[b]{0.5\linewidth}
        \begin{subfigure}[b]{0.242\linewidth}
        \includegraphics[width=\linewidth]{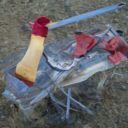}
        \end{subfigure}%
        \begin{subfigure}[b]{0.242\linewidth}
        \includegraphics[width=\linewidth]{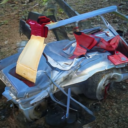}
        \end{subfigure}%
        \begin{subfigure}[b]{0.242\linewidth}
        \includegraphics[width=\linewidth]{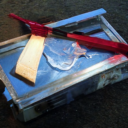}
        \end{subfigure}%
        \begin{subfigure}[b]{0.242\linewidth}
        \includegraphics[width=\linewidth]{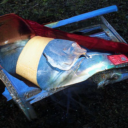}
        \end{subfigure}
        \hfill
    \end{subfigure}%
    \begin{subfigure}[b]{0.5\linewidth}
        \begin{subfigure}[b]{0.242\linewidth}
        \includegraphics[width=\linewidth]{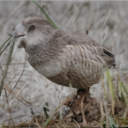}
        \end{subfigure}
        \begin{subfigure}[b]{0.242\linewidth}
        \includegraphics[width=\linewidth]{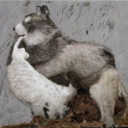}
        \end{subfigure}
        \begin{subfigure}[b]{0.242\linewidth}
        \includegraphics[width=\linewidth]{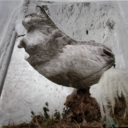}
        \end{subfigure}
        \begin{subfigure}[b]{0.242\linewidth}
        \includegraphics[width=\linewidth]{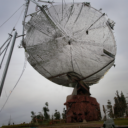}
        \end{subfigure}
    \end{subfigure}\\
    
    \begin{subfigure}[b]{0.5\linewidth}
        \begin{subfigure}[b]{0.242\linewidth}
        \includegraphics[width=\linewidth]{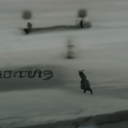}
        \end{subfigure}%
        \begin{subfigure}[b]{0.242\linewidth}
        \includegraphics[width=\linewidth]{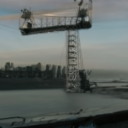}
        \end{subfigure}%
        \begin{subfigure}[b]{0.242\linewidth}
        \includegraphics[width=\linewidth]{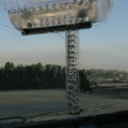}
        \end{subfigure}%
        \begin{subfigure}[b]{0.242\linewidth}
        \includegraphics[width=\linewidth]{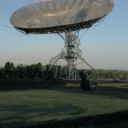}
        \end{subfigure}
        \hfill
    \end{subfigure}%
    \begin{subfigure}[b]{0.5\linewidth}
        \begin{subfigure}[b]{0.242\linewidth}
        \includegraphics[width=\linewidth]{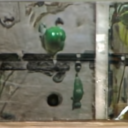}
        \end{subfigure}
        \begin{subfigure}[b]{0.242\linewidth}
        \includegraphics[width=\linewidth]{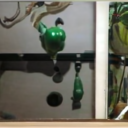}
        \end{subfigure}
        \begin{subfigure}[b]{0.242\linewidth}
        \includegraphics[width=\linewidth]{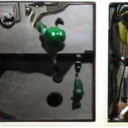}
        \end{subfigure}
        \begin{subfigure}[b]{0.242\linewidth}
        \includegraphics[width=\linewidth]{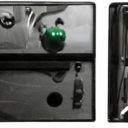}
        \end{subfigure}
    \end{subfigure}\\
    
    \begin{subfigure}[b]{0.5\linewidth}
        \begin{subfigure}[b]{0.242\linewidth}
        \includegraphics[width=\linewidth]{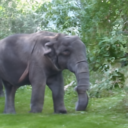}
        \end{subfigure}%
        \begin{subfigure}[b]{0.242\linewidth}
        \includegraphics[width=\linewidth]{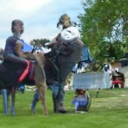}
        \end{subfigure}%
        \begin{subfigure}[b]{0.242\linewidth}
        \includegraphics[width=\linewidth]{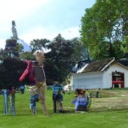}
        \end{subfigure}%
        \begin{subfigure}[b]{0.242\linewidth}
        \includegraphics[width=\linewidth]{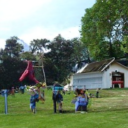}
        \end{subfigure}
        \hfill
    \end{subfigure}%
    \begin{subfigure}[b]{0.5\linewidth}
        \begin{subfigure}[b]{0.242\linewidth}
        \includegraphics[width=\linewidth]{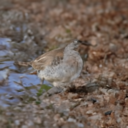}
        \end{subfigure}
        \begin{subfigure}[b]{0.242\linewidth}
        \includegraphics[width=\linewidth]{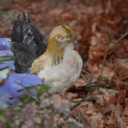}
        \end{subfigure}
        \begin{subfigure}[b]{0.242\linewidth}
        \includegraphics[width=\linewidth]{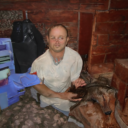}
        \end{subfigure}
        \begin{subfigure}[b]{0.242\linewidth}
        \includegraphics[width=\linewidth]{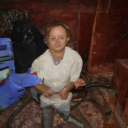}
        \end{subfigure}
    \end{subfigure}\\
    
    \begin{subfigure}[b]{0.5\linewidth}
        \begin{subfigure}[b]{0.242\linewidth}
        \includegraphics[width=\linewidth]{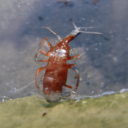}
        \end{subfigure}%
        \begin{subfigure}[b]{0.242\linewidth}
        \includegraphics[width=\linewidth]{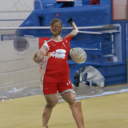}
        \end{subfigure}%
        \begin{subfigure}[b]{0.242\linewidth}
        \includegraphics[width=\linewidth]{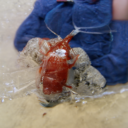}
        \end{subfigure}%
        \begin{subfigure}[b]{0.242\linewidth}
        \includegraphics[width=\linewidth]{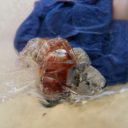}
        \end{subfigure}
        \hfill
    \end{subfigure}%
    \begin{subfigure}[b]{0.5\linewidth}
        \begin{subfigure}[b]{0.242\linewidth}
        \includegraphics[width=\linewidth]{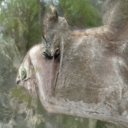}
        \end{subfigure}
        \begin{subfigure}[b]{0.242\linewidth}
        \includegraphics[width=\linewidth]{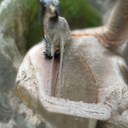}
        \end{subfigure}
        \begin{subfigure}[b]{0.242\linewidth}
        \includegraphics[width=\linewidth]{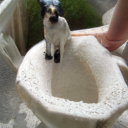}
        \end{subfigure}
        \begin{subfigure}[b]{0.242\linewidth}
        \includegraphics[width=\linewidth]{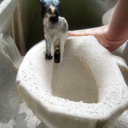}
        \end{subfigure}
    \end{subfigure}\\
    
    \begin{subfigure}[b]{0.5\linewidth}
        \begin{subfigure}[b]{0.242\linewidth}
        \includegraphics[width=\linewidth]{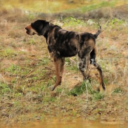}
        \end{subfigure}%
        \begin{subfigure}[b]{0.242\linewidth}
        \includegraphics[width=\linewidth]{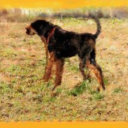}
        \end{subfigure}%
        \begin{subfigure}[b]{0.242\linewidth}
        \includegraphics[width=\linewidth]{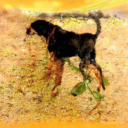}
        \end{subfigure}%
        \begin{subfigure}[b]{0.242\linewidth}
        \includegraphics[width=\linewidth]{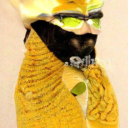}
        \end{subfigure}
        \hfill
    \end{subfigure}%
    \begin{subfigure}[b]{0.5\linewidth}
        \begin{subfigure}[b]{0.242\linewidth}
        \includegraphics[width=\linewidth]{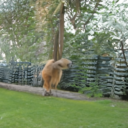}
        \end{subfigure}
        \begin{subfigure}[b]{0.242\linewidth}
        \includegraphics[width=\linewidth]{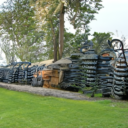}
        \end{subfigure}
        \begin{subfigure}[b]{0.242\linewidth}
        \includegraphics[width=\linewidth]{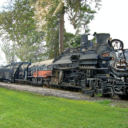}
        \end{subfigure}
        \begin{subfigure}[b]{0.242\linewidth}
        \includegraphics[width=\linewidth]{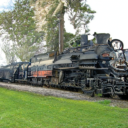}
        \end{subfigure}
    \end{subfigure}\\
    
    \begin{subfigure}[b]{0.5\linewidth}
        \begin{subfigure}[b]{0.242\linewidth}
        \includegraphics[width=\linewidth]{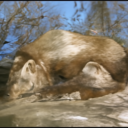}
        \end{subfigure}%
        \begin{subfigure}[b]{0.242\linewidth}
        \includegraphics[width=\linewidth]{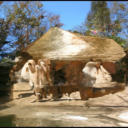}
        \end{subfigure}%
        \begin{subfigure}[b]{0.242\linewidth}
        \includegraphics[width=\linewidth]{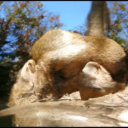}
        \end{subfigure}%
        \begin{subfigure}[b]{0.242\linewidth}
        \includegraphics[width=\linewidth]{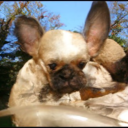}
        \end{subfigure}
        \hfill
    \end{subfigure}%
    \begin{subfigure}[b]{0.5\linewidth}
        \begin{subfigure}[b]{0.242\linewidth}
        \includegraphics[width=\linewidth]{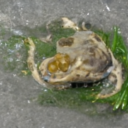}
        \end{subfigure}
        \begin{subfigure}[b]{0.242\linewidth}
        \includegraphics[width=\linewidth]{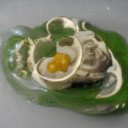}
        \end{subfigure}
        \begin{subfigure}[b]{0.242\linewidth}
        \includegraphics[width=\linewidth]{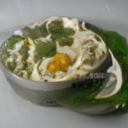}
        \end{subfigure}
        \begin{subfigure}[b]{0.242\linewidth}
        \includegraphics[width=\linewidth]{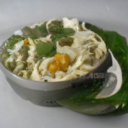}
        \end{subfigure}
    \end{subfigure}\\
    
    \begin{subfigure}[b]{0.5\linewidth}
        \begin{subfigure}[b]{0.242\linewidth}
        \includegraphics[width=\linewidth]{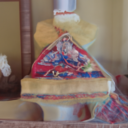}
        \end{subfigure}%
        \begin{subfigure}[b]{0.242\linewidth}
        \includegraphics[width=\linewidth]{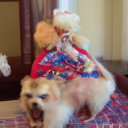}
        \end{subfigure}%
        \begin{subfigure}[b]{0.242\linewidth}
        \includegraphics[width=\linewidth]{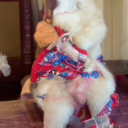}
        \end{subfigure}%
        \begin{subfigure}[b]{0.242\linewidth}
        \includegraphics[width=\linewidth]{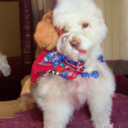}
        \end{subfigure}
        \hfill
    \end{subfigure}%
    \begin{subfigure}[b]{0.5\linewidth}
        \begin{subfigure}[b]{0.242\linewidth}
        \includegraphics[width=\linewidth]{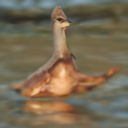}
        \end{subfigure}
        \begin{subfigure}[b]{0.242\linewidth}
        \includegraphics[width=\linewidth]{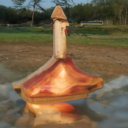}
        \end{subfigure}
        \begin{subfigure}[b]{0.242\linewidth}
        \includegraphics[width=\linewidth]{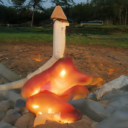}
        \end{subfigure}
        \begin{subfigure}[b]{0.242\linewidth}
        \includegraphics[width=\linewidth]{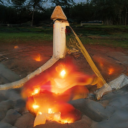}
        \end{subfigure}
    \end{subfigure}\\
    
    \begin{subfigure}[b]{0.5\linewidth}
        \begin{subfigure}[b]{0.242\linewidth}
        \includegraphics[width=\linewidth]{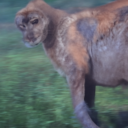}
        \end{subfigure}%
        \begin{subfigure}[b]{0.242\linewidth}
        \includegraphics[width=\linewidth]{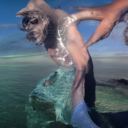}
        \end{subfigure}%
        \begin{subfigure}[b]{0.242\linewidth}
        \includegraphics[width=\linewidth]{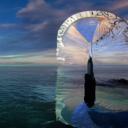}
        \end{subfigure}%
        \begin{subfigure}[b]{0.242\linewidth}
        \includegraphics[width=\linewidth]{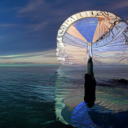}
        \end{subfigure}
        \hfill
    \end{subfigure}%
    \begin{subfigure}[b]{0.5\linewidth}
        \begin{subfigure}[b]{0.242\linewidth}
        \includegraphics[width=\linewidth]{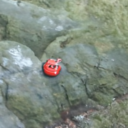}
        \end{subfigure}
        \begin{subfigure}[b]{0.242\linewidth}
        \includegraphics[width=\linewidth]{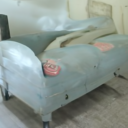}
        \end{subfigure}
        \begin{subfigure}[b]{0.242\linewidth}
        \includegraphics[width=\linewidth]{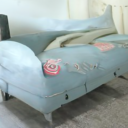}
        \end{subfigure}
        \begin{subfigure}[b]{0.242\linewidth}
        \includegraphics[width=\linewidth]{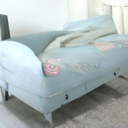}
        \end{subfigure}
    \end{subfigure}\\
    
    \begin{subfigure}[b]{0.5\linewidth}
        \begin{subfigure}[b]{0.242\linewidth}
        \includegraphics[width=\linewidth]{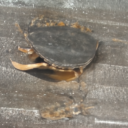}
        \end{subfigure}%
        \begin{subfigure}[b]{0.242\linewidth}
        \includegraphics[width=\linewidth]{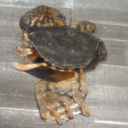}
        \end{subfigure}%
        \begin{subfigure}[b]{0.242\linewidth}
        \includegraphics[width=\linewidth]{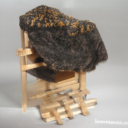}
        \end{subfigure}%
        \begin{subfigure}[b]{0.242\linewidth}
        \includegraphics[width=\linewidth]{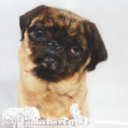}
        \end{subfigure}
        \hfill
    \end{subfigure}%
    \begin{subfigure}[b]{0.5\linewidth}
        \begin{subfigure}[b]{0.242\linewidth}
        \includegraphics[width=\linewidth]{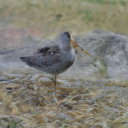}
        \end{subfigure}
        \begin{subfigure}[b]{0.242\linewidth}
        \includegraphics[width=\linewidth]{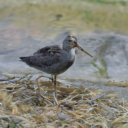}
        \end{subfigure}
        \begin{subfigure}[b]{0.242\linewidth}
        \includegraphics[width=\linewidth]{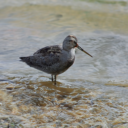}
        \end{subfigure}
        \begin{subfigure}[b]{0.242\linewidth}
        \includegraphics[width=\linewidth]{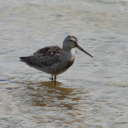}
        \end{subfigure}
    \end{subfigure}\\
    
    \begin{subfigure}[b]{0.5\linewidth}
        \begin{subfigure}[b]{0.242\linewidth}
        \includegraphics[width=\linewidth]{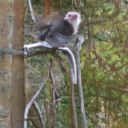}
        \end{subfigure}%
        \begin{subfigure}[b]{0.242\linewidth}
        \includegraphics[width=\linewidth]{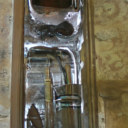}
        \end{subfigure}%
        \begin{subfigure}[b]{0.242\linewidth}
        \includegraphics[width=\linewidth]{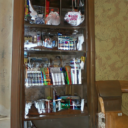}
        \end{subfigure}%
        \begin{subfigure}[b]{0.242\linewidth}
        \includegraphics[width=\linewidth]{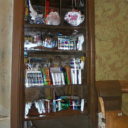}
        \end{subfigure}
        \hfill
    \end{subfigure}%
    \begin{subfigure}[b]{0.5\linewidth}
        \begin{subfigure}[b]{0.242\linewidth}
        \includegraphics[width=\linewidth]{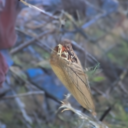}
        \end{subfigure}
        \begin{subfigure}[b]{0.242\linewidth}
        \includegraphics[width=\linewidth]{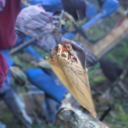}
        \end{subfigure}
        \begin{subfigure}[b]{0.242\linewidth}
        \includegraphics[width=\linewidth]{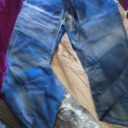}
        \end{subfigure}
        \begin{subfigure}[b]{0.242\linewidth}
        \includegraphics[width=\linewidth]{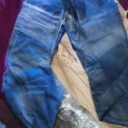}
        \end{subfigure}
    \end{subfigure}\\
    
    \begin{subfigure}[b]{0.5\linewidth}
        \begin{subfigure}[b]{0.242\linewidth}
        \includegraphics[width=\linewidth]{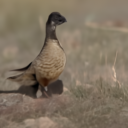}
        \end{subfigure}%
        \begin{subfigure}[b]{0.242\linewidth}
        \includegraphics[width=\linewidth]{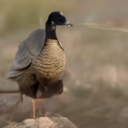}
        \end{subfigure}%
        \begin{subfigure}[b]{0.242\linewidth}
        \includegraphics[width=\linewidth]{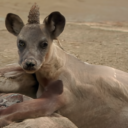}
        \end{subfigure}%
        \begin{subfigure}[b]{0.242\linewidth}
        \includegraphics[width=\linewidth]{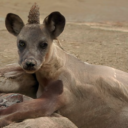}
        \end{subfigure}
        \hfill
    \end{subfigure}%
    \begin{subfigure}[b]{0.5\linewidth}
        \begin{subfigure}[b]{0.242\linewidth}
        \includegraphics[width=\linewidth]{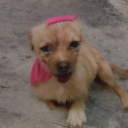}
        \end{subfigure}
        \begin{subfigure}[b]{0.242\linewidth}
        \includegraphics[width=\linewidth]{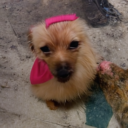}
        \end{subfigure}
        \begin{subfigure}[b]{0.242\linewidth}
        \includegraphics[width=\linewidth]{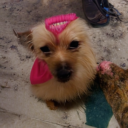}
        \end{subfigure}
        \begin{subfigure}[b]{0.242\linewidth}
        \includegraphics[width=\linewidth]{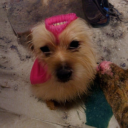}
        \end{subfigure}
    \end{subfigure}\\
    
    \begin{subfigure}[b]{0.5\linewidth}
        \begin{subfigure}[b]{0.242\linewidth}
        \includegraphics[width=\linewidth]{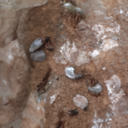}
        \caption*{NFE=10}
        \end{subfigure}%
        \begin{subfigure}[b]{0.242\linewidth}
        \includegraphics[width=\linewidth]{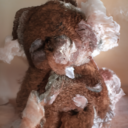}
        \caption*{NFE=20}
        \end{subfigure}%
        \begin{subfigure}[b]{0.242\linewidth}
        \includegraphics[width=\linewidth]{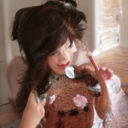}
        \caption*{NFE=40}
        \end{subfigure}%
        \begin{subfigure}[b]{0.242\linewidth}
        \includegraphics[width=\linewidth]{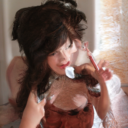}
        \caption*{NFE=100}
        \end{subfigure}
        \hfill
    \end{subfigure}%
    \begin{subfigure}[b]{0.5\linewidth}
        \begin{subfigure}[b]{0.242\linewidth}
        \includegraphics[width=\linewidth]{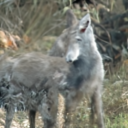}
        \caption*{NFE=10}
        \end{subfigure}
        \begin{subfigure}[b]{0.242\linewidth}
        \includegraphics[width=\linewidth]{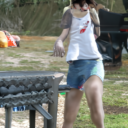}
        \caption*{NFE=20}
        \end{subfigure}
        \begin{subfigure}[b]{0.242\linewidth}
        \includegraphics[width=\linewidth]{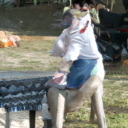}
        \caption*{NFE=40}
        \end{subfigure}
        \begin{subfigure}[b]{0.242\linewidth}
        \includegraphics[width=\linewidth]{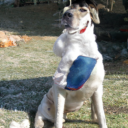}
        \caption*{NFE=100}
        \end{subfigure}
    \end{subfigure}\\
    \caption{Generated samples from the same initial noise, but with varying number of function evaluations (NFE). 
    Flow matching with OT path trained on ImageNet-128.}
\end{figure}

\begin{figure}
    \centering
    
    \begin{subfigure}[b]{0.19\linewidth}
    \includegraphics[width=\linewidth]{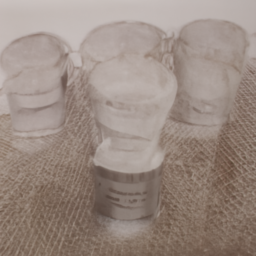}
    \end{subfigure}
    \begin{subfigure}[b]{0.19\linewidth}
    \includegraphics[width=\linewidth]{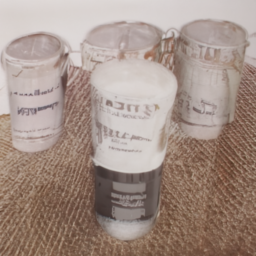}
    \end{subfigure}
    \begin{subfigure}[b]{0.19\linewidth}
    \includegraphics[width=\linewidth]{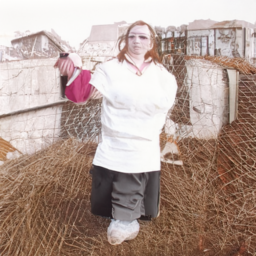}
    \end{subfigure}
    \begin{subfigure}[b]{0.19\linewidth}
    \includegraphics[width=\linewidth]{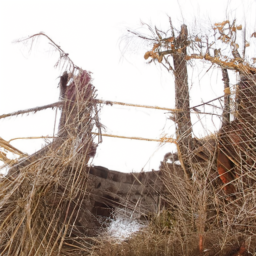}
    \end{subfigure}
    \begin{subfigure}[b]{0.19\linewidth}
    \includegraphics[width=\linewidth]{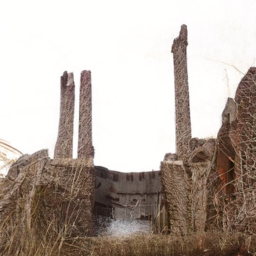}
    \end{subfigure}\\
    
    \begin{subfigure}[b]{0.19\linewidth}
    \includegraphics[width=\linewidth]{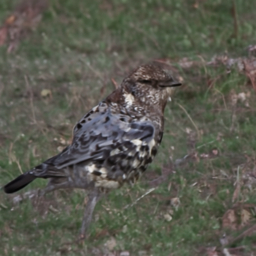}
    \end{subfigure}
    \begin{subfigure}[b]{0.19\linewidth}
    \includegraphics[width=\linewidth]{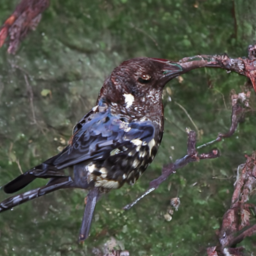}
    \end{subfigure}
    \begin{subfigure}[b]{0.19\linewidth}
    \includegraphics[width=\linewidth]{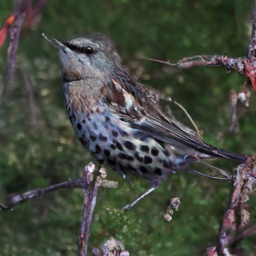}
    \end{subfigure}
    \begin{subfigure}[b]{0.19\linewidth}
    \includegraphics[width=\linewidth]{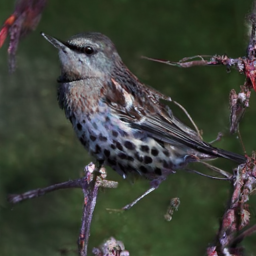}
    \end{subfigure}
    \begin{subfigure}[b]{0.19\linewidth}
    \includegraphics[width=\linewidth]{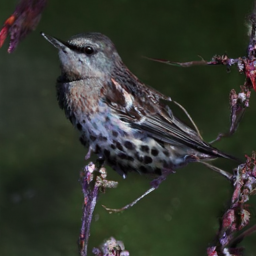}
    \end{subfigure}\\
    
    \begin{subfigure}[b]{0.19\linewidth}
    \includegraphics[width=\linewidth]{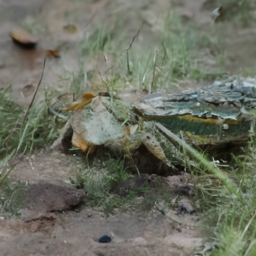}
    \end{subfigure}
    \begin{subfigure}[b]{0.19\linewidth}
    \includegraphics[width=\linewidth]{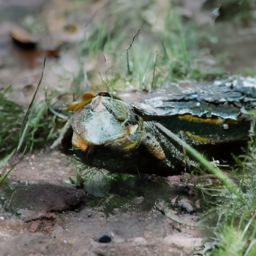}
    \end{subfigure}
    \begin{subfigure}[b]{0.19\linewidth}
    \includegraphics[width=\linewidth]{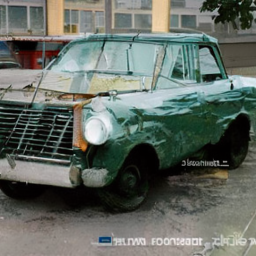}
    \end{subfigure}
    \begin{subfigure}[b]{0.19\linewidth}
    \includegraphics[width=\linewidth]{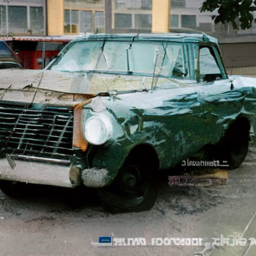}
    \end{subfigure}
    \begin{subfigure}[b]{0.19\linewidth}
    \includegraphics[width=\linewidth]{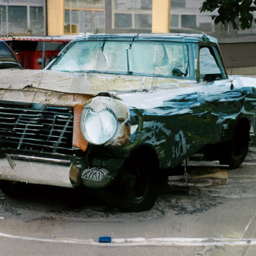}
    \end{subfigure}\\
    
    \begin{subfigure}[b]{0.19\linewidth}
    \includegraphics[width=\linewidth]{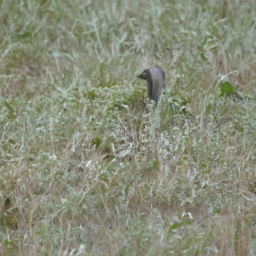}
    \end{subfigure}
    \begin{subfigure}[b]{0.19\linewidth}
    \includegraphics[width=\linewidth]{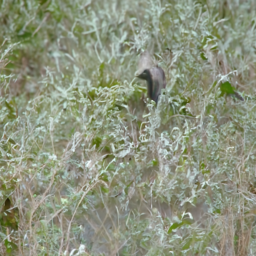}
    \end{subfigure}
    \begin{subfigure}[b]{0.19\linewidth}
    \includegraphics[width=\linewidth]{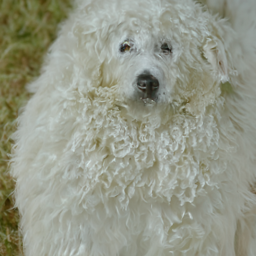}
    \end{subfigure}
    \begin{subfigure}[b]{0.19\linewidth}
    \includegraphics[width=\linewidth]{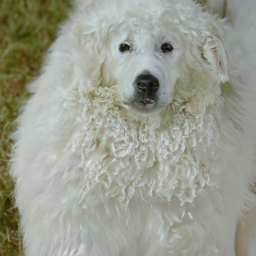}
    \end{subfigure}
    \begin{subfigure}[b]{0.19\linewidth}
    \includegraphics[width=\linewidth]{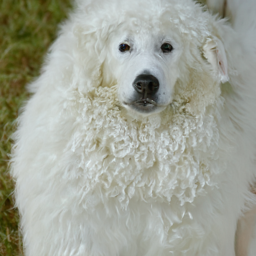}
    \end{subfigure}\\
    
    \begin{subfigure}[b]{0.19\linewidth}
    \includegraphics[width=\linewidth]{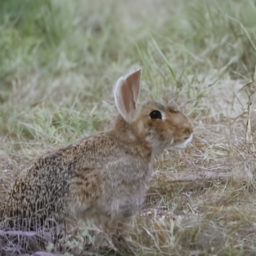}
    \end{subfigure}
    \begin{subfigure}[b]{0.19\linewidth}
    \includegraphics[width=\linewidth]{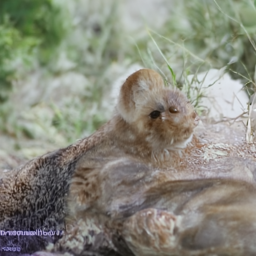}
    \end{subfigure}
    \begin{subfigure}[b]{0.19\linewidth}
    \includegraphics[width=\linewidth]{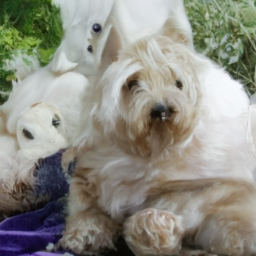}
    \end{subfigure}
    \begin{subfigure}[b]{0.19\linewidth}
    \includegraphics[width=\linewidth]{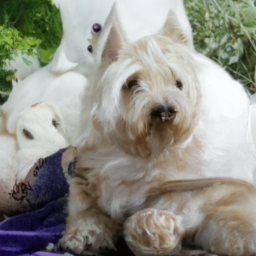}
    \end{subfigure}
    \begin{subfigure}[b]{0.19\linewidth}
    \includegraphics[width=\linewidth]{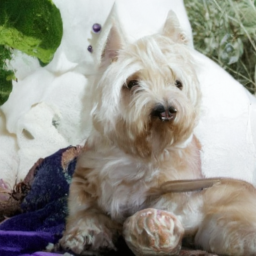}
    \end{subfigure}\\
    
    \begin{subfigure}[b]{0.19\linewidth}
    \includegraphics[width=\linewidth]{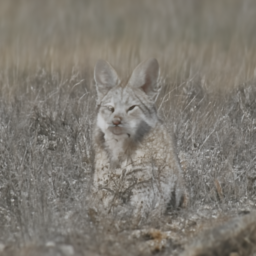}
    \end{subfigure}
    \begin{subfigure}[b]{0.19\linewidth}
    \includegraphics[width=\linewidth]{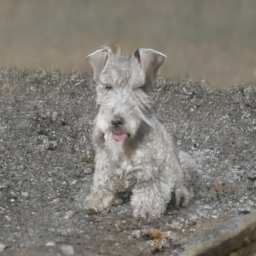}
    \end{subfigure}
    \begin{subfigure}[b]{0.19\linewidth}
    \includegraphics[width=\linewidth]{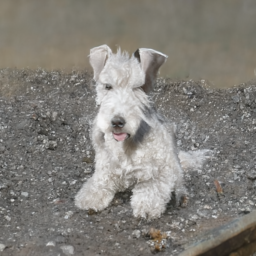}
    \end{subfigure}
    \begin{subfigure}[b]{0.19\linewidth}
    \includegraphics[width=\linewidth]{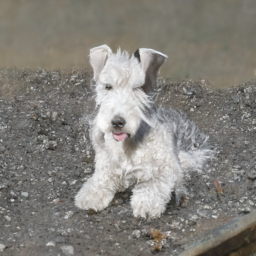}
    \end{subfigure}
    \begin{subfigure}[b]{0.19\linewidth}
    \includegraphics[width=\linewidth]{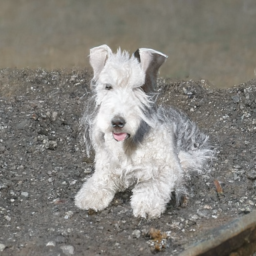}
    \end{subfigure}\\
    
    \begin{subfigure}[b]{0.19\linewidth}
    \includegraphics[width=\linewidth]{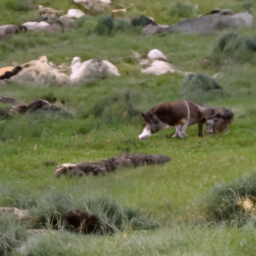}
    \end{subfigure}
    \begin{subfigure}[b]{0.19\linewidth}
    \includegraphics[width=\linewidth]{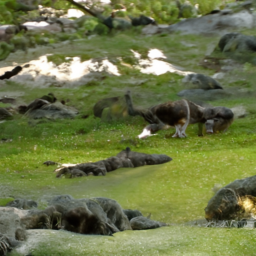}
    \end{subfigure}
    \begin{subfigure}[b]{0.19\linewidth}
    \includegraphics[width=\linewidth]{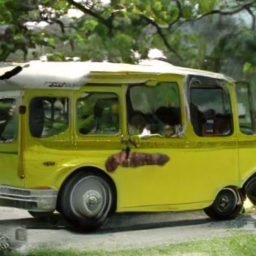}
    \end{subfigure}
    \begin{subfigure}[b]{0.19\linewidth}
    \includegraphics[width=\linewidth]{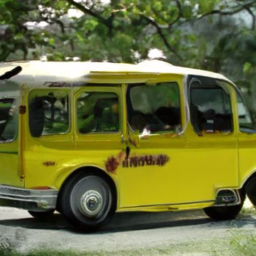}
    \end{subfigure}
    \begin{subfigure}[b]{0.19\linewidth}
    \includegraphics[width=\linewidth]{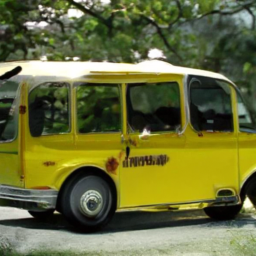}
    \end{subfigure}\\
    
    \begin{subfigure}[b]{0.19\linewidth}
    \includegraphics[width=\linewidth]{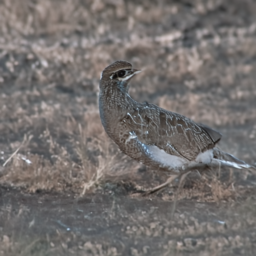}
    \caption*{NFE=10}
    \end{subfigure}
    \begin{subfigure}[b]{0.19\linewidth}
    \includegraphics[width=\linewidth]{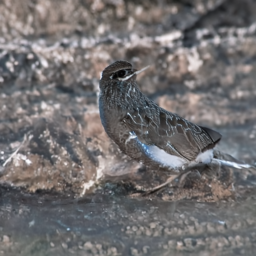}
    \caption*{NFE=20}
    \end{subfigure}
    \begin{subfigure}[b]{0.19\linewidth}
    \includegraphics[width=\linewidth]{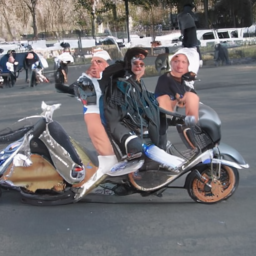}
    \caption*{NFE=40}
    \end{subfigure}
    \begin{subfigure}[b]{0.19\linewidth}
    \includegraphics[width=\linewidth]{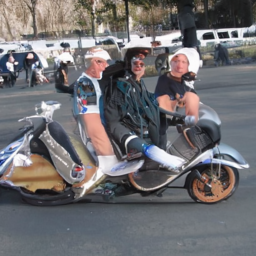}
    \caption*{NFE=60}
    \end{subfigure}
    \begin{subfigure}[b]{0.19\linewidth}
    \includegraphics[width=\linewidth]{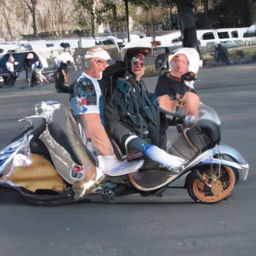}
    \caption*{NFE=100}
    \end{subfigure}\\

    \caption{Generated samples from the same initial noise, but with varying number of function evaluations (NFE). Flow matching with OT path trained on ImageNet 256$\times $256.}
    \label{fig:imagenet256_solver_samples}
\end{figure}

\end{document}